%% file: main.tex
\documentclass[dvipsnames]{article}
\usepackage{graphicx}
\usepackage{booktabs}
\usepackage{caption}
\usepackage{subcaption}
\usepackage{showlabels}
\usepackage{enumitem}
\usepackage[dvipsnames]{xcolor}
\usepackage{amsmath} 
\usepackage{amssymb}
\usepackage{amsthm}  
\usepackage{microtype}
\usepackage{alphalph}
\usepackage{etoolbox}
\usepackage[dvipsnames]{xcolor}
\usepackage{minitoc}

\AtBeginDocument{%
  \AtBeginEnvironment{subequations}{%
  }
}

\allowdisplaybreaks

\newif\ifshowcomments
\showcommentsfalse

\ifshowcomments

\newcommand{\ben}[1]{\textcolor{OliveGreen}{[BA: #1]}}
\newcommand{\jp}[1]{\textcolor{Blue}{[JP: #1]}}

\else

\newcommand{\ben}[1]{}
\newcommand{\jp}[1]{}

\fi

\renewcommand{\cal}{\mathcal} 
\newcommand{\e}{{\varepsilon}}

\renewcommand{\P}{\mathbb{P}}
\newcommand{\E}{\mathbb{E}}
\newcommand{\C}{\mathbb{C}}

\newcommand{\V}{\mathbb{V}}
\newcommand{\deq}{\mathrel{\mathop:}=} 
\newcommand{\deqrev}{=\mathrel{\mathop:}} 
\newcommand{\ft}{\sigma_\textsc{t}}
\newcommand{\fs}{\sigma}
\newcommand{\bfx}{\mathbf{x}}
\newcommand{\bfy}{\mathbf{y}}
\newcommand{\bft}{\mathbf{\theta}}
\newcommand{\bfe}{\boldsymbol{\varepsilon}}

\DeclareMathOperator{\diag}{diag}
\newcommand{\bias}{B}
\newcommand{\tE}{T}
\newcommand{\Etrain}{E_\text{train}}
\newcommand{\Etest}{E_\text{test}}
\DeclareMathOperator{\tr}{tr}

\newcommand{\eq}[1]{\begin{equation}#1\end{equation}}

\newcommand{\al}[1]{\begin{align}#1\end{align}}

\newcommand{\pa}[1]{\left({#1}\right)}

\newcommand{\q}[1]{[{#1}]}

\newcommand{\qa}[1]{\left[{#1}\right]}

\newcommand{\ha}[1]{\left\{{#1}\right\}}

\newcommand{\bfi}{\mathbf{i}}
\newcommand{\bfj}{\mathbf{j}}
\newcommand{\bfee}{\mathbf{e}}
\newcommand{\bfX}{\mathbf{X}}

\newtheorem{theorem}{Theorem}
\newtheorem{lemma}{Lemma}

\newtheorem{proposition}{Proposition}
\newtheorem{corollary}{Corollary}
\newtheorem{remark}{Remark}
\newtheorem{example}{Example}
\PassOptionsToPackage{numbers, compress}{natbib}

\usepackage[final]{neurips_2020}

\usepackage[utf8]{inputenc} 
\usepackage[T1]{fontenc}    
\usepackage{hyperref}       
\usepackage{url}            
\usepackage{amsfonts}       
\usepackage{nicefrac}       
\usepackage{microtype}      
\usepackage{hyperref}

\ifshowcomments
    \title{COMMENTS ARE DISPLAYED: Understanding Double Descent Requires a Fine-Grained Bias-Variance Decomposition}
\else
    \title{Understanding Double Descent Requires a Fine-Grained Bias-Variance Decomposition}
\fi
\author{%
  Ben Adlam\thanks{Both authors contributed equally to this work. \textsuperscript{\textdagger}Work done as a member of the Google AI Residency program (https://g.co/airesidency).} ~\textsuperscript{\textdagger} \qquad Jeffrey Pennington\textsuperscript{*}\\
  Google Brain\\
  \texttt{\{adlam, jpennin\}@google.com} \\
}

\begin{document}
\doparttoc 
\faketableofcontents 


\maketitle

\begin{abstract}

Classical learning theory suggests that the optimal generalization performance of a machine learning model should occur at an intermediate model complexity, with simpler models exhibiting high bias and more complex models exhibiting high variance of the predictive function. However, such a simple trade-off does not adequately describe deep learning models that simultaneously attain low bias and variance in the heavily overparameterized regime. A primary obstacle in explaining this behavior is that deep learning algorithms typically involve multiple sources of randomness whose individual contributions are not visible in the total variance. To enable fine-grained analysis, we describe an interpretable, symmetric decomposition of the variance into terms associated with the randomness from sampling, initialization, and the labels. Moreover, we compute the high-dimensional asymptotic behavior of this decomposition for random feature kernel regression, and analyze the strikingly rich phenomenology that arises. We find that the bias decreases monotonically with the network width, but the variance terms exhibit non-monotonic behavior and can diverge at the interpolation boundary, even in the absence of label noise. The divergence is caused by the \emph{interaction} between sampling and initialization and can therefore be eliminated by marginalizing over samples (\emph{i.e.} bagging) \emph{or} over the initial parameters (\emph{i.e.} ensemble learning).
\end{abstract}

\section{Introduction}
\label{sec_intro}

It is undeniable that modern neural networks (NNs) are becoming larger and more complex, with many state-of-the-art models now employing billions of trainable parameters \cite{radford2019language, adiwardana2020towards,shazeer2017outrageously}. While parameter count may be a crude way of quantifying complexity, there is little doubt that these models have enormous capacity, often far more than is needed to perfectly fit the training data, even if the labels are pure noise~\cite{zhang2016understanding}. Surprisingly, these same high-capacity models generalize well when trained on real data.

These observations conflict with classical generalization theory, which contends that models of intermediate complexity should generalize best, striking a balance between the bias and the variance of their predictive functions. A paradigm for understanding the observed generalization behavior of modern methods is known as \emph{double descent}~\cite{belkin2019reconciling}, in which the test error behaves as predicted by classical theory and follows the standard U-shaped curve until the point where the training set can be fit exactly, but after this point it begins to descend again, eventually finding its global minimum in the overparameterized regime.

While double descent has been the focus of significant research, a concrete and interpretable theoretical explanation for the phenomenon has thus far been lacking. One of the challenges in developing such an explanation is that the full phenomenology of double descent is not evident in linear models that are easy to analyze. Indeed, for linear models the number of parameters is tied to the number of features and there is no natural way to adjust the capacity of the model without simultaneously adjusting the data distribution. In this work, we overcome this challenge by providing a precise asymptotic analysis of random feature kernel regression, which is a model rich enough to exhibit all the interesting features of double descent.

Another challenge in understanding double descent is that the classical bias-variance decomposition is \emph{itself} insufficiently nuanced to reveal all the underlying explanatory factors. Indeed, modern learning algorithms typically involve multiple sources of randomness and isolating the variation caused by each of these sources of randomness is key to building an effective interpretation. As we will see, it is not possible to fully understand the spike in test error near the interpolation threshold without performing a truly multivariate variance decomposition.

While decomposing the variance has been proposed before, prior work has naively relied on the law of total variance, which requires specifying an ordering of conditioning that leads to some arbitrariness. Instead, we present a principled symmetric decomposition which leads to unambiguous interpretations and clear credit assignment. Decomposing the variance of a random variable in this way is related to ANOVA \cite{efron1981jackknife}, which has been used previously in a machine learning context to find the best approximating functions (in terms of mean squared error) to a random variable with limited dependence on the inputs \cite{stone1994use,huang1998projection} and to study quasi Monte Carlo methods for integration \cite{owen2003dimension}.

Finally, we remark that an improved understanding of the bias and variance of machine learning models might naturally suggest ways to improve their performance. Specifically, any prior knowledge about what sources of variance may be dominant could help inform decisions about which types of ensemble or bagging techniques to utilize.

\subsection{Related Work}
\label{subsec_related_work}

The idea of a trade-off between bias and variance has a long history, with theoretical and experimental support having been well established in a variety of contexts over the years. The seminal paper of~\citet{geman1992neural} examines a number of models, ranging from kernel regression to $k$-nearest neighbor to neural networks, and concludes that the trade-off exists in all cases\footnote{Interestingly, the variance of simple feed-forward neural networks was observed to eventually be a decreasing function of width, but the authors rationalized this early evidence of double descent as a quirk of the optimization.}. The resulting U-shaped test error curve was verified theoretically in a variety of classical settings, see \emph{e.g.}~\cite{vapnik1999overview}.

In recent years, these conclusions have been called into question by the intriguing experimental results of~\cite{zhang2016understanding,belkin2018understand}, which were later replicated in a number of settings, see \emph{e.g.}~\cite{nakkiran2019deep}, which showed that deep neural networks and kernel methods can generalize well even in the interpolation regime, implying that both the bias and the variance can decrease as the model complexity increases. A number of theoretical results have since established this behavior in certain settings, such as interpolating nearest neighbor schemes~\cite{belkin2018overfitting} and kernel regression~\cite{belkin2018does,liang2018just}. These observations have given rise to the double descent paradigm for understanding how test error depends on model complexity \cite{belkin2019reconciling}. The influential work \cite{advani2017high} (which actually predates \cite{belkin2019reconciling}) established initial theoretical insights for linear networks and found empirical evidence of double descent for nonlinear networks; more evidence has followed recently in \cite{nakkiran2019deep,geiger2019scaling}. Precise theoretical predictions soon confirmed this picture for linear regression in various scenarios~\cite{kobak2018optimal,belkin2019two,hastie2019surprises, mitra2019understanding}, and recently even for kernel regression~\cite{mei2019generalization,adlamneural} with random features related to neural networks.

The primary focus of these recent works has been on double descent in the total test error, or perhaps the standard bias-variance decomposition with respect to label noise~\cite{mei2019generalization}. A multivariate philosophy similar to ours is advanced in~\cite{neal2018modern}, which revisited the empirical study of the bias-variance tradeoff in neural networks from \cite{geman1992neural} and showed the variance can decrease in the overparameterized regime. However, in that work the variance is simply decomposed using the law of total variance, which, while mathematically sound, can lead to ambiguous conclusions, as we discuss in Sec.~\ref{sec_fine_grained_analysis_dd}.

The main mathematical tools we utilize come from random matrix theory and build on the results of~\cite{pennington2017nonlinear,pennington2018spectrum,adlam2019random,louart2018random,peche2019note} for studying random matrices with nonlinear dependencies. We also rely on techniques from operator-valued free probability for computing traces of large block matrices~\cite{far2006spectra}. One advantage of these tools is that they facilitate the extension of our analysis to more general settings, including the case of kernel regression with respect to the Neural Tangent Kernel (NTK)~\cite{jacot2018neural}. To ease the exposition we have deferred the discussion of the NTK and all proofs to the Supplementary Material (SM).

While finalizing this manuscript, we became aware of several concurrent works that examine similar questions. \citet{yang2020rethinking} define the total bias and variance similarly to \cite{neal2018modern}, but they do not attempt a decomposition of the variance. Their results can be derived as a special case of our fine-grained decomposition by summing the variance terms in Thm.~\ref{thm:main}. \citet{jacot2020implicit} study the relationship between the random feature model and the nonparametric Gaussian process which it approximates. The bias-variance decomposition considered in that paper is again univariate and is with respect to the randomness in the random features (the expressions are subsequently averaged over the training data). Closest to our work is \cite{d2020double}, which also studies a multivariate decomposition of the random feature model in the high-dimensional limit. Unlike our approach, their decomposition is not symmetric with respect to the underlying random variables, and the results depend on the chosen order of conditioning. Their particular choice, and indeed all possible choices, arise as special cases of our general result. See Sec.~\ref{sec_dasc_bv} for a detailed discussion.

\subsection{Our Contributions}
\label{subsec_contribution}
\begin{enumerate}[nolistsep]
    \item We develop a symmetric, interpretable variance decomposition suitable for modern deep learning algorithms
    \item We compute this decomposition analytically for random feature kernel regression in the high-dimensional asymptotic regime
    \item We prove that the bias is monotonically decreasing as the width increases and that it is finite at the interpolation threshold
    \item We clarify the relationship between label noise and double descent: while the test loss can diverge at the interpolation threshold without label noise, the divergence is exacerbated by it
    \item We provide a quantitative description of how both ensemble and bagging methods can eliminate double descent, since the divergence is caused by variance terms due to the interactions between sampling and initialization
\end{enumerate}

\section{Bias-Variance Decomposition}
\label{sec_bias_variance}

In this section, we trace through the evolution of several ways to analyze the bias-variance trade-off. By analyzing their shortcomings, we motivate our fine-grained analysis that follows.

\subsection{Classical Bias-Variance Decomposition}
The bias-variance trade-off has long served as a useful paradigm for understanding the generalization of machine learning algorithms. For a given test point $\bfx$, it decomposes the expected error as
\eq{\label{eq_bias_variance_x}
    \E \qa{\hat{y}(\bfx)-y(\bfx)}^2 = \pa{\E\hat{y}(\bfx) - \E y(\bfx)}^2 + \V\qa{\hat{y}(\bfx)} + \V\q{y(\bfx)}\,,
}
and subsequently averages over the test point to obtain a decomposition of the test error in which the first term is the bias, the second term is the variance, and the third term is the irreducible noise. In classical settings, the randomness of the predictive function is usually regarded as coming from randomness in the training data, \emph{i.e.} sampling noise. This leads to two common conventions, where the expectations in eqn.~\eqref{eq_bias_variance_x} are over both $X$ and $\bfy$ or are conditional on $X$ and only over the label noise in $\bfy$. For concreteness and to simplify the exposition, in this subsection we adopt the latter convention and make the common modelling assumption that the sampling noise is an additive term $\bfe$ on the training labels but is zero on the test labels $y(\bfx)$. Using $\E_\bfx$ to denote expectation over the test point, we have
\eq{
\label{eq_bias_variance_eps}
\Etest \deq \E_{\bfx}\E_{\bfe} \qa{\hat{y}(\bfx)-y(\bfx)}^2 = \underbrace{\E_{\bfx}\pa{\E_{\bfe}[\hat{y}(\bfx)] -  y(\bfx)}^2}_{\text{Bias}} + \underbrace{\E_{\bfx}\V_{\bfe}\qa{\hat{y}(\bfx)}}_{\text{Variance}} \,.
}
We refer to eqn.~\eqref{eq_bias_variance_eps} as the \emph{classical bias-variance decomposition}.

\subsection{Bias-Variance Decompositions for Modern Learning Methods}
\label{sec_bv_for_modern_ml}
Modern methods for training neural networks often utilize additional sources of randomness, such as the initial parameter values, minibatch selection, \emph{etc}., which we collectively denote by $\bft$. One is therefore left with a choice regarding whether or not to include $\bft$ in the expectations in eqn.~\eqref{eq_bias_variance_x}, or to simply average over $\bft$ when computing the test loss. We explore the ramifications of these different choices below.

\paragraph{Semi-classical Approach.}
In what we call the \emph{semi-classical} approach, the additional random variables $\bft$ coming from initialization or optimization are not included in the expectations in eqn.~\eqref{eq_bias_variance_x}; we instead average over these quantities to define
 \eq{
\label{eq_bias_variance_theta}
\Etest \deq \E_{\bfx}\E_{\bft}\E_{\bfe} [(\hat{y}(\bfx)-y(\bfx))^2|\bft] =  \underbrace{\E_{\bfx}\E_{\bft}\pa{\E_{\bfe}[\hat{y}(\bfx)|\bft] -  y(\bfx)}^2}_{B_{SC}} + \underbrace{\E_{\bfx}\E_{\bft}\V_{\bfe}[\hat{y}(\bfx)|\bft]}_{V_{SC}}\,.
}
In some scenarios, such as the high-dimensional setup analyzed in~\cite{mei2019generalization}, the additional averaging over $\theta$ is unnecessary as the distributions concentrate around their mean. In those situations, the semi-classical decomposition is identical to the classical one, thus motivating this particular approach.

\paragraph{Multivariate Approach.}
In what we call the \emph{multivariate approach}, the additional random variables $\bft$ are included in the expectations in eqn.~\eqref{eq_bias_variance_x}, so that all random variables are on the same footing. We can then drop explicit references to $\bft$ and $\bfe$ and simply write,
 \eq{
\label{eq_bias_variance_total}
\Etest \deq \E_{\bfx}\E (\hat{y}(\bfx)-y(\bfx))^2 = \underbrace{\E_{\bfx}\pa{\E[\hat{y}(\bfx)] -  y(\bfx)}^2}_{B}  + \underbrace{\E_{\bfx}\V[\hat{y}(\bfx)]}_{V}\,.
}
One advantage of this perspective is that its form is completely symmetric with respect to the underlying random variables. Another is that the predictive function $\hat{y}(\bfx)$ appearing in the bias $B$ is not conditional on any random variables. As we discuss in Sec.~\ref{sec_fine_grained_analysis_dd}, this facilitates its interpretation as a measure of erroneous assumptions in the model.

The downside of this perspective is that the variance $V$ no longer admits a simple interpretation since it contains contributions from multiple random variables. This problem can be remedied by further decomposing the variance.

\subsubsection{Symmetric Decomposition of the Variance}
\label{sec:symm}
To gain further insight into the structure of the total variance $V$ and how individual random variables contribute to it, it can be useful to write $V$ as a sum of individual terms, each with an unambiguous meaning.

One path forward is to rely on the law of total variance: $\V \qa{\cal{Y}} = \E\V\qa{\cal{Y}|\cal{X}} + \V\E\qa{\cal{Y}|\cal{X}} $, where the terms represent the variance of $\cal{Y}$ \emph{unexplained} and \emph{explained} by $\cal{X}$ respectively. However, one is immediately confronted by the question of which source of randomness to condition on. As we discuss in Sec.~\ref{sec:ambig}, different choices yield different terms and can lead to ambiguous interpretations.

To avoid this ambiguity, we introduce a fully-symmetric decomposition, which turns out to be unique if we additionally require self-consistency under marginalization with respect to all variables.

\begin{proposition}
\label{prop:decomp}
    Let $X_1,\ldots,X_K,$ and $Y$ be random variables and $\mathcal{X}\deq \{X_1,\ldots,X_K\}$. We define a \emph{variance decomposition} of $Y$ to be a multiset $\{V_1,\ldots,V_N \}$ of nonnegative real numbers such that $\V[Y] = \sum_i V_i$. Then there exists a unique variance decomposition $\mathcal{V} \deq \{V_s : s \subseteq \mathcal{X}\}$ such that $\mathcal{V}$ is invariant under permutations of $\mathcal{X}$, and such that for all $S \subseteq \mathcal{X}$ the marginal variances satisfy the subset-sum relation,
    \eq{
    \label{eq:subset_sum}
        \V\E[Y|X_j \text{ for } j\in S] = \sum_{s\subseteq S} V_{s}\,.
    }
\end{proposition}
\begin{example}\label{ex_2_var}
Consider the case of two random variables, the parameters $P$ and the data $D$. Then $\mathcal{X} = \{P,D\}$ and the decomposition satisfying Prop.~\ref{prop:decomp} is given by
\begin{align}
    V_P &\deq \E_\bfx \V\E[\hat{y}| P]\\
    V_D &\deq \E_\bfx \V\E[\hat{y}| D]\\
    V_{PD} &\deq \E_\bfx \V\E[\hat{y}|P,D] - \E_\bfx \V\E[\hat{y}| P] - \E_\bfx \V\E[\hat{y}| D]\,.
\end{align}
We can interpret $V_{PD}$ as the variance explained by the parameters and data together beyond what they explain individually.
\end{example}
\begin{example}\label{ex_3_var}
Further decomposing $D$ into randomness from sampling the inputs $X$ and label noise $\bfe$, we can write $\mathcal{X} = \{P,X,\bfe\}$ and the decomposition satisfying Prop.~\ref{prop:decomp} is given by,
\begin{align}
\label{eq_A3_start}
    V_{X} & \deq \E_\bfx\V\E[\hat{y}| X], \\
    V_{\bfe} & \deq \E_\bfx\V\E[\hat{y}| \bfe], \\
    V_{P} & \deq \E_\bfx\V\E[\hat{y}| P], \\
    V_{X\bfe} & \deq \E_\bfx\V\E[\hat{y}| X, \bfe] - \E_\bfx\V\E[\hat{y}| X] - \E_\bfx\V\E[\hat{y}| \bfe] ,\\
    V_{PX} & \deq \E_\bfx\V\E[\hat{y}| P, X] -\E_\bfx \V\E[\hat{y}| X] - \E_\bfx\V\E[\hat{y}| P] ,\\
    V_{P\bfe} & \deq \E_\bfx\V\E[\hat{y}|X, \bfe] - \E_\bfx\V\E[\hat{y}| \bfe] -\E_\bfx \V\E[\hat{y}| P] , \\
    V_{P X\bfe} & \deq \E_\bfx\V\E[\hat{y}|P, X,  \bfe] - \E_\bfx\V\E[\hat{y}| X, \bfe] - \E_\bfx\V\E[\hat{y}| P,X] - \E_\bfx\V\E[\hat{y}| X,\bfe]\nonumber \\\label{eq_A3_end}
    &\quad  +\E_\bfx \V\E[\hat{y}| X] + \E_\bfx\V\E[\hat{y}| \bfe] + \E_\bfx\V\E[\hat{y}| P].
\end{align}
\end{example}
\begin{remark}
\label{rem:areas}
Because $V_s \ge 0$ and $V = \V[\hat{y}] = \sum_s V_s$, the subset-sum relation~(\ref{eq:subset_sum}) yields an interpretation of $V$ as the union of disjoint areas, forming a Venn diagram. See Fig.~\ref{fig:venn_variance}(d,e). The reader may also recognize the quantities above as those that are estimated in a three-way ANOVA.
\end{remark}

\section{Asymptotic Variance Decomposition for Random Feature Regression}
\label{sec:model}
\paragraph{Problem setup and notation.}
Following prior work modeling double descent~\cite{hastie2019surprises,mei2019generalization,adlamneural}, we perform our analysis in the high-dimensional asymptotic scaling limit in which the dataset size $m$, feature dimensionality $n_0$, and hidden layer size $n_1$ all tend to infinity at the same rate, with $\phi \deq n_0/m$ and $\psi \deq n_0/n_1$ held constant.

We consider the task of learning an unknown function from $m$ independent samples $(\bfx_i, y_i) \in \mathbb{R}^{n_0}\times \mathbb{R},\,i=1,\ldots, m$, where the datapoints are standard Gaussian, $\bfx_i \sim \cal{N}(0, I_{n_0})$, and the labels are generated by a linear function parameterized by $\beta\in\mathbb{R}^{n_0}$, whose entries are drawn independently from $\mathcal{N}(0,1)$. Concretely, we let
\eq{ 
    y(\bfx_i) = \beta^\top \bfx_i / \sqrt{n_0} + \bfe_i\,,
}
where $\bfe_i \sim \cal{N}(0, \sigma_{\bfe}^2)$ is additive label noise on the training points, yielding a signal-to-noise ratio $\text{SNR} = \sigma_{\bfe}^{-2}$. Although this may seem like a simple data distribution, it turns out that, in these high-dimensional asymptotics, the much more general setting in which the labels are produced by a non-linear teacher neural network can be exactly modeled with a linear teacher of this form (see Sec.~\ref{subsec_setup_sm}).

We consider predictive functions $\hat{y}$ defined by approximate kernel ridge regression using the random feature model\footnote{See the SM for an extension to the Neural Tangent Kernel of a single-hidden-layer neural network \cite{jacot2018neural}.} of \cite{neal1996priors, rahimi2008random}, for which the random features are given by a single-layer neural network with random weights. Specifically, we define the random features on the training set $X = [\bfx_1,\ldots,\bfx_m]$ and test point $\bfx$ to be
\begin{equation}
 F\deq\fs(W_1X/\sqrt{n_0})\,\quad\text{and}\quad f\deq\fs(W_1\bfx/\sqrt{n_0})\,,
\end{equation}
for a weight matrix $W_1 \in \mathbb{R}^{n_1\times n_0}$ with iid entries $[W_1]_{ij} \sim \mathcal{N}(0,1)$\footnote{Any non-zero variance $\sigma_{W_1}^2$ can be absorbed into a redefinition of $\fs$.}. The kernel induced by these random features is 
\eq{\label{eq_K2}
    K(\bfx_1,\bfx_2)\deq \frac{1}{n_1} \fs(W_1\bfx_1/\sqrt{n_0})^\top\fs(W_1\bfx_2/\sqrt{n_0})\,,
}
and the model's predictions are given by
\eq{
\label{eq_yhat_K}
\hat{y}(\bfx) = YK^{-1}K_\bfx\,,
}
where $Y\deq[y(\bfx_1),\ldots,y(\bfx_m)]$, $K\deq K(X,X) + \gamma I_m$, $K_\bfx \deq K(X, \bfx)$, and $\gamma$ is a ridge regularization constant. For this model, $W_1$ plays the role of $\theta$ from Sec.~\ref{sec_bv_for_modern_ml}.

Altogether, the test loss can be written as
\eq{
\label{eq_test_error}
   \Etest = \E_{\beta}\E_{\bfx}(y(\bfx) - \hat{y}(\bfx))^2 = \E_{\bfx}(\beta^\top \bfx/\sqrt{n_0} - Y K^{-1}K_{\bfx})^2\,,
}
where we dropped the outer expectation over $\beta$ because the distribution concentrates around its mean (see the SM).

\subsection{Main Result: Exact Asymptotics for the Fine-Grained Variance Decomposition}
\label{sec:limit}

\label{sec:asymptotic}

\begin{lemma}
\label{lemma:t1t2}
Let $\eta \deq \mathbb{E}[\fs(g)^2]$ and $\zeta \deq (\mathbb{E}[g \sigma(g)])^2$ for $g\sim \mathcal{N}(0,1)$. Then, in the high-dimensional asymptotics defined above, the traces $\tau_1(\gamma) \deq \frac{1}{m}\E\tr(K^{-1})$ and $\tau_2(\gamma) \deq \frac{1}{m}\E \tr(\frac{1}{n_0}X^\top X K^{-1})$ are given by the unique solutions to the coupled polynomial equations,
\begin{equation}
\label{eq:tau1tau2}
  \zeta  \tau_1 \tau_2
   \left(1-\gamma \tau_1\right) = \phi/\psi\left(\zeta  \tau_1 \tau_2 + \phi(\tau_2 -\tau_1) \right) = \left(\tau_1-\tau_2\right) \phi  \left((\eta-\zeta)\tau_1+\zeta \tau_2\right)\,,
\end{equation}
such that $\tau_1, \tau_2\in \mathbb{C}^+$ for $\gamma \in \mathbb{C}^-$.
\end{lemma}
\begin{theorem}
\label{thm:main}
Let $\tau_1$ and $\tau_2$ be defined as in Lemma~\ref{lemma:t1t2}, and use the prime symbol to denote their derivatives with respect to $\gamma$. Then, as $\Im(\gamma)\to0^-$, the asymptotic bias and variance terms of eqns.~\eqref{eq_A3_start}-\eqref{eq_A3_end} are given by
\begin{equation}
\begin{aligned}[c]
 B &= \tau_2^2/\tau_1^2\\
    V_{P} &= \tau_2'/\tau_1' - B\\
    V_{X} &= \phi B (\tau_1-\tau_2)^2 /(\tau_1^2-\phi(\tau_1-\tau_2)^2)\\
    V_{\bfe} &= 0
\end{aligned}
\qquad
\begin{aligned}[c]
V_{PX} &= -\tau_2'/\tau_1^2-B-V_{P}-V_{X}\\
V_{P\bfe} &= 0\\
V_{X\bfe} &=\sigma_{\bfe}^2 V_X/B\\
    V_{PX\bfe} &= \sigma_{\bfe}^2(-\tau_1'/\tau_1^2-1) - V_{X\bfe}\,.
\end{aligned}
\end{equation}
\end{theorem}
\begin{corollary}
\label{cor:interpolation_threshold}
In the ridgeless setting, the bias $B$ is a non-increasing function of the overparameterization ratio $n_1/m = \phi/\psi$. Furthermore, at the interpolation boundary $\psi = \phi$, $V_{PX}$ and $V_{PX\bfe}$ are divergent while the remaining terms are bounded.
\end{corollary}

\section{Fine-Grained Analysis of Double Descent}\label{sec_fine_grained_analysis_dd}
The fine-grained variance decomposition given in Thm.~\ref{thm:main} provides a powerful tool for understanding the origins of double descent. In this section, we use this tool to reinterpret several counterintuitive observations made in prior work and to provide a clear and unambiguous characterization of the source of double descent.

\begin{figure*}[t]
    \centering
\includegraphics[width=\linewidth]{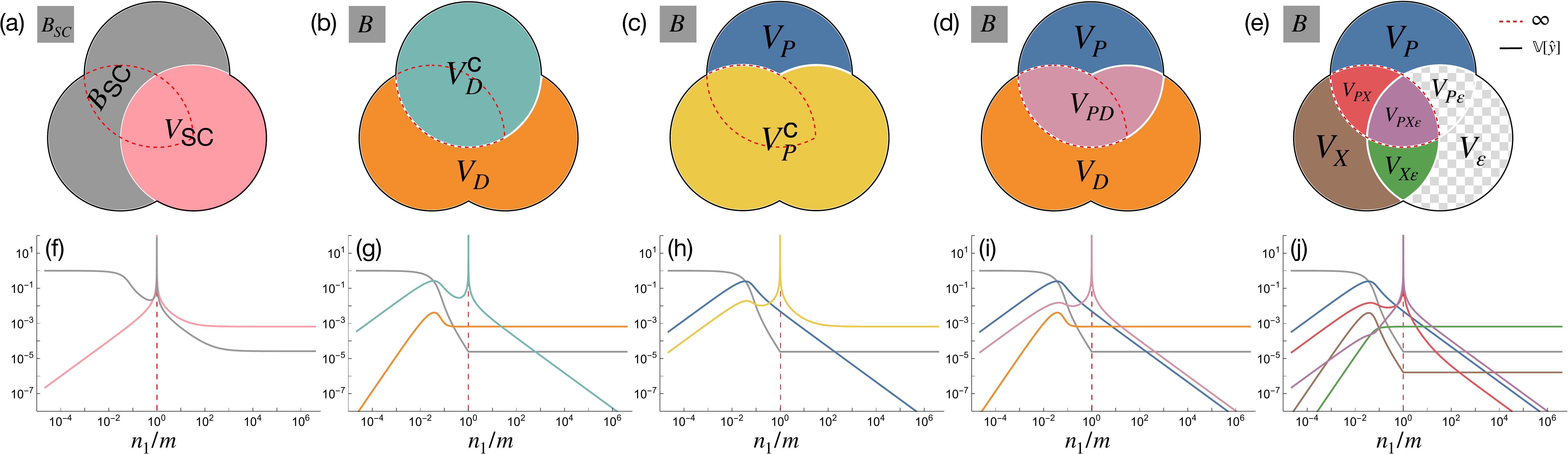}
\caption{(a-e) The different bias-variance decompositions described in Sec.~\ref{sec_fine_grained_analysis_dd}. (f-j) Corresponding theoretical predictions of Thm.~\ref{thm:main} for $\gamma =0$, $\phi=1/16$ and $\fs = \tanh$ with $\text{SNR} = 100$ as the model capacity varies across the interpolation threshold (dashed red). (a,f) The semi-classical decomposition of~\cite{hastie2019surprises,mei2019generalization} has a nonmonotonic and divergent bias term, conflicting with standard definitions of the bias. (b,g) The decomposition of~\cite{neal2018modern} utilizing the law of total variance interprets the diverging term $V_D^\textsc{c}$ as ``variance due to optimization''. (c,h) An alternative application of the law of total variance suggests the opposite, \emph{i.e.} the diverging term $V_P^\textsc{c}$ comes from ``variance due to sampling''. (d,i) A bivariate symmetric decomposition of the variance resolves this ambiguity and shows that the diverging term is actually $V_{PD}$, \emph{i.e.} ``the variance explained by the parameters and data together beyond what they explain individually.'' (e,j) A trivariate symmetric decomposition reveals that the divergence comes from two terms, $V_{PX}$ and $V_{PX\bfe}$ (outlined in dashed red), and shows that label noise exacerbates but does not cause double descent. Since $V_\e=V_{P\e}=0$, they are not shown in (j).}
\label{fig:venn_variance}
\end{figure*}

\subsection{Semi-classical Approach: The Bias Diverges}
In~\cite{hastie2019surprises,mei2019generalization}, double descent in random feature kernel regression was analyzed through the lens of the semi-classical bias-variance decomposition introduced in eqn.~\eqref{eq_bias_variance_theta}. In our setting,
\begin{equation}
    \Etest = B_{SC} + V_{SC}\,,
\end{equation}
where,
\begin{equation}
 B_{SC} = \E_{\bfx}\E_{PX}\pa{\E_{\bfe}[\hat{y}(\bfx)|P,X] -  y(\bfx)}^2\,,\quad\text{and}\quad V_{SC} = \mathbb{E}_{\bfx}\mathbb{E}_{PX}[\mathbb{V}_{\bfe}[\hat{y}|P,X]|\bfx]\,.
\end{equation}

To gain further insight into this decomposition, we can express $B_{SC}$ and $V_{SC}$ in terms of the variables in Thm.~\ref{thm:main}:
\begin{equation}
\label{eq:BSC}
    B_{SC} = B + V_P + V_X + V_{PX} \,,\quad\text{and}\quad V_{SC} = V_{\bfe} + V_{P\bfe}  +  V_{X\bfe} + V_{PX\bfe}\,.
\end{equation}
Using the correspondence between the variance terms and areas mentioned in Remark~\ref{rem:areas}, we illustrate this decomposition in Fig.~\ref{fig:venn_variance}(a). The figure shows that $B_{SC}$ is partially comprised of variance terms. Thm.~\ref{thm:main} allows us to exactly characterize how $B_{SC}$ and $V_{SC}$ depend on the capacity of the model, with results shown in Fig.~\ref{fig:venn_variance}(f). As in~\cite{mei2019generalization}, we observe that the bias $B_{SC}$ and variance $V_{SC}$ exhibit nonmonotonic behavior with respect to the model size and both diverge at the interpolation threshold. 

Because $V_{\bfe} = V_{P\bfe} = 0$ and $V_{X\bfe}$ and $V_{PX\bfe}$ both vanish in the noiseless setting, the semi-classical decomposition has the nice property that $V_{SC} = 0$ when there is no label noise. However, it is hard to reconcile the nonmonotonicity of the bias with its desired interpretation as a measure of the erroneous assumptions in the model as the latter are expected to decrease as the model increases in capacity. For this reason, we believe the multivariate approach outlined in Sec.~\ref{sec_bv_for_modern_ml} provides a more interpretable basis for understanding double descent.

\subsection{Multivariate Approach}
\paragraph{The Law of Total Variance: Ambiguous Conclusions.}
\label{sec:ambig}
Neal \emph{et al.} \cite{neal2018modern} adopt the multivariate approach of Sec.~\ref{sec_bv_for_modern_ml} and decompose the test loss in terms of two sources of randomness, the optimization/initial parameters $P$ and data sampling $D$. The total variance is additionally decomposed according to the law of total variance:
\begin{equation}
\label{eq:vdc}
    V = \underbrace{\E_{\bfx} \V_{D}[\E_{P}[\hat{y}| D]|\bfx]}_{V_D} + \underbrace{\E_{\bfx} \E_{D}[\V_{P}[\hat{y}| D]|\bfx]}_{V^\mathsf{c}_D}\,,
\end{equation}
where Neal \emph{et al.} \cite{neal2018modern} suggests an interpretation for the two terms as ``variance due to sampling'' and ``variance due to optimization,'' respectively. While the expressions in eqn.~\eqref{eq:vdc} are themselves unambiguous, we will see that attributing such an interpretation to them can be somewhat misleading.

Some simple algebra allows us to express $V_D^\textsf{c}$ in terms of the terms in Thm.~\ref{thm:main} as
\eq{\label{eq_neal_v_decomposed}
V_D^\textsf{c} = V_P + V_{PX} + V_{P\bfe} + V_{PX\bfe}\,.
}
Because eqn.~\eqref{eq_neal_v_decomposed} contains $V_{PX}$ and $V_{PX\bfe}$, Corollary~\ref{cor:interpolation_threshold} implies that $V_D^\textsf{c}$ diverges at the interpolation threshold, and indeed we observe that in Fig.~\ref{fig:venn_variance}(g). From the above interpretation of the meaning of $V_D^\textsf{c}$, we might therefore conclude that the ``variance due to optimization'' is the source of double descent.

On the other hand, we could have equally well decided to decompose the variance by conditioning on $P$ instead of $D$, yielding,
\begin{equation}
    V = \underbrace{\E_{\bfx} \V_{P}[\E_{D}[\hat{y}| P]|\bfx]}_{V_P} + \underbrace{\E_{\bfx} \E_{P}[\V_{D}[\hat{y}| P]|\bfx]}_{V^\mathsf{c}_P}\,.
\end{equation}
The corresponding interpretations of these terms would then be ``variance due to optimization'' and ``variance due to sampling,'' respectively. As above, it is straightforward to express $V^\mathsf{c}_P$ as,
\begin{equation}
    V^\mathsf{c}_P = V_X + V_{PX} + V_{X\bfe} + V_{PX\bfe}\,.
\end{equation}
In this case, Corollary~\ref{cor:interpolation_threshold} implies that $V_P^\textsf{c}$ diverges at the interpolation threshold, as Fig.~\ref{fig:venn_variance}(h) confirms. In this case, we might therefore conclude that the ``variance due to sampling'' is the source of double descent.

The above analysis reveals conflicting explanations for the source double descent, depending on which source of randomness is conditioned on when applying the law of total variance. We believe this ambiguity is undesirable and provides further motivation for the symmetric variance decomposition in Prop.~\ref{prop:decomp}.

\paragraph{Bivariate Symmetric Decomposition: $V_{PD}$ is the Source of Divergence.}
In the previous two-variable setting, the symmetric decomposition can be written as (see Example~\ref{ex_2_var}),
\eq{
V = V_P + V_D + V_{PD}\,.
}
See Fig.~\ref{fig:venn_variance}(d) for an illustration of this decomposition. This figure shows that $V_{PD}$ inhabits the ambiguous overlap region that was responsible for the inconsistent interpretations arising from a naive application of the law of total variance. From the theoretical results shown in Fig.~\ref{fig:venn_variance}(i), it is clear that neither the variance explained by the parameters, $V_P$, nor the variance explained by the data, $V_D$, can be responsible for double descent; instead it must be $V_{PD}$ that is causing the divergence. Recalling the definition of $V_{PD}$ in Ex.~\ref{ex_2_var}, we conclude that the divergence at the interpolation boundary is caused by ``the variance explained by the parameters and training data together beyond what they explain individually.'' 

One implication of this interpretation is that if we had a way of removing \emph{either} the variance from the parameters \emph{or} the variance from the data, then the divergence would be eliminated. We examine this phenomenon from the perspective of ensemble and bagging methods in Sec.~\ref{sec:ensemble} and confirm empirically that this is indeed the case. See Fig.~\ref{fig_ens}.

\begin{figure*}[t]
    \centering
    \includegraphics[width=\linewidth]{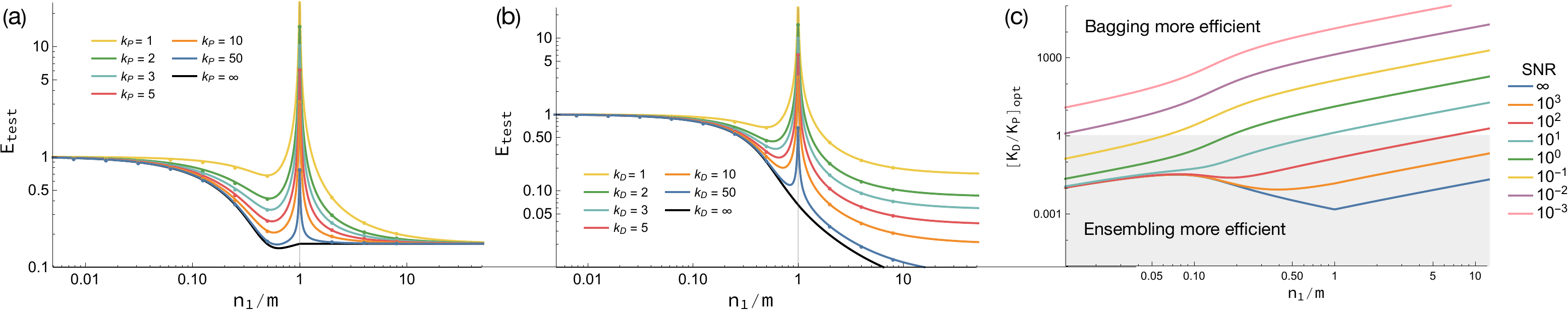}
    \caption{Comparison of (a) ensembles and (b) bagging. Solid lines are theoretical predictions and dots are simulation results. In (a,b) we set $\gamma=10^{-6}$, $n_0=2^{13}$, $m=2^{14}$, $\fs=\tanh$, and $\text{SNR}=5$. Note that as either $k_P$ or $k_D$ increase, the peak around the interpolation threshold decreases. In (c), we plot the optimal ratio $\qa{{k_D/k_P}}_{\text{optimal}}$~\eqref{eq_optimal_ratio} as a function of $n_1/m$ for different SNRs. The shaded area, $\qa{{k_D}/{k_P}}_{\text{optimal}}<1$, is where averaging over the parameters reduces variance more efficiently. As expected, for large width, bagging is much more efficient.}
    \label{fig_ens}
\end{figure*}

\paragraph{Trivariate Symmetric Decomposition: Divergence Persists in Absence of Label Noise.}
Returning to the full model from Sec.~\ref{sec:model} with three sources of randomness, we know from Thm.~\ref{thm:main} that
\eq{
V = V_P + V_X + V_{PX} + V_{X\bfe} + V_{PX\bfe}\,,
}
while the other two variance terms $V_{\bfe}$ and $V_{P\bfe}$ vanish. The seven variance terms are illustrated in Fig.~\ref{fig:venn_variance}(e). The dependence of the five non-zero terms on the model's capacity is plotted in Fig.~\ref{fig:venn_variance}(j). We find that $V_{PX}$ and $V_{PX\bfe}$ both diverge at the interpolation threshold while the other terms remain finite. This result helps explain recent empirical results that have found that label noise amplifies the double descent phenomena~\cite{nakkiran2019deep}: because $V_{PX}$ itself diverges, there is double descent even without label noise, but because $V_{PX\bfe}$ also diverges, label noise can exacerbate the effect.

\section{Ensemble Learning}
\label{sec:ensemble}
The understanding we have developed for the sources of variance enables explicit prediction of the effectiveness of ensemble and bagging techniques. 
We consider averaging the predictive functions of several independently initialized base learners as well as bagging the predictions from models with independent samples of training data. Specifically, we consider $k_P$ independent samples of the parameters, $P_i$, and $k_D$ independent samples of the training data, $X_j$ and $\e_j$. Then our predictive function on a test point $\bfx$ is 
\eq{\label{eq_def_ens_bag}
    \hat{y}^*(\bfx)\deq \frac{1}{k_P k_D} \sum_{i,j} \hat{y}_{ij}(\bfx)\,,
}
where the indicies of $\hat{y}$ indicate the specific sample of parameters and training data used to construct the predictor. A simple calculation gives the variance decomposition of $\hat{y}^*$ as
\al{\label{eq_ensemble_bagging}
    V_P^* &= \frac{V_P}{k_P},~ V_X^* = \frac{V_X}{k_D}, ~ V_\e^* = \frac{V_\e}{k_D}, ~ V_{X\e}^* = \frac{V_{X\e}}{k_D}, \\
     V_{P\e}^* &= \frac{V_{P\e}}{k_Pk_D}, ~ V_{PX}^* = \frac{V_{PX}}{k_Pk_D}, \text{ and } V_{PX\e}^* = \frac{V_{PX\e}}{k_Pk_D}\,,
}
while the bias remains the same. We illustrate these results empirically in Fig.~\ref{fig_ens} and show that ensembles of base learners and bagging are both able to independently reduce the divergence around the interpolation threshold, as they reduce the divergent terms $V_{PX}$ and $V_{PX\e}$.

As the computation of eqn.~\eqref{eq_def_ens_bag} requires evaluating $k_Pk_D$ base learners, it is natural to try to characterize the optimal combination of ensembles and bagging given a fixed computational budget. We find the optimal ratio is given as
\eq{\label{eq_optimal_ratio}
    \qa{k_D/k_P}_\text{optimal} = (V_X + V_{\e} + V_{X\e})/V_P\,.
}
See Fig.~\ref{fig_ens}, which shows that, for the kernel regression problem studied here, ensembles are typically more efficient at small width and bagging is more efficient at large width.

\section{Conclusion}
We analyzed the bias and variance trade-off in the modern setting, where the difference to the classical picture of under- and overfitting is marked. We argued that understanding the behavior of the bias and variance in learning algorithms that depend on large sources of randomness requires rethinking the classical definitions to encompass these sources. 

We presented a bias-variance decomposition that is suitable for these settings, and showed how it can help attribute components of the loss to their causes, while avoiding counterintuitive or ambiguous conclusions. For random feature kernel regression, we gave exact predictions for all of the terms in the decomposition and proved that the bias is monotonically decreasing and identified the source of divergence at the interpolation threshold to be the interaction between the noise from sampling and initialization. We showed that while label noise does not cause the divergence, it can exacerbate the effect. Finally, we made exact predictions for ensemble learning and bagging and provided the computationally optimal strategy to combine them.

\section*{Broader Impact}

While it is hard to envision all future applications of this research, the authors do not believe this theoretical work will raise any ethical concerns or will generate any adverse future societal consequences.

\begin{ack}
We are grateful to Boris Hanin, Jaehoon Lee, Mihai Nica, D. Sculley, Jasper Snoek, and Lechao Xiao for valuable feedback on an earlier version of the paper. We also thank the anonymous reviewers for pointing us to many related works, including the connection to ANOVA.

Funding in direct support of this work came from Google. No third party funding was used.
\end{ack}

\bibliography{references}

\bibliographystyle{unsrtnat}

\input{supplementary.tex}

\end{document}

%% file: supplementary.tex
\newpage
\onecolumn

\newgeometry{verbose,
lmargin=55pt,
rmargin=55pt
}

\setcounter{equation}{0}
\setcounter{figure}{0}
\setcounter{table}{0}
\setcounter{page}{1}
\setcounter{section}{0}

\makeatletter
\renewcommand{\theequation}{S\arabic{equation}}
\renewcommand{\thefigure}{S\arabic{figure}}
\renewcommand{\bibnumfmt}[1]{[S#1]}
\newcommand{\ttau}{\tilde{\tau}}

\setcounter{table}{0}
\renewcommand{\thetable}{S\arabic{table}}
\setcounter{figure}{0}
\renewcommand{\thefigure}{S\arabic{figure}}
\setcounter{section}{0}
\renewcommand{\thesection}{S\arabic{section}}
\setcounter{theorem}{0}
\renewcommand{\thetheorem}{S\arabic{theorem}}
\setcounter{lemma}{0}
\renewcommand{\thelemma}{S\arabic{lemma}}
\setcounter{corollary}{0}
\renewcommand{\thecorollary}{S\arabic{corollary}}
\setcounter{proposition}{0}
\renewcommand{\theproposition}{S\arabic{proposition}}

\newcommand{\thm}[1]{\begin{theorem} #1 \end{theorem}}
\newcommand{\lem}[1]{\begin{lemma} #1 \end{lemma}}
\newcommand{\pf}[1]{\begin{proof} #1 \end{proof}}

\addcontentsline{toc}{section}{Appendix} 
\parttoc 

\section{Symmetric variance decomposition}
\label{sec_var_decomp_sm}
The purpose of this section is to prove Prop.~\ref{prop:decomp} and derive eqns.~\eqref{eq_ensemble_bagging}-\eqref{eq_optimal_ratio}. The strategy is to use the subset-sum relationship, eqn.~\eqref{eq:subset_sum}, as a definition, derive explicit formulae for the variance terms, then prove all terms are nonnegative. In the statistics literature this approach is referred to as functional ANOVA \cite{efron1981jackknife,stone1994use,huang1998projection,owen2003dimension}, but we present a derivation here as it may be unfamiliar to members of the machine learning community. 

\paragraph{Motivation.} The law of total variance for two random variables $X$ and $Y$ is
\eq{
    \V[Y] = \E\V[Y|X] + \V\E[Y|X],
}
where the two terms represents the variance of $Y$ that is unexplained and explained by $X$ respectively. Since the variance must be nonnegative it is possible to interpret it as an area. In Fig.~\ref{fig_venn_1}, the total variance is represented by the square, which is in turn broken up into the explained variance (red circle) and unexplained variance (area outside of the circle).

It is possible to extend this idea to several variables. An observation that is key to the interpretation is 
\eq{\label{eq_superadditive}
    \V\E[Y|X_1] + \V\E[Y|X_2] \leq \V\E[Y|X_1,X_2],
}
\emph{i.e.} the ``variance explained'' is a superadditive function. So the decomposition for two variables could be written as
\eq{
    \V[Y] = \V\E[Y|X_1] + \V\E[Y|X_2] + \pa{\V\E[Y|X_1, X_2] - \V\E[Y|X_1] - \V\E[Y|X_2]} + \E\V[Y|X_1,X_2],
}
with the terms interpreted as the variance explained by $X_1$, the variance explained by $X_2$, the additional variance explained by $X_1$ and $X_2$ together, and the variance left unexplained by $X_1$ and $X_2$. Note that the terms are all guaranteed to be positive by eqn.~\eqref{eq_superadditive}. See Fig \ref{fig_venn_2}.

\paragraph{Several variables.} Generalizing, let $\bfX:=(X_1,\ldots,X_k)$ be a collection of random variables. Consider a Venn diagram of $k$ circles, and denote the disjoint areas using $V_\mathbf{i}$ for a vector $\bfi\in\ha{0,1}^k$, where $i_j$ indicates whether the area is inside the $j$th circle (see Fig \ref{fig_venn}). We make use of the natural partial ordering on $\{0,1\}^k$, \emph{i.e.} $\bfi\leq \bfj$ if and only if $i_l\leq j_l$ for all $l$. Note the ordering indicates the subset relation if the vectors are thought of as indicator vectors. We also use the notation $\bfee_j$ for the standard basis vectors, and define the vectors $\bfX_\bfi:=(X_j: i_j=1)$.

For simplicity, assume $Y\in\sigma(\bfX)$, so that $\E[f(Y)|\bfX]= f(Y)$ for any measurable function $f$ and all the variance of $Y$ is explained by $\bfX$, \emph{i.e.} 
\eq{\label{eq_total_var_explained}
    \V\E[Y|\bfX]= \V[Y]
}
or $V_\mathbf{0}=0$. In fact, let us write $Y=h(\bfX)$. We make this assumption without loss of generality as one can otherwise consider $X_{k+1}\deq Y - \E(Y|\bfX)$, \emph{i.e.} the orthogonal complement of $Y$ under projection onto the sigma algebra generated by $\bfX$.

Consistent with the $k=1$ case, we define
\eq{\label{eq_V_def_special}
    V_{\bfee_j} = \V\E[Y|X_{j}],
}
or more generally
\eq{\label{eq_V_def}
    \sum_{\bfi:\bfi\leq \bfj}V_{\bfi} = \V\E[Y|\bfX_{\bfj}].
}
Eqn.~\eqref{eq_V_def} is exactly the subset-sum relationship in \eqref{eq:subset_sum}. 

\begin{lemma}
    Eqn.~\eqref{eq_V_def} is sufficient to define $V_\bfi$ for all $\bfi$.
\end{lemma}
\begin{proof}
    This lemma follows directly from the fact that \eqref{eq_V_def} defines $2^k$ equations in terms of $2^k$ unknowns, $V_\bfi$. However, we may get a more explicit solution for each $V_\bfi$: We proceed by induction on $|\bfi|:=\sum_j i_j$. The special case of eqn.~\eqref{eq_V_def}, eqn.~\eqref{eq_V_def_special}, proves the base case, $|\bfi|=1$. Assume we have defined $V_\bfi$ for all $|\bfi|\leq m$. Then for $\bfi$ such that $|\bfi|=m+1$, using eqn.~\eqref{eq_V_def} we may write
    \eq{\label{eq_V_recursive}
        V_\bfi + \sum_{\bfj: \bfj\leq \bfi , \bfj\neq \bfi} V_\bfj = \V\E[Y|\bfX_{\bfi}].
    }
    Noting that $|\bfj|\leq m$ if $\bfj\leq \bfi$ and $\bfj\neq \bfi$ completes the proof.
\end{proof}

\paragraph{Calculating Variances}
To calculate the variances terms used above for our model, we use a coupling: we introduce a copy of the underlying random variables $\bfX$ and take expectations under different independence assumptions on $\bfX$ and its copy. The simplest illustration of this idea is to express the variance of a random variable $Y$ using an iid copy of $Y$, denoted $Y'$, then define $\tilde{Y}\deq B Y + (1-B) Y'$ for $B\sim\text{Bern}(1/2)$ independent of $Y$ and $Y'$. We have
\eq{
    \V[Y] = \E[Y^2] - \E[Y]^2 = \E[YY] - \E[YY'] = \E[Y\tilde{Y} | B=1] - \E[Y\tilde{Y}|B=0].
}
This idea extends naturally to our setting. Recall $Y=h(\bfX)$, then let $\tilde{\bfX}:=(B_1X_1+(1-B_1)X_1',\ldots,B_kX_k+(1-B_k)X_k')$ for $\bfX'$ an iid copy of $\bfX$ and $B_i\sim\text{Bern}(1/2)$ iid. Define
\eq{
    H_\bfi := \E \qa{h(\bfX) h(\tilde{\bfX}) | \mathbf{B}=\bfi }.
}
Thus, $\V\E[Y|\bfX_{\bfi}] =  H_\bfi - H_\mathbf{0}$ and $\E\V[Y|\bfX_{\bfi}] = H_\mathbf{1} - H_\bfi$.

\begin{theorem}\label{thm_V_formula}
    Using $H$, we have the following formula for the areas. Let $|\bfi|>0$, then
    \eq{
        V_\bfi \deq \sum_{l=0}^{|\mathbf{i}|} \; \sum_{\bfj: \bfj\leq \bfi,|\bfj|=l} (-1)^{|\bfi|-l} H_\bfj.
    }
\end{theorem}
\begin{proof}
    We again use induction on $|\mathbf{i}|$. The formula clearly holds for $|\mathbf{i}|=1$. Then using the induction hypothesis and eqn.~\eqref{eq_V_recursive}, we see
    \al{
        V_\bfi &= H_\bfi - H_\mathbf{0} - \sum_{\bfj: \bfj\leq \bfi , \bfj\neq \bfi} \sum_{l=0}^{|\bfj|} \; \sum_{\mathbf{k}: \mathbf{k}\leq \bfj,|\mathbf{k}|=l} (-1)^{|\bfj|-l}H_\mathbf{k} \\
        &=\sum_{l=0}^{|\mathbf{i}|} \; \sum_{\bfj: \bfj\leq \bfi,|\bfj|=l} (-1)^{|\bfi|-l} H_\bfj \nonumber.
    }
\end{proof}

\begin{lemma}\label{lem_partially_ordered}
    The function $H$ is partially ordered, that is
    \eq{
        H_\bfi \leq  H_\bfj,
    }
    if and only if $i_k\leq j_k$ for all $k\in\{1,\ldots,k\}$.
\end{lemma}
\begin{proof}
    Define $Z:=\E[Y|X_\bfi]$, then
    \eq{
        H_\bfi - H_\bfj = \E\qa{ Z^2 - \E[Z|X_\bfj]^2 } \geq \E\qa{ Z^2 - \E[Z^2|X_\bfj] } = 0.
    }
\end{proof}

\begin{theorem}\label{thm_positive_As}
    The areas $V_\bfi$ are nonnegative. 
\end{theorem}

\begin{proof}
    The idea is similar to the proof of Lemma \ref{lem_partially_ordered}, and indeed this generalize the result. First we prove eqn.~\eqref{eq_superadditive} to illustrate the idea with simple notation. We see
    \al{
        \V\E[Y|&X_1,X_2]- \V\E[Y|X_1] - \V\E[Y|X_2] \\
        & = H_{11}+H_{00}-H_{01}-H_{10} \nonumber\\
        &= \frac{1}{4}\E\pa{h(X_1,X_2) + h(\tilde{X}_1,\tilde{X}_2) - h(\tilde{X}_1,X_2) - h(X_1,\tilde{X}_2)}^2\nonumber\\
        &\geq0\nonumber.
    }

    For the general case, fix $\bfi$ and define
    \eq{
        \bar{h}(\bfX_\bfi) \deq \E_{\bfX_{\mathbf{1}-\bfi}}[ h(\bfX) | \bfX_{\bfi}],
    }
    that is, marginalize over all $X_j$ such that $i_j=0$. Then note for $\bfj\leq\bfi$ that
    \eq{
        H_\bfj = \E h(\bfX_\bfj, \bfX_{\bfi-\bfj}, \bfX_{\mathbf{1}-\bfi}) h(\bfX_\bfj, \tilde{\bfX}_{\bfi-\bfj}, \tilde{\bfX}_{\mathbf{1}-\bfi}) = \E \bar{h}(\bfX_\bfj, \bfX_{\bfi-\bfj}) \bar{h}(\bfX_\bfj, \tilde{\bfX}_{\bfi-\bfj}).
    }
    Now using Theorem \ref{thm_V_formula}, we see
    \al{
        V_\bfi &= \sum_{l=0}^{|\mathbf{i}|} \; \sum_{\bfj: \bfj\leq \bfi,|\bfj|=l} (-1)^{|\bfi|-l} \E \bar{h}(\bfX_\bfj, \bfX_{\bfi-\bfj}) \bar{h}(\bfX_\bfj, \tilde{\bfX}_{\bfi-\bfj}) \\
        &=\frac{1}{2^{|\bfi|}} \E \pa{ \sum_{\bfj:\bfj\leq\bfi} \bar{h}(\bfX_\bfj, \tilde{\bfX}_{\bfi-\bfj})  }^2 \nonumber\\
        &\geq0 . \nonumber
    }
\end{proof}

\paragraph{Examples for $k=2$ and $k=3$ used in the main text.}

\begin{figure}
    \centering
    \begin{subfigure}[b]{0.3\textwidth}
        \includegraphics[width=\textwidth]{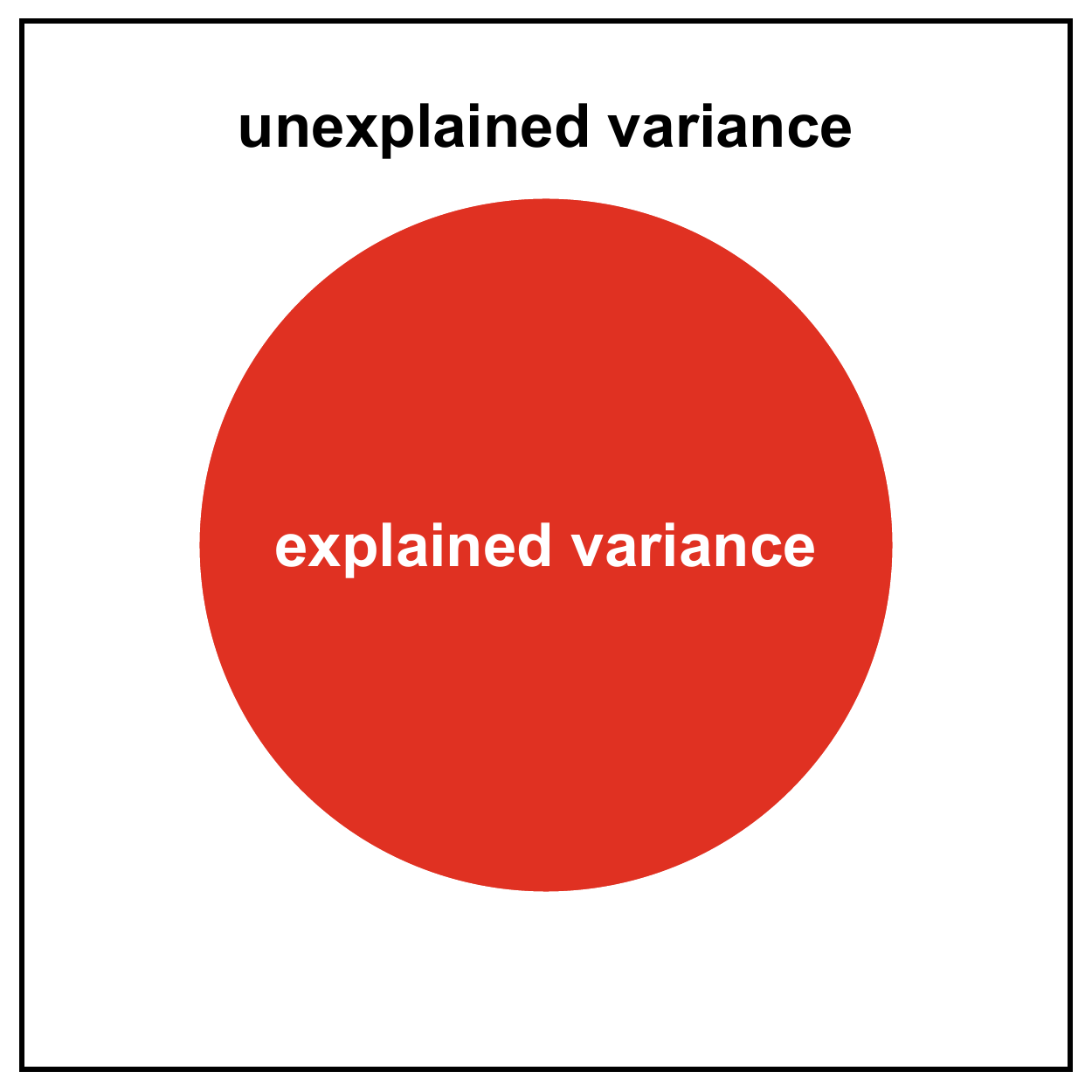}
        \caption{$k=1$}
        \label{fig_venn_1}
    \end{subfigure}
    \begin{subfigure}[b]{0.3\textwidth}
        \includegraphics[width=\textwidth]{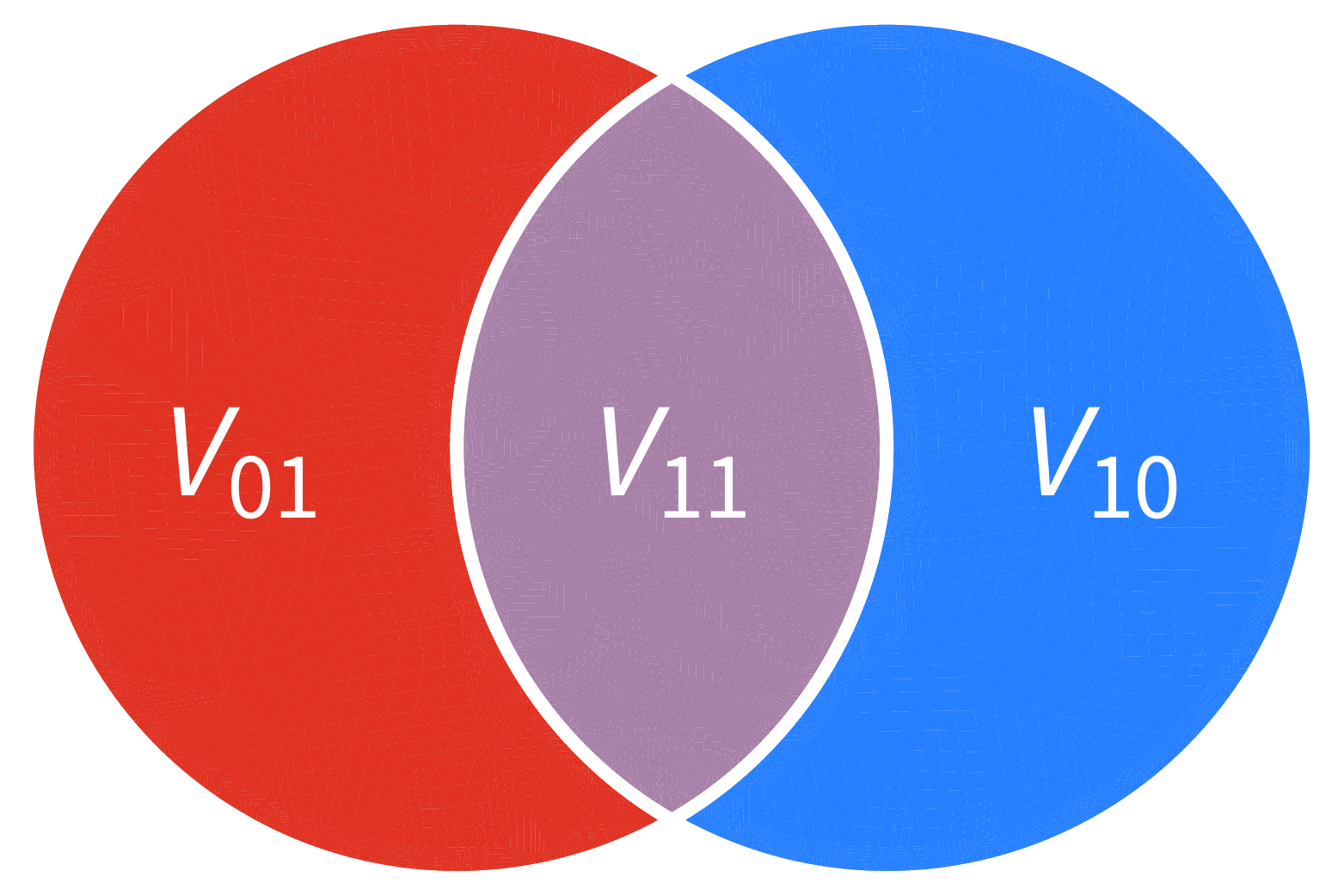}
        \caption{$k=2$}
        \label{fig_venn_2}
    \end{subfigure}
    \begin{subfigure}[b]{0.3\textwidth}
        \includegraphics[width=\textwidth]{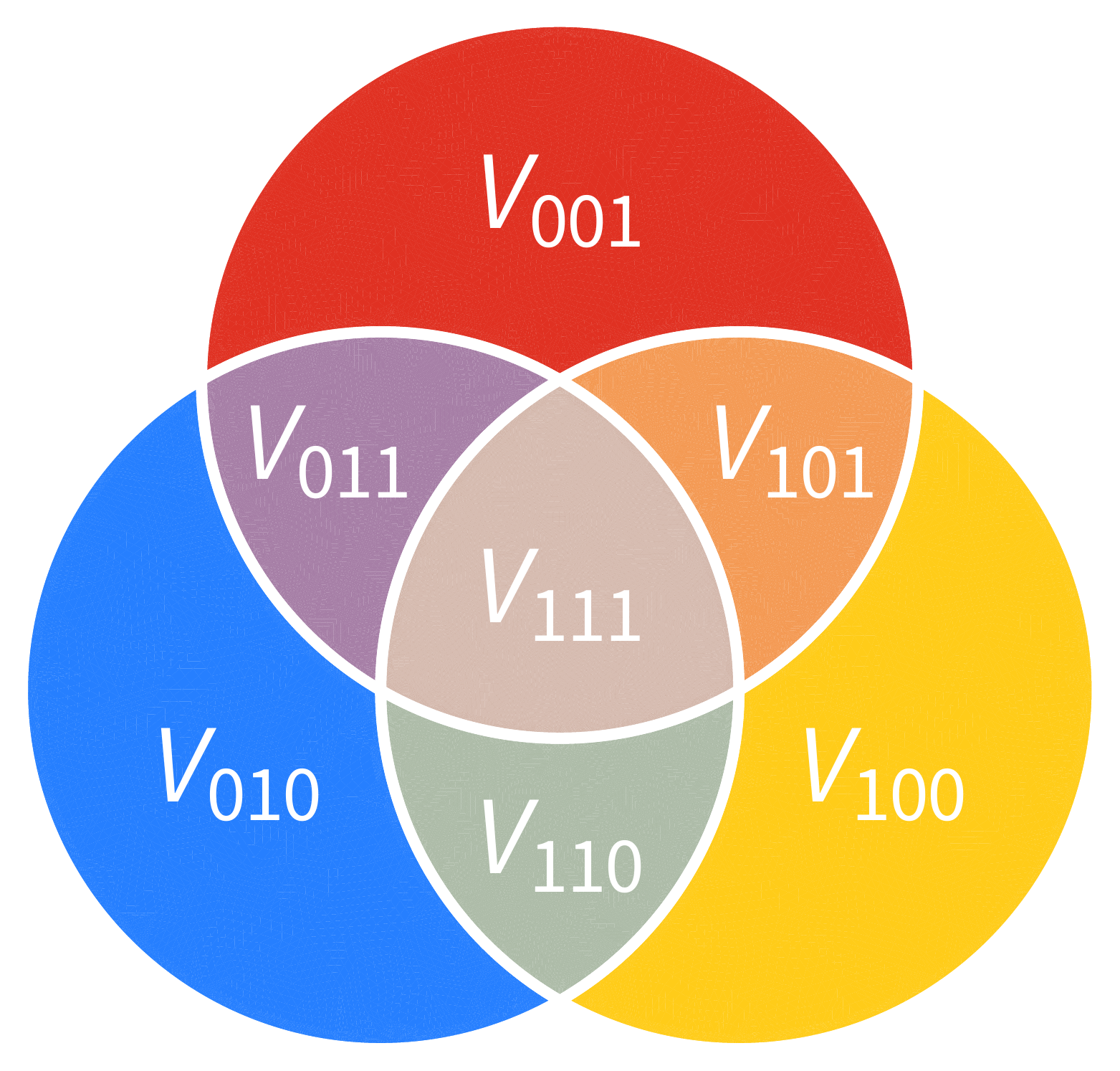}
        \caption{$k=3$}
        \label{fig_venn_3}
    \end{subfigure}
    \caption{In (a) the two disjoint areas represent $\V\E[Y|X]$, the variance of $Y$ explained by $X$, and $\E\V[Y|X]$, the variance of $Y$ unexplained by $X$. For simplicity, we assume that in (b) and (c) there is no variance that is not explained by $\bfX$, so that the area outside of the circles is zero.}
    \label{fig_venn}
\end{figure}

See Fig \ref{fig_venn_2}. For $k=2$, we have:
\begin{align*}
    V_{01} &= H_{10} - H_{00}\\
    V_{10} &= H_{01} - H_{00}\\
    V_{11} &=  H_{11} - H_{01} - H_{10} + H_{00}.
\end{align*}

See Fig \ref{fig_venn_3}. For $k=3$, we have:
\begin{equation}
\label{eq:VH}
\begin{split}
    V_{001} &= H_{001} - H_{000}\\
    V_{010} &= H_{010} - H_{000}\\
    V_{100} &= H_{100} - H_{000}\\
    V_{011} &= H_{011} - H_{001} - H_{010} + H_{000}\\
    V_{101} &= H_{101} - H_{001} - H_{100} + H_{000}\\
    V_{110} &= H_{110} - H_{010} - H_{100} + H_{000}\\
    V_{111} &= H_{111} - H_{011} - H_{101} - H_{110} + H_{001} + H_{010} + H_{100} - H_{000}.
\end{split}
\end{equation}

\paragraph{Ensemble and bagging formulas.} To obtain these results, we first calculate the $H_\bfi$ terms associated with the averaged predictor. Specifically, define $P\deq\ha{P_1,\ldots,P_{k_P}}$, $X\deq\ha{X_1,\ldots,X_{k_D}}$, $\e\deq \ha{\e_1,\ldots,\e_{k_D}}$, and
\eq{
    Y \deq \frac{1}{k_P k_D} \sum_{i,j} \hat{y}_{ij}(\bfx)\,,
}
where the indices denote iid samples. We consider the variance decomposition of $Y$ with respect to $P$, $X$, and $\e$. Note, we could instead use the notation
\eq{
    Y = \hat{y}(P, X, \e) = \frac{1}{k_P k_D} \sum_{i=1}^{k_P}\sum_{j=1}^{k_D} \hat{y}(P_i, X_j, \e_j), 
}
to make explicit the dependence on each of the random variables.

Clearly, $\E \hat{y}(P, X, \e) = \hat{y}(P_1, X_1, \e_1)$, so the predictors have the same bias. Now, we calculate the $H$s using superscripts to denote the ensemble and bagging sizes. First,
\eq{
    H_{000}^{k_Pk_D} = \E \hat{y}(P, X, \e) \hat{y}(\tilde{P}, \tilde{X}, \tilde{\e}) = H_{000}^{11} \,.
}
Next, we see
\al{
    H_{100}^{k_Pk_D} &= \E \hat{y}(P, X, \e) \hat{y}(P, \tilde{X}, \tilde{\e}) \nonumber\\
    &= \frac{1}{k_P^2 k_D^2} \sum_{i=1}^{k_P}\sum_{j=1}^{k_D} \sum_{i'=1}^{k_P}\sum_{j'=1}^{k_D} \E \hat{y}(P_i, X_j, \e_j) \hat{y}(P_{i'}, \tilde{X}_{j'}, \tilde{\e}_{j'})  \nonumber\\ 
    &= \frac{1}{k_P^2 k_D^2} \sum_{i,j,j'} \E \hat{y}(P_i, X_j, \e_j) \hat{y}(P_{i}, \tilde{X}_{j'}, \tilde{\e}_{j'}) + \frac{1}{k_P^2 k_D^2} \sum_{i\neq i'} \sum_{j,j'} \E \hat{y}(P_i, X_j, \e_j) \E\hat{y}(P_{i'}, \tilde{X}_{j'}, \tilde{\e}_{j'}) \nonumber\\ 
    &= \frac{H_{100}^{11} - H_{000}^{11} }{k_P}+ H_{000}^{11} \,.
}
Similarly, to above we find
\eq{
    H_{010}^{k_Pk_D} = \frac{H_{010}^{11}-H_{000}^{11}}{k_D}+ H_{000}^{11}\,,
}
\eq{
    H_{001}^{k_Pk_D} = \frac{H_{001}^{11}-H_{000}^{11}}{k_D}+ H_{000}^{11} \,,
}
and
\eq{
    H_{011}^{k_Pk_D} = \frac{H_{011}^{11}-H_{000}^{11}}{k_D}+ H_{000}^{11} \,.
}
The other terms are more complex, but the idea is the same. We find
\al{
    H_{110}^{k_Pk_D} &= \frac{1}{k_P^2 k_D^2} \sum_{i=1}^{k_P}\sum_{j=1}^{k_D} \sum_{i'=1}^{k_P}\sum_{j'=1}^{k_D} \E \hat{y}(P_i, X_j, \e_j) \hat{y}(P_{i'}, X_{j'}, \tilde{\e}_{j'})  \nonumber\\ 
    &= \frac{1}{k_P^2 k_D^2} \sum_{i}\sum_{j} \E \hat{y}(P_i, X_j, \e_j) \hat{y}(P_{i}, X_{j}, \tilde{\e}_{j}) +\frac{1}{k_P^2 k_D^2} \sum_{i\neq i'}\sum_{j} \E \hat{y}(P_i, X_j, \e_j) \hat{y}(P_{i'}, X_{j}, \tilde{\e}_{j}) \\
    & \quad +\frac{1}{k_P^2 k_D^2} \sum_{i}\sum_{j \neq j'} \E \hat{y}(P_i, X_j, \e_j) \hat{y}(P_{i}, X_{j'}, \tilde{\e}_{j'}) + \frac{1}{k_P^2 k_D^2} \sum_{i\neq i'}\sum_{j\neq j'} \E \hat{y}(P_i, X_j, \e_j) \hat{y}(P_{i'}, X_{j'}, \tilde{\e}_{j'}) \\
    &= \frac{H_{110}^{11} - H_{010}^{11} - H_{100}^{11} + H_{000}^{11}}{k_Dk_P} + \frac{H_{010}^{11}-H_{000}^{11}}{k_P} + \frac{ H_{100}^{11} - H_{000}^{11}}{k_D}  + H_{000}^{11}\,.
    \nonumber\\ 
}
Similarly, we have
\eq{
    H_{101}^{k_Pk_D} = \frac{H_{101}^{11} - H_{001}^{11} - H_{100}^{11} + H_{000}^{11}}{k_Dk_P} + \frac{H_{001}^{11}-H_{000}^{11}}{k_P} + \frac{ H_{100}^{11} - H_{000}^{11}}{k_D}  + H_{000}^{11}
}
and
\eq{
    H_{111}^{k_Pk_D} = \frac{H_{111}^{11} - H_{011}^{11} - H_{100}^{11} + H_{000}^{11}}{k_Dk_P} + \frac{H_{011}^{11}-H_{000}^{11}}{k_P} + \frac{ H_{100}^{11} - H_{000}^{11}}{k_D}  + H_{000}^{11}.
}
Finally, substituting the expressions for the $H$s into eqn.~\eqref{eq:VH} and simplifying completes the derivation.

To find the optimal ratio, we write the test error as
\eq{\label{eq_total_error_for_average}
    B + \frac{V_P}{k_P} + \frac{V_X}{k_D} + \frac{V_\e}{k_D} + \frac{V_{X\e}}{k_D} + \frac{V_{PX}}{k_Pk_D} + \frac{V_{P\e}}{k_Pk_D}+ \frac{V_{PX\e}}{k_Pk_D}
}
and substitute $k_D = K/k_P$, where $K$ is a fixed constant. Then differentiating eqn.~\eqref{eq_total_error_for_average} with respect to $k_P$ and solving for the stationary point yields eqn.~\eqref{eq_optimal_ratio}.

\section{Model Definitions for the Full Neural Tangent Kernel}
\label{sec_sm_ntk}
For clarity of presentation, in the main text we focused on a linear teacher and a simple unstructured random feature model. This model can also be viewed as a degeneration of the Neural Tangent Kernel (NTK) of a single-hidden-layer neural network under which the first-layer weights are held at their randomly-initialized values and only the second-layer weights are optimized. Our analysis and results actually extend to the full NTK, where all weights are optimized, and to a wide nonlinear teacher neural network. The results in the main text are special cases of the more general results we present here.

\subsection{Data distribution}
\label{subsec_setup_sm}
Following~\cite{adlamneural}, we consider the task of learning an unknown function from $m$ independent samples $(\bfx_i, y_i) \in \mathbb{R}^{n_0}\times \mathbb{R},\,i\le m$, where the datapoints are standard Gaussian, $\bfx_i \sim \cal{N}(0, I_{n_0})$, and the labels are generated by a wide\footnote{We assume the width $n_\textsc{t} \to \infty$, but the rate is not important.} single-hidden-layer neural network: 
\eq{
    y_i | \bfx_i, \Omega,\omega \sim \omega \ft(\Omega \bfx_i / \sqrt{n_0}) / \sqrt{n_\textsc{t}} + \varepsilon_i\,.
}
The teacher's activation function $\ft$ is applied coordinate-wise, and its parameters $\Omega \in \mathbb{R}^{n_\textsc{t}\times n_0}$ and $\omega \in \mathbb{R}^{1\times n_\textsc{t}}$ are matrices whose entries are independently sampled once for all data from $\cal{N}(0,1)$. We also allow for independent label noise, $\varepsilon_i \sim \cal{N}(0, \sigma_{\varepsilon}^2)$. In this case, the test loss for a predictive function $\hat{y}$ becomes,
\eq{\label{eq_test_error_sm}
     \E(\omega \ft(\Omega \bfx/\sqrt{n_0})/\sqrt{n_\textsc{t}} + \varepsilon - \hat{y}(\bfx))^2\,.
}

Recall that in our high-dimensional asymptotics the limiting ratios $n_0/m\to\phi$ and $n_0/n_1\to\psi$ are constant.
As we will discuss in Sec.~\ref{sec_gaussian}, in this regime only linear functions of the data can be learned, a finding that is consistent with observations made in~\cite{ghorbani2019linearized,mei2019generalization}. When the teacher width $n_\textsc{t}\to\infty$, a precise decomposition of the teacher emerges that neatly captures its learning and unlearnable components. Specifically, if we define,
\begin{equation}
\label{eq:teacher_consts}
    \zeta_\textsc{t} \deq (\mathbb{E}\ft'(g))^2\,,\;\;\;\,\text{and}\;\;\;\, \eta_\textsc{t} \deq \mathbb{E} \ft(g)^2\,,
\end{equation}
then there is an equivalent linear teacher plus noise with signal-to-noise ratio given by,
\eq{\label{eq_SNR}
\text{SNR} = {\zeta_\textsc{t}}/\pa{\eta_\textsc{t} -\zeta_\textsc{t} + \sigma_{\bfe}^2}\,.
}
We often make this equivalence to a linear teacher explicit by setting $\ft(x) = x$ (which implies $\eta_\textsc{t} = \zeta_\textsc{t} = 1$) and explicitly adding label noise $\sigma_{\varepsilon}^2 = 1/\text{SNR}$. This procedure also removes the noise from the test label, but since this noise merely contributes an additive shift to the test loss, removing it does not change any of our conclusions.

\subsection{NTK Regression}
\label{subsec_kernel_regression_sm}

We consider predictive functions $\hat{y}$ defined by approximate (\emph{i.e.} random feature) kernel ridge regression using the NTK of a single-hidden-layer neural network of width $n_1$ with entry-wise activation function $\fs$, defined by,
\eq{\label{eq_N0}
N_0(\bfx) = W_2\fs(W_1\bfx/\sqrt{n_0})/\sqrt{n_1}\,,
}
for initial $n_1\times n_0$ and $1\times n_1$ weight matrices with iid entries $[W_1]_{ij} \sim \mathcal{N}(0,1)$\footnote{Any non-zero $\sigma_{W_1}^2$ can be absorbed into a redefinition of $\fs$.} and $[W_2]_{i} \sim \mathcal{N}(0,\sigma_{W_2}^2)$.

The NTK can be considered a kernel $K$ that is approximated by random features corresponding to the Jacobian $J$ of the network's output with respect to its parameters, \emph{i.e.} $K(\bfx_1,\bfx_2) = J(\bfx_1)J(\bfx_2)^\top$. The Jacobian itself naturally decomposes into the Jacobian with respect to $W_1$ and $W_2$, \emph{i.e.} $J(\bfx) = [\partial N_0(\bfx)/\partial W_1, \partial N_0(\bfx)/\partial W_2] = [J_1(\bfx), J_2(\bfx)]$. Therefore the kernel $K$ also decomposes this way, and we can write.
\eq{
K(\bfx_1, \bfx_2) =  \; J_1(\bfx_1)J_1(\bfx_2)^\top +  J_2(\bfx_1)J_2(\bfx_2)^\top \deqrev  \;  K_1(\bfx_1, \bfx_2) + K_2(\bfx_1, \bfx_2).
\label{eqn_K_K1K2}
}
As the width of the network becomes very large (compared to all other relevant scales in the system), the approximate NTK converges to a constant kernel determined by the network's initial parameters and describes the trajectory of the network's output under gradient descent.\footnote{If the width is not asymptotically larger than the dataset size, the kernel system may not accurately describe the late-time predictions of the neural network.} In this work, we focus on the predictive function defined by the solution to this kernel regression problem,
\eq{\label{eq_yhat_ntk}
\hat{y}(\bfx) \deq N_0(\bfx) + (Y - N_0(X))K^{-1}K_\bfx\,
}
for $K\deq K(X,X) + \gamma I_m$, $K_\bfx \deq K(X, \bfx)$, and $\gamma$ is a ridge regularization constant. A simple calculation yields the per-layer constituent kernels,
\begin{align}
\label{eq_K1}
K_1(\bfx_1, \bfx_2) & = \frac{X^\top X}{n_0} \odot \frac{\pa{F'}^\top \diag(W_2)^2 F'}{n_1} \quad \text{and}\\
\label{eq_K2_sm}
K_2(\bfx_1, \bfx_2) &= \frac{1}{n_1}F^\top F\,,
\end{align}
where we have introduced the abbreviations $F\deq\fs(W_1X/\sqrt{n_0})$ and $F'\deq\fs'(W_1X/\sqrt{n_0})$. Notice that when $\sigma_{W_2}^2 \to 0$, $K = K_2$, \emph{i.e.} the NTK degenerates into the standard random features kernel of the main text. 

\paragraph{Centering} The predictive function \eqref{eq_yhat_ntk} contains an offset $N_0(\bfx)$ which would typically be set to zero in standard random feature kernel regression because it simply increases the variance of test predictions. Removing this variance component has an analogous operation in neural network training: either the function value at initialization can be subtracted throughout training, or a symmetrization trick can be used in which two copies of the neural network are initialized identically, and their normalized difference $N \equiv \pa{N^{(a)}-N^{(b)}}/{\sqrt{2}}$ is trained with gradient descent. Either method preserves the kernel $K$ while enforcing $N_0 \equiv 0$. We call this procedure \emph{centering}, and present results with and without it.

Finally, we note that ridge regularization in the kernel perspective corresponds to using L2 regularization of the neural network's weights toward their initial values.

\subsection{Exact Asymptotics for the Fine-Grained Variance Decomposition of the NTK}
\ben{Clarify what is meant by asymptotic here. What type of convergence are we talking about?}
Here we state a generalization of the results from Sec.~\ref{sec:limit} to the NTK. As discussed above, the results for random feature kernel regression follow by setting $\sigma_{W_2} = 0$. The proofs are presented in the subsequent sections.

\paragraph{High-dimensional asymptotics.}
We consider the limiting behavior of tracial expressions as the dimensions in our model diverge to infinity as their ratios are held fixed according to $\phi$ and $\psi$. The tracial expressions are random variables that converge in probability to deterministic constants, which are specified as the solution to a coupled equation defined below. 

\begin{lemma}
\label{lemma:t1t2_ntk}
Let $g\sim \mathcal{N}(0,1)$ and define,
\eq{\label{eq_f_consts}
    \zeta \deq (\mathbb{E}\fs'(g))^2\,,\;\;\;\,\eta \deq \mathbb{E}\fs(g)^2\,,\quad\text{and}\quad\eta' \deq \mathbb{E} \fs'(g)^2 \,.
}
Then, in the high-dimensional asymptotics defined above, the limits of the traces $\tau_1(\gamma) = \frac{1}{m}\E \tr(K^{-1})$ and $\tau_2(\gamma) = \frac{1}{m} \E\tr(\frac{1}{n_0}X^\top X K^{-1})$ converge in probability to the unique solutions to the coupled polynomial equations,
\begin{align}
\label{eq:tau_ntk}
    0&=\phi\left(\zeta  \tau_2 \tau_1+ \phi(\tau_2 -\tau_1) \right)+\zeta  \tau_1 \tau_2
  \psi  \left(\gamma \tau_1-1\right) +\zeta  \tau_1 \tau_2 \sigma _{W_2}^2
  \left(\zeta  \left(\tau_2-\tau_1\right) \psi +\tau_1 \psi  \eta '+\phi \right)\\
  0&=\zeta \tau_1^2 \tau_2 \left(\eta'-\eta\right) \sigma _{W_2}^2+\zeta  \tau_1 \tau_2 \left(\gamma \tau_1 -1\right) - \left(\tau_2-\tau_1\right) \phi  \left(\zeta  \left(\tau_2-\tau_1\right)+\eta  \tau_1\right)\,.
\end{align}
such that $\tau_1, \tau_2\in \mathbb{C}^+$ for $\gamma \in \mathbb{C}^+$.
\end{lemma}
\begin{corollary}
Lemma~\ref{lemma:t1t2} follows from Lemma~\ref{lemma:t1t2_ntk} by setting $\sigma_{W_2} = 0$.
\end{corollary}

\begin{theorem}
\label{thm:main_ntk}
Let $\tau_1$ and $\tau_2$ be defined as in Lemma~\ref{lemma:t1t2_ntk}. Then the asymptotic bias and variance terms of eqns.~\eqref{eq_A3_start}-\eqref{eq_A3_end} for the NTK are given by,
\begin{equation}
\begin{aligned}[c]
B &= \tau_2^2/\tau_1^2\\
V_{P} &= \tau_2'/\tau_1' - B - \nu T_2/\tau_1'\\
V_{X} &= \phi B (\tau_1-\tau_2)^2 /(\tau_1^2-\phi(\tau_1-\tau_2)^2)\\
V_{\bfe} &= 0
\end{aligned}
\qquad
\begin{aligned}[c]
V_{PX} &= -\tau_2'/\tau_1^2-B-V_{P}-V_{X} + \nu T_2/(\gamma \tau_1)^2\\
V_{P\bfe} &= 0\\
V_{X\bfe} &=\sigma_{\bfe}^2 V_X/B\\
V_{PX\bfe} &= \sigma_{\bfe}^2(-\tau_1'/\tau_1^2-1) - V_{X\bfe}\,,
\end{aligned}
\end{equation}
where 
\eq{
    T_2 \deq \sigma_{W_2}^2\gamma^2\left(\tau_1 + (\sigma_{W_2}^2(\eta'-\zeta) + \gamma)\tau_1'+\sigma_{W_2}^2\zeta \tau_2'\right)\label{eq:T2_2}\,,
}
$\tau_i'$ is the derivative of $\tau_i$ with respect to $\gamma$, and $\nu = 0$ with centering and $\nu = 1$ without it.
\end{theorem}
\begin{corollary}
Theorem~\ref{thm:main} follows from Theorem~\ref{thm:main_ntk} by setting $\sigma_{W_2} = 0$.
\end{corollary}

\subsection{Discussion of Results for the NTK}

We briefly highlight some results for the full version of Theorem~\ref{thm:main_ntk} that are distinct from the special case Theorem~\ref{thm:main}. Note that since the model in eqn.~\eqref{eq_yhat_ntk} corresponds to the full NTK, the model has $n_1(n_0+1)$ parameters. Thus $n_1=m$ does not occur at the interpolation threshold but instead represents a significantly overparameterized model. Previous work has found nonmonotonic behavior in the test loss for the model in eqn.~\eqref{eq_yhat_ntk} at both the interpolation threshold and when $n_1=m$ \cite{adlamneural}. Since $n_1=m$ is far beyond the interpolation threshold, this second occurrence of nonmonotonicity is qualitatively different than double descent behavior. Our variance decomposition sheds light on the source of this second occurrence of nonmonotonic behavior (see Fig.~\ref{fig:venn_variance_sm}). 

We find that none of the variance terms are divergent, but the sources of the nonmonotonicity are $V_P$, $V_{PX}$, and $V_{PX\e}$. Curiously, bagging this predictive function for a large number of dataset samples would remove all other sources of variance except $V_P$. This would have the effect of highlighting the nonmonotonicity in the total variance.

\begin{figure*}[t]
    \centering
\includegraphics[width=\linewidth]{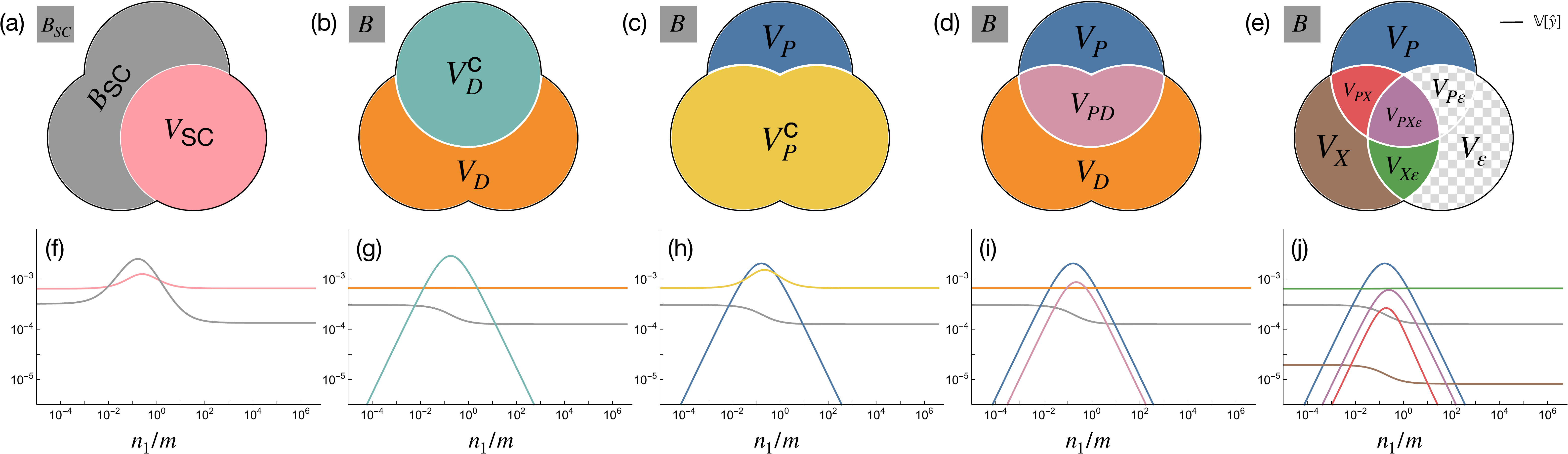}
\caption{We replicate Fig.~\ref{fig:venn_variance} from the main text but using the NTK. As before we set $\gamma =0$, $\phi=1/16$ and $\fs = \tanh$ with $\text{SNR} = 100$, and we use centering.
Recall that the number of trainable parameters for the NTK is $n_1(n_0+1)$, so $n_1=m$ no longer corresponds to the interpolation threshold but represents very overparameterized models. Despite this we still find nonmonotonic behavior in many of the variance terms. Specifically, $V_P$, $V_{PX}$, and $V_{PX\e}$ are all nonmonotonic and have a peak slightly before $n_1=m$. In the semi-classical decomposition (a), these nonmonotonicities would again cause the bias to be nonmonotonic. Similar ambiguities to the random feature case occur for the NTK in (b) and (c). (d) shows the two variable decomposition and (e) the three variable decomposition. As in Fig.~\ref{fig:venn_variance} the terms $V_\e$ and $V_{P\e}$ are zero.}
\label{fig:venn_variance_sm}
\end{figure*}

\section{Gaussian Equivalents}
\label{sec_gaussian}
Here we review the analysis from~\cite{adlamneural} for computing the test loss in our high-dimensional asymptotic limit.  In the next sections, we extend this procedure to compute the constituent bias and variance terms.

As a first step, we exploit some simplifications that happen in our asymptotic limit that allow us to make the following replacements without changing any of the variance terms or the bias:
\begin{align}
K_1 &\to \sigma_{W_2}^2 (\eta'-\zeta) I_m + \frac{\sigma_{W_2}^2 \zeta}{n_0}X^\top X\label{eq_replacement1}\\
F \to & \sqrt{\frac{\zeta}{n_0}} W_1 X + \sqrt{\eta - \zeta} \Theta_F \label{eq_Flin}\\
Y \to & \sqrt{\frac{\zeta_\textsc{t}}{n_\textsc{t} n_0}} \omega \Omega X + \sqrt{\frac{\eta_\textsc{t} - \zeta_\textsc{t}}{n_\textsc{t}}} \omega \Theta_Y + \mathcal{E} \label{eq_Ylin} \\
f \to &\sqrt{\frac{\zeta}{n_0}} W_1 \bfx + \sqrt{\eta - \zeta} \theta_f \label{eq_flin}\\
y \to & \sqrt{\frac{\zeta_\textsc{t}}{n_\textsc{t}n_0}} \omega\Omega \bfx + \sqrt{\frac{\eta_\textsc{t} - \zeta_\textsc{t}}{n_\textsc{t}}} \omega\theta_y \label{eq_ylin}\,.
\end{align}
where $f \deq \fs(W_1\bfx/\sqrt{n_0})$ is the random feature representation of the test point $\bfx$ and $y \deq \omega \ft(\Omega \bfx / \sqrt{n_0}) / \sqrt{n_\textsc{t}}$ is its label. The new objects $\Theta_F$, $\Theta_Y$, $\theta_f$, and $\theta_y$ are matrices of the appropriate shapes with iid standard Gaussian entries. The constants $\eta', \eta$, and $\zeta$ (see eq.~\eqref{eq_f_consts}), as well as $\eta_\textsc{t}$ and $\zeta_\textsc{t}$ (see eqn.~(\ref{eq:teacher_consts})) are chosen so that the mixed moments up to second order are the same for the original and linearized versions. 

To give some intuition on these substitutions, many of the statistics of random matrices are universal, that is, their limiting behavior as the matrix gets larger is insensitive to the detailed properties of their entries' distributions. Considerable work has gone into demonstrating universality for an increasingly large class of random matrices and a growing number of detailed statistics. In our case, the test loss is a global measurement of several random matrices. This perspective gives some intuition for why we are able to replace many of the intractable terms in the expressions we analyze with tractable terms, which only need to match quite superficial properties of the distributions to ensure the limiting test loss is the same.

In Secs.~\ref{sec_exact_asymptotics_train} and \ref{sec_exact_asymptotics}, we use this replacement strategy in two distinct situations. The first is for terms of the form
\eq{\label{eq_con_1}
    \tr(AB) = \sum_{ij}A_{ij}B_{ji},
}
for deterministic $A$ and random $B$. 
Under assumptions on $A$ and $B$, standard concentration inequalities can be used to describe the limiting behavior of sums like eqn.~\eqref{eq_con_1}. In our setting, one finds that this behavior only depends on the the low-order moments of $B$. By matching these low-order moments with Gaussian random variables, we can replace $B$ with a Gaussian random matrix with the same limiting behavior. Note, often $A$ is not actually deterministic, we are simply conditioning on it and only considering the randomness in $B$. The approach is suitable for determining the average behavior of eqn.~\eqref{eq_con_1} when we have control over the (weak) correlations in the entries of $A$ and $B$. Linearizing the matrices $A$ and $B$ in this setting is just a convenient bookkeeping device for performing these computations.

When one of the matrices in eqn.~\eqref{eq_con_1} is inverted, the situation is more complex, and indeed this is the case for the kernel matrix $K$ in expressions for the training and test loss. As in~\cite{adlamneural}, to apply the linear pencil algorithm~\cite{helton2018applications,mingo2017free}, we must first replace the kernels in all expressions with linearized versions (using eqns.~\eqref{eq_replacement1}-\eqref{eq_ylin}), yielding a rational expression of the  i.i.d. Gaussian matrices, $X$, $W_1$, \emph{etc.} 

It should be expected that a linearized version of $F$ will lead to the same asymptotic statistics due to some very general results on the limiting behavior of expressions of the form,
\eq{\label{eq_con_2}
    \tr\pa{A\frac{1}{B-zI}},
}
where $A$ is symmetric and $z\in\C^{+}$. The resolvent matrix $(B-z)^{-1}$ is intimately related to the spectral properties of $B$. 
Recently, isotropic results for quite general $A$ have been developed for matrices with correlated entries, which show that under certain assumptions the limiting behavior of eqn.~\eqref{eq_con_2} depends only on the low-order moments of $B$. Specifically, the limiting behavior of eqn.~\eqref{eq_con_2} is described by the matrix Dyson equation in many cases. For a summary of these results and related topics see e.g.~\cite{erdos2019matrix}. 

Finding Gaussian equivalents for $A$ and $B$ in expressions like eqns.~\eqref{eq_con_1} and \eqref{eq_con_2} is relatively simple in our case. We encounter terms for which the matrix $B$ depends on some other random matrix $C$ through a coordinate-wise nonlinear function $f(C)$. For such cases, Taylor expanding the function $f$ is the key tool to finding these equivalents (see e.g.~\cite{adlam2019random} for more details on this type of approach).

\section{Exact asymptotics for the training loss}
\label{sec_exact_asymptotics_train}
\subsection{Decomposition of terms}
The model's predictions on the training set, $\hat{y}(X)$, take a simple form,
\begin{align}
\hat{y}(X) &= N_0(X) + (Y - N_0(X))K^{-1}K(X,X)\\
&= Y - \gamma (Y-N_0(X))K^{-1}\,.
\end{align}
The training loss can be written as,
\begin{align}
\Etrain &= \frac{1}{m}\mathbb{E}_{(X,Y)}\tr\big((Y - \hat{y}(X))(Y - \hat{y}(X))^\top\big)\\
&= \frac{\gamma^2}{m}\mathbb{E}_{(X,\varepsilon)}\tr\big((Y-N_0(X))^\top (Y-N_0(X))K^{-2}\big)\\
&= \tE_1 + \nu \tE_2
\end{align}
where $\nu=0$ with centering and $\nu=1$ without it and,
\begin{align}
T_1 &\deq \frac{\gamma^2}{m}\mathbb{E}_{\varepsilon}\tr(Y^\top Y K^{-2})\\
T_2 &\deq \frac{\gamma^2}{m}\tr(N_0(X)^\top N_0(X) K^{-2})\,.
\end{align}
We have suppressed the terms linear in $N_0$ since they vanish owing to the linear dependence on the symmetric random variable $W_2$. The Neural Tangent Kernel $K = K(X,X) + \gamma I_m$ and is given by,
\eq{\label{eq_K}
    K = \sigma_{W_2}^2 \qa{(\eta' - \zeta)I_m + \frac{\zeta X^\top\! X}{n_0}} + \frac{F^\top \!F}{n_1} + \gamma I_m\,.
}
Note that $N_0(X)^\top N_0(X) = \sigma_{W_2}^2/n_1 F^T F$, so eqn.~(\ref{eq_K}) gives,
\eq{
    N_0(X)^\top N_0(X) = \sigma_{W_2}^2 K -\sigma_{W_2}^2 \big[\sigma_{W_2}^2\big(\eta' - \zeta\big) + \gamma I_m\big] -\sigma_{W_2}^4 \frac{\zeta X^\top\! X}{n_0}\,.
}
Next we recall the substitution~\eqref{eq_Ylin} (as mentioned above, without loss of generality we special to the case of a linear teacher),
\eq{\label{eq_sub_Y}
Y \to \sqrt{\frac{1}{n_\textsc{t} n_0}} \omega \Omega X + \mathcal{E} \,,
}
and consider the leading order behavior with respect to the random variables $\omega$, $\Omega$, and $W_2$ using eqn.~\eqref{eq_con_1} to find
\begin{equation}
    Y^\top Y = \frac{1}{n_0}X^\top X + \sigma_{\varepsilon}^2 I_m\,.
\end{equation}
Putting these pieces together, we can write for $\tau_1 = \tau_1(\gamma)$ and $\tau_2 = \tau_2(\gamma)$,
\begin{align}
    T_1 &= -\gamma^2(\sigma_\varepsilon^2 \tau_1' + \tau_2')\\
    T_2 &= \sigma_{W_2}^2\gamma^2\left(\tau_1 + (\sigma_{W_2}^2(\eta'-\zeta) + \gamma)\tau_1'+\sigma_{W_2}^2\zeta \tau_2'\right)\label{eq:T2}\,,
\end{align}
where,
\begin{equation}
\label{eqn:tau}
    \tau_1 = \frac{1}{m}\tr(K^{-1})\,,\quad\text{and}\quad\tau_2 = \frac{1}{m}\tr(\frac{1}{n_0}X^\top XK^{-1})\,.
\end{equation}
Self-consistent equations for $\tau_1$ and $\tau_2$ can be computed using the resolvent method, as was done in~\cite{adlam2019random} for the case of $\sigma_{W_2} = 0$. In order to pave the way for the analysis of the test error, we instead demonstrate how to compute these traces using operator-valued free probability.

\begin{remark}
In the remainder of this section, and in Sec.~\ref{sec_exact_asymptotics}, we assume at times that $\sigma$ is non-linear (so that $\eta'>\zeta$ and $\eta > \zeta$) and/or $\gamma > 0$ in order that certain denominator factors are non-zero.  The linear and/or ridgeless cases can be obtained by limits of our general results, or through special cases of the pertinent intermediate formulas.
\end{remark}

\subsection{Linear pencils}
To begin, we construct linear pencils for $\tau_1$ and $\tau_2$. Specifically, straightforward block-matrix inversion confirms that
\begin{equation}
    \tau_1 = \tr([Q_T^{-1}]_{1,1})\,\quad\text{and}\quad \tau_2 = \tr([Q_T^{-1}]_{2,4})\,, 
\end{equation}
where,
\begin{equation}
    Q_T = \left(
\begin{array}{cccc}
 I_m \left(\gamma +\sigma_{W_2}^2 \left(\eta '-\zeta \right)\right) & \frac{\zeta  X^\top \sigma_{W_2}^2}{n_0} & \frac{\sqrt{\eta -\zeta } \Theta_F^\top}{n_1} & \frac{\sqrt{\zeta } X^\top}{\sqrt{n_0} n_1} \\
 -X & I_{n_0} & 0 & 0 \\
 -\sqrt{\eta -\zeta } \Theta_F & -\frac{\sqrt{\zeta } W_1}{\sqrt{n_0}} & I_{n_1} & 0 \\
 0 & 0 & \frac{\sqrt{\zeta } \psi  W_1^\top}{\sqrt{n_0} \phi } & -\frac{\sqrt{\zeta } \psi  I_{n_0}}{\sqrt{n_0} \phi } \\
\end{array}
\right)\,.
\end{equation}
The matrix $Q_T$ is not self-adjoint, but a self-adjoint representation can be obtained from it by doubling the dimensionality. In particular, letting
\eq{
\bar{Q}_T = \begin{pmatrix}
0 & Q_T^\top \\
Q_T & 0
\end{pmatrix}\,,
}
we have,
\eq{
\tau_1 = \tr([\bar{Q}_T^{-1}]_{1,5})\,,\quad\text{and}\quad \tr([\bar{Q}_T^{-1}]_{2,8})\,.
}
Observe that $\bar{Q}_{T}$ is a self-adjoint matrix whose blocks are either constants or proportional to one of $\{X,X^\top,W_1,W_1^\top,\Theta_F,\Theta_F^\top\}$; let us denote the constant terms as $Z$. As such, we can directly utilize the results of~\cite{far2006spectra,mingo2017free} to compute the necessary traces.

\subsection{Operator-valued Stieltjes transform}
The traces can be extracted from the operator-valued Stieltjes transform $G:M_d(\mathbb{C})^+\to M_d(\mathbb{C})^+$, which is a solution of the equation,
\eq{\label{eqn_Geqn_train}
ZG = I_d + \eta(G) G\,,
}
where $d$ is the number of blocks, $\eta: M_d(\mathbb{C})\to M_d(\mathbb{C})$ defined by
\eq{\label{eqn_Deqn}
[\eta(D)]_{ij} = \sum_{kl} \sigma(i,k;l,j) \alpha_k D_{kl} \,,
}
where $\alpha_k$ is dimensionality of the $k$th block and $\sigma(i,k;l,k)$ denotes the covariance between the entries of the blocks $ij$ block of $\bar{Q}$ and entries of the $kl$ block of $\bar{Q}$. Eqn.~\eqref{eqn_Geqn_train} may admit many solutions, but there is a unique solution such that $\text{Im}G \succ 0$ for $\text{Im} Z\succ 0$.

The constants $Z$, the entries of $\sigma$, and therefore the equations~\eqref{eqn_Deqn} are manifest by inspection of the block matrix representation for $\bar{Q}_T$. Although the matrix representation of the equations is too large to reproduce here, we can nevertheless extract the equations satisfied by each entry of $G$.

The equations satisfied by the operator-valued Stieltjes transform $G$ of $\bar{Q}_T$ induce the following structure on $G$,
\eq{
G = \begin{pmatrix} 0 & G_{12} \\ G_{12}^\top & 0 \end{pmatrix}\,,
}\,
where,
\eq{
G_{12} = \left(
\begin{array}{cccc}
 \tau_1 & 0 & 0 & 0 \\
 0 & g_3 & 0 & \tau_2 \\
 0 & 0 & g_4 & 0 \\
 0 & g_6 & 0 & g_5 \\
\end{array}
\right)
}
and the independent entry-wise component functions $g_i$, $\tau_1$ and $\tau_2$ satisfy the following system of polynomial equations,
\begin{align}
0 &= \sqrt{\zeta } g_6 \psi -\zeta  g_3 g_4 \sqrt{n_0}\\
0 &= \sqrt{\zeta } \psi  \big(\tau _2-g_3 \tau _1\big)\\
0 &= \sqrt{\zeta } \psi  \big(g_5-g_6 \tau _1\big)+\sqrt{n_0} \phi\\
0 &= -\zeta  g_4 g_5-g_6 \big(\zeta  \tau _1 \sigma _{W_2}^2+\phi \big)\\
0 &= \sqrt{\zeta } g_5 \psi +\sqrt{n_0} \big(\phi -\zeta  g_4 \tau _2\big)\\
0 &= \phi -g_4 \big(\tau _1 \psi  (\eta -\zeta )+\zeta  \tau _2 \psi +\phi \big)\\
0 &= -\zeta  g_4 \tau _2-g_3 \big(\zeta  \tau _1 \sigma _{W_2}^2+\phi \big)+\phi\\
0 &= -\sqrt{\zeta } g_5 \tau _1 \psi -\sqrt{n_0} \tau _2 \big(\zeta  \tau _1 \sigma _{W_2}^2+\phi \big)\\
0 &= \sqrt{n_0} \big(\phi -g_3 \big(\zeta  \tau _1 \sigma _{W_2}^2+\phi \big)\big)-\sqrt{\zeta } g_6 \tau _1 \psi\\
0 &= \sqrt{n_0} \big(1-\tau _1 \big(\gamma +g_4 (\eta -\zeta )+\sigma _{W_2}^2 \big(\eta '+\zeta  \big(g_3-1\big)\big)\big)\big)-\sqrt{\zeta } g_6 \tau _1 \psi
\,.
\end{align}
It is straightforward algebra to eliminate $g_3,g_4,g_5$ and $g_6$ from the above equations. A simple set of equations for $\tau_1$ and $\tau_2$ follows,
\begin{align}
 0&=\phi\left(\zeta  \tau_2 \tau_1+ \phi(\tau_2 -\tau_1) \right)+\zeta  \tau_1 \tau_2
  \psi  \left(\gamma \tau_1-1\right) +\zeta  \tau_1 \tau_2 \sigma _{W_2}^2
  \left(\zeta  \left(\tau_2-\tau_1\right) \psi +\tau_1 \psi  \eta '+\phi \right)\label{eq:tau1}\\
  0&=\zeta \tau_1^2 \tau_2 \left(\eta'-\eta\right) \sigma _{W_2}^2+\zeta  \tau_1 \tau_2 \left(\gamma \tau_1 -1\right) - \left(\tau_2-\tau_1\right) \phi  \left(\zeta  \left(\tau_2-\tau_1\right)+\eta  \tau_1\right)\label{eq:tau2}\,.
\end{align}
Although these equations admit multiple solutions, the general results of~\cite{far2006spectra,mingo2017free} guarantee that the correct root is given by the unique solutions $\tau_1,\tau_2: \mathbb{C}^+\to\mathbb{C}^+$ which are analytic in the upper half-plane.

It will prove useful to obtain expressions for $\tau_1'(\gamma)$ and $\tau_2'(\gamma)$. By differentiating eqns.~(\ref{eq:tau1}) and~(\ref{eq:tau2}) with respect to $\gamma$, we find
\begin{align}
\tau_1' &= -\frac{\zeta ^2 \tau _2^2 \big(\psi  \tilde{\tau }_1^2-\phi ^2\big)}{\psi  \tilde{\tau }_1^2 \big(\zeta ^2 \big(\tilde{\tau }_2+1\big){}^2+\phi  \big(\zeta  \tilde{\tau }_2+\eta \big) \big(\zeta  \tilde{\tau }_2 \big(2 \tilde{\tau }_2+3\big)+\eta \big)\big)+\zeta ^2 \phi ^2 \big(\tilde{\tau }_2+1\big){}^2 \big(\phi  \tilde{\tau }_2^2-1\big)}\label{eq:dtau1}\\
\tau_2'&=-\frac{\zeta  \tau _2^2 \big(\psi  \tilde{\tau }_1^2 (\zeta -\eta )-\zeta  \phi ^2 \big(\tilde{\tau }_2+1\big){}^2\big)}{\psi  \tilde{\tau }_1^2 \big(\zeta ^2 \big(\tilde{\tau }_2+1\big){}^2+\phi  \big(\zeta  \tilde{\tau }_2+\eta \big) \big(\zeta  \tilde{\tau }_2 \big(2 \tilde{\tau }_2+3\big)+\eta \big)\big)+\zeta ^2 \phi ^2 \big(\tilde{\tau }_2+1\big){}^2 \big(\phi  \tilde{\tau }_2^2-1\big)}\label{eq:dtau2}\,,
\end{align}
where we have introduced some auxiliary variables to ease the presentation,
\begin{equation}
\label{eq:ttau}
    \ttau_1 = \sigma_{W_2}^2 \zeta \tau_2 + \phi \ttau_2\,\quad\text{and}\quad \ttau_2 = -1 + \tau_2/\tau_1\,.
\end{equation}
\section{Exact asymptotics for the test loss}
\label{sec_exact_asymptotics}
\subsection{Decomposition of terms}
The test loss can be written as,
\eq{
\label{eq:Etest}
\Etest = \mathbb{E}_{(\bfx,y)} (y - \hat{y}(\bfx))^2  = E_1 + E_2 + E_3
}
with
\begin{align}
E_1 &= \mathbb{E}_{(\bfx,\varepsilon)}\tr(y(\bfx)y(\bfx)^\top) + \mathbb{E}_{(\bfx,\varepsilon)}\tr(N_0(\bfx)N_0(\bfx)^\top)\\
E_2 &= -2\mathbb{E}_{(\bfx,\varepsilon)}\tr(K_\bfx^\top K^{-1}Y^\top y(\bfx)) - 2\mathbb{E}_{(\bfx,\varepsilon)}\tr(K_\bfx^\top K^{-1}N_0(X)^\top N_0(\bfx))\\
E_3 &= \mathbb{E}_{(\bfx,\varepsilon)}\tr(K_\bfx^\top K^{-1}Y^\top Y K^{-1} K_\bfx) +  \mathbb{E}_{(\bfx,\varepsilon)}\tr(K_\bfx^\top K^{-1}N_0(X)^\top N_0(X) K^{-1} K_\bfx) \,,
\end{align}
where we have suppressed the terms linear in $N_0$ since they vanish owing to the linear dependence on the symmetric random variable $W_2$. The Neural Tangent Kernels $K = K(X,X)$ and $K_\bfx = K(X,\bfx)$ are given by,
\eq{\label{eq_K_and_Kx}
    K = \sigma_{W_2}^2 \qa{(\eta' - \zeta)I_m + \frac{\zeta X^\top\! X}{n_0}} + \frac{F^\top \!F}{n_1} + \gamma I_m\qquad\text{and}\qquad K_\bfx = \frac{\sigma_{W_2}^2 \zeta}{n_0}X^\top \bfx + \frac{1}{n_1} F^\top f\,.
}

\begin{remark}
    In eqn.~\eqref{eq_replacement1}, we argued that the leading order behavior (all that is relevant for the test loss) of $K_1$ is relatively simple, leading to the expression for $K$ in eqn.~\eqref{eq_K_and_Kx}. Implicitly this requires that $\eta'\neq\zeta$, and similarly, in many of the expressions denominators are assumed to be nonzero. We handle degenerate expressions of this kind as special cases, but avoid details here to streamline the presentation.
\end{remark}

Using the cyclicity and linearity of the trace, the expectation over $\bfx$ requires the computation of
\eq{
\mathbb{E}_{\bfx}K_\bfx K_\bfx^\top \,,\qquad \mathbb{E}_{\bfx} y(\bfx)K_\bfx^\top\,,\qquad \mathbb{E}_{\bfx} y(\bfx) y(\bfx)^\top\,,\qquad \mathbb{E}_{\bfx} N_0(\bfx)K_\bfx^\top\,,  \qquad\text{and}\quad \mathbb{E}_{\bfx} N_0(\bfx) N_0(\bfx)^\top\,.
}
As described in Sec.~\ref{sec_gaussian}, without loss of generality we can consider the case of a linear teacher, so that $\eta_\textsc{t} = \zeta_\textsc{t} = 1$ and ~\eqref{eq_ylin} and~\eqref{eq_flin} become
\eq{\label{eq_sub_y_f}
y \to y^{\text{lin}} = \frac{\sqrt{\zeta_\textsc{t}}}{\sqrt{n_0 n_\textsc{t}}} \omega\Omega \bfx + \sqrt{\eta_\textsc{t} - \zeta_\textsc{t}} \frac{1}{\sqrt{n_\textsc{t}}}\omega\theta_y = \frac{
1}{\sqrt{n_0 n_\textsc{t}}}\omega\Omega \bfx\,\qquad\text{and}\qquad f \to f^{\text{lin}} = \frac{\sqrt{\zeta}}{\sqrt{n_0}} W_1 \bfx + \sqrt{\eta - \zeta} \theta_f\, .
}
Using these substitutions, the expectations over $\bfx$ are now trivial and we readily find,
\begin{align}
   \mathbb{E}_\bfx K_\bfx K_\bfx^\top & = \frac{\sigma_{W_2}^4 \zeta^2}{n_0^2} X^\top X + \frac{\sigma_{W_2}^2 \zeta^{3/2}}{n_0^{3/2} n_1}(X^\top W_1^T F + F^\top W_1 X) + \frac{1}{n_1^2}F^\top\big(\frac{\zeta}{n_0} W_1 W_1^\top + (\eta - \zeta) I_{n_1}\big)F\\
\mathbb{E}_{\bfx} y(\bfx)K_\bfx^\top & =   \frac{\sigma_{W_2}^2 \zeta}{n_0^{3/2} \sqrt{n_\textsc{t}}} \omega\Omega X + \frac{\sqrt{\zeta}}{n_0n_1 \sqrt{n_\textsc{t}} } \omega\Omega  W_1^\top F\\
\mathbb{E}_{\bfx} y(\bfx)y(\bfx)^\top &= \frac{1}{n_0 n_\textsc{t}} \omega\Omega \Omega^\top\omega^\top\\
\mathbb{E}_{\bfx} N_0(\bfx)K_\bfx^\top & = \frac{\sigma_{W_2}^2 \zeta^{3/2}}{n_0^{3/2}\sqrt{n_1}} W_2W_1X + \frac{1}{n_1^{3/2}}W_2\big(\frac{\zeta}{n_0}W_1W_1^\top+(\eta-\zeta)I_{n_1}\big)F\\ 
\mathbb{E}_{\bfx} \tr(N_0(\bfx) N_0(\bfx)^\top) & = \sigma_{W_2}^2\eta\,.
\end{align}
One may interpret the substitutions in eqn.~\eqref{eq_sub_y_f} as a tool to calculate the expectations above to leading order as it leads to terms like eqn.~\eqref{eq_con_1}. Next we recall the substitution~\eqref{eq_sub_Y},
\eq{
Y \to \frac{1}{\sqrt{n_0 n_\textsc{t}}} \omega\Omega X + \mathcal{E}\,.
}
As above, we consider the leading order behavior with respect to the random variables $\omega$, $\Omega$, and $W_2$ using eqn.~\eqref{eq_con_1} to find
\begin{align}
    \E_{\omega,\Omega,\mathcal{E}}\qa{Y^\top Y} &= \frac{1}{n_0}X^\top X + \sigma_{\varepsilon}^2 I_m\\
    \E_{\omega,\Omega,\mathcal{E},W_2}\qa{Y^\top \mathbb{E}_{\bfx} y(\bfx)K_\bfx^\top} &= \frac{\sigma_{W_2}^2 \zeta}{n_0^2} X^\top X + \frac{\sqrt{\zeta}}{n_0^{3/2}n_1} X^\top  W_1^\top F\\
    \E_{W_2}\qa{N_0(X)^\top N_0(X)} &= \frac{\sigma_{W_2}^2}{n_1}F^\top F\\
    \E_{W_2}\qa{N_0(X)^\top \mathbb{E}_{\bfx} N_0(\bfx)K_\bfx^\top} &= \frac{\sigma_{W_2}^4 \zeta^{3/2}}{n_0^{3/2}n_1} F^\top W_1X + \frac{\sigma_{W_2}^2}{n_1^2}F^\top\big(\frac{\zeta}{n_0}W_1W_1^\top+(\eta-\zeta)I_{n_1}\big)F\,.
\end{align}
\eq{\label{eq_sub_F}
F \to F^{\text{lin}} = \frac{\sqrt{\zeta}}{\sqrt{n_0}} W_1 X + \sqrt{\eta - \zeta} \Theta_F\,,
}
we can write,
\eq{
\frac{\sqrt{\zeta}}{\sqrt{n_0}} F^\top W_1 X + \frac{\sqrt{\zeta}}{\sqrt{n_0}} X^\top W_1^\top F = F^\top F + \frac{\zeta}{n_0}X^\top W_1^\top W_1 X - (\eta-\zeta)\Theta_F^\top\Theta_F\,.
}
Putting these pieces together, we have
\begin{align}
    E_1 & = 1 + \nu \sigma_{W_2}^2 \eta\label{eqn:E1}\\
    E_2 &= E_{21} + \nu E_{22}\label{eqn:E2}\\
    E_3 &= E_{31} + E_{32} + \nu E_{33}\label{eqn:E3}\,,
\end{align}
where $\nu = 0$ with centering and $\nu = 1$ without it,
\begin{align}
E_{21} & = -\E\tr\bigg(2\frac{\sigma_{W_2}^2 \zeta}{n_0^2}X K^{-1} X^\top + \frac{1}{n_0 n_1} F K^{-1} F^\top + \frac{\zeta}{n_0^2 n_1} W_1 X K^{-1}X^\top W_1^\top - \frac{\eta-\zeta}{n_0 n_1}\Theta_F K^{-1}\Theta_F^\top\bigg)\label{eq_E2}\\
E_{22} & = -\frac{2 \sigma_{W_2}^2}{n_1}\E\tr\bigg( \frac{\sigma_{W_2}^2 \zeta^{3/2}}{n_0^{3/2}} K^{-1}F^\top W_1 X +\frac{\zeta}{n_0 n_1} K^{-1}F^\top W_1W_1^\top F + \frac{\eta-\zeta}{n_1}K^{-1}F^\top F \bigg)\\
E_{31} &= \sigma_{\varepsilon}^2 \E\tr\left(K^{-1} \Sigma_3 K^{-1}\right)\\
E_{32} &= \frac{1}{n_0} \E\tr\left(X K^{-1} \Sigma_3 K^{-1}X^\top\right)\\
E_{33} &= \frac{\sigma_{W_2}^2}{n_1} \E\tr\left(F K^{-1} \Sigma_3 K^{-1}F^\top\right)\,,
\end{align}
and,
\eq{
\Sigma_3 = \frac{\sigma_{W_2}^4 \zeta^2}{n_0^2} X^\top X + \big(\frac{\sigma_{W_2}^2 \zeta}{n_0 n_1} + \frac{\eta-\zeta}{n_1^2}\big)F^\top F + \frac{\zeta}{n_0n_1^2}F^\top W_1 W_1^\top F + \frac{\sigma_{W_2}^2 \zeta^2}{n_0^2 n_1} X^\top W_1^\top W_1 X - \frac{\sigma_{W_2}^2 \zeta(\eta - \zeta)}{n_0 n_1}\Theta_F^\top \Theta_F\,.
}
\subsection{Linear pencils}
Repeated application of the Schur complement formula for block matrix inversion establishes the following representations for $E_{21}, E_{22}, E_{31}, E_{32}, E_{33}.$
\subsubsection{$E_{21}$}
A linear pencil for $E_{21}$ follows from the representation,
\eq{
E_{21} = \tr(U_{21}^T Q_{21}^{-1} V_{21})\,,
}
where,
\begin{align}
    U_{21}^T  &= \left(
\begin{array}{cccccccccccccc}
 0 & -\frac{2 \zeta  I_{n_0} \sigma_{W_2}^2}{n_0} & 0 & 0 & 0 & \frac{(\eta -\zeta ) I_{n_1}}{n_0} & 0 & 0 & 0 & 0 & 0 & -\frac{I_{n_1}}{n_0} & 0 & 0 
\end{array}
\right)\\
V_{21}^T & = 
\left(
\begin{array}{cccccccccccccc}
 0 & 0 & 0 & -\frac{\sqrt{n_0} n_1 I_{n_0}}{\sqrt{\zeta }} & 0 & 0 & 0 & 0 & 0 & I_{n_1} & 0 & 0 & 0 & 0 \\
\end{array}
\right)
\end{align}
and,
\eq{
Q_{21} =
\begin{pmatrix}
Q_{21}^{11} & 0 & 0 \\
0 & Q_{21}^{22} & Q_{21}^{23} \\
0 & 0 & Q_{21}^{33} \\ 
\end{pmatrix}
}
with,
\begin{align}
Q_{21}^{11} &=
\begin{pmatrix}
 I_m \left(\gamma +\sigma_{W_2}^2 \left(\eta '-\zeta \right)\right) & \frac{\zeta  X^\top \sigma_{W_2}^2}{n_0} & \frac{\sqrt{\eta -\zeta } \Theta_F^\top}{n_1} & \frac{\sqrt{\zeta } X^\top}{\sqrt{n_0} n_1}  \\
 -X & I_{n_0} & 0 & 0  \\
 -\sqrt{\eta -\zeta } \Theta_F & -\frac{\sqrt{\zeta } W_1}{\sqrt{n_0}} & I_{n_1} & 0\\
 0 & 0 & -W_1^\top & I_{n_0}
\end{pmatrix}\\
Q_{21}^{22} &=
\begin{pmatrix}
 I_m \left(\gamma +\sigma_{W_2}^2 \left(\eta '-\zeta \right)\right) & 0 & \frac{\zeta  X^\top \sigma_{W_2}^2}{n_0} & \frac{\sqrt{\eta -\zeta } \Theta_F^\top}{n_1} & \frac{\sqrt{\zeta } X^\top}{\sqrt{n_0} n_1}\\
-\Theta_F & I_{n_1} & -\frac{\sqrt{\zeta } W_1}{\sqrt{n_0} \sqrt{\eta -\zeta }} & 0 & 0\\
-X & 0 & I_{n_0} & 0 & 0 \\
-\sqrt{\eta -\zeta } \Theta_F & 0 & -\frac{\sqrt{\zeta } W_1}{\sqrt{n_0}} & I_{n_1}\\
0 & 0 & 0 & -W_1^\top & I_{n_0}\\
\end{pmatrix}\\
Q_{21}^{23} &=
\begin{pmatrix}
-\Theta_F^\top & 0 & 0 & 0 & 0 \\
 0 & 0 & 0 & \frac{\sqrt{\zeta } W_1}{\sqrt{n_0} (\eta -\zeta )} & 0 \\
 0 & 0 & 0 & 0 & 0 \\
 0 & 0 & 0 & 0 & 0 \\
 0 & 0 & 0 & 0 & 0 \\
 I_{n_1} & 0 & 0 & 0 & 0 \\
\end{pmatrix}\\
Q_{21}^{33} &=
\begin{pmatrix}
-\sqrt{\eta -\zeta } \Theta_F^\top & I_m \left(\gamma +\sigma_{W_2}^2 \left(\eta '-\zeta \right)\right) & \frac{\sqrt{\eta -\zeta } \Theta_F^\top}{n_1} & \frac{\zeta  X^\top \sigma_{W_2}^2}{n_0} & \frac{\sqrt{\zeta } X^\top}{\sqrt{n_0} n_1} \\
 0 & -\sqrt{\eta -\zeta } \Theta_F & I_{n_1} & -\frac{\sqrt{\zeta } W_1}{\sqrt{n_0}} & 0 \\
 0 & -X & 0 & I_{n_0} & 0 \\
 n_1 W_1^\top & 0 & -W_1^\top & 0 & I_{n_0}
\end{pmatrix}\,.
\end{align}

\subsubsection{$E_{22}$}
A linear pencil for $E_{22}$ follows from the representation,
\eq{
E_{22} = \tr(U_{22}^T Q_{22}^{-1} V_{22})\,,
}
where,
\begin{align}
    U_{22}^T  &= 
\begin{pmatrix}
0 & -\frac{2 \sqrt{\zeta } I_{n_1} \sigma_{W_2}^2 \left(n_0 (\eta -\zeta )+\zeta  n_1 \sigma_{W_2}^2\right)}{n_0^{3/2} n_1} & 0 & \frac{2 (\zeta -\eta ) I_{n_1} \sigma_{W_2}^2}{n_1} & 0 & 0 & 0
\end{pmatrix}\\
V_{22}^T & = 
\begin{pmatrix}
 0 & 0 & 0 & 0 & 0 & - n_1 I_{n_1} & 0
\end{pmatrix}
\end{align}
and,
\eq{
Q_{22} =
\left(
\begin{smallmatrix}
I_{n_0} & 0 & -X & 0 & 0 & 0 & 0 \\
 -W_1 & I_{n_1} & 0 & 0 & -\frac{\sqrt{n_0} W_1}{\sqrt{\zeta } n_1 \sigma_{W_2}^2} & 0 & 0 \\
 \frac{\zeta  X^\top \sigma_{W_2}^2}{n_0} & 0 & I_m \left(\gamma +\sigma_{W_2}^2 \left(\eta '-\zeta \right)\right) & 0 & 0 & \frac{\sqrt{\eta -\zeta } \Theta_F^\top}{n_1} & \frac{\sqrt{\zeta } X^\top}{\sqrt{n_0} n_1} \\
 0 & 0 & -\sqrt{\eta -\zeta } \Theta_F & I_{n_1} & \frac{W_1}{n_1 \sigma_{W_2}^2} & 0 & 0 \\
 0 & -\frac{\sqrt{\zeta } W_1^\top}{\sqrt{n_0}} & 0 & -W_1^\top & I_{n_0} & 0 & 0 \\
 -\frac{\sqrt{\zeta } W_1}{\sqrt{n_0}} & 0 & -\sqrt{\eta -\zeta } \Theta_F & 0 & 0 & I_{n_1} & 0 \\
 0 & 0 & 0 & 0 & 0 & -W_1^\top & I_{n_0} 
\end{smallmatrix}
\right)\,.
}

\subsubsection{$E_{31}$}
A linear pencil for $E_{31}$ follows from the representation,
\eq{
E_{31} = \tr(U_{31}^T Q_{31}^{-1} V_{31})\,,
}
where,
\eq{
    U_{31}^T  = 
\begin{pmatrix}
m \sigma _ {\varepsilon }^2 I_m & 0 & 0 & 0 & 0 & 0 & 0 & \
0
\end{pmatrix}\,,\quad
V_{31}^T = 
\begin{pmatrix}
 0 & 0 & 0 & 0 & 0 & I_m & 0 & 0
\end{pmatrix}
}
and, for $\beta = \left(n_0 (\zeta -\eta )-\zeta  n_1 \sigma_{W_2}^2\right)$,
\eq{
Q_{31} =
\left(
\begin{smallmatrix}
I_m \left(\gamma +\sigma_{W_2}^2 \left(\eta '-\zeta \right)\right) & \frac{\zeta  X^\top \sigma_{W_2}^2}{n_0} & \frac{\sqrt{\eta -\zeta } \Theta_F^\top}{n_1} & \frac{\sqrt{\zeta } X^\top}{\sqrt{n_0} n_1} & -\frac{\zeta ^2 X^\top \sigma_{W_2}^4}{n_0^2} & 0 & \frac{\sqrt{\eta -\zeta } \Theta_F^\top \beta}{n_0 n_1^2} & \frac{\sqrt{\zeta } X^\top \beta}{n_0^{3/2} n_1^2} \\
 -X & I_{n_0} & 0 & 0 & 0 & 0 & 0 & 0 \\
 -\sqrt{\eta -\zeta } \Theta_F & -\frac{\sqrt{\zeta } W_1}{\sqrt{n_0}} & I_{n_1} & 0 & 0 & -\frac{\zeta  \sqrt{\eta -\zeta } \Theta_F \sigma_{W_2}^2}{n_0} & 0 & \frac{\zeta  W_1}{n_0 n_1} \\
 0 & 0 & -W_1^\top & I_{n_0} & 0 & 0 & \frac{\zeta  W_1^\top \sigma_{W_2}^2}{n_0} & 0 \\
 0 & 0 & 0 & 0 & I_{n_0} & -X & 0 & 0 \\
 0 & 0 & 0 & 0 & \frac{\zeta  X^\top \sigma_{W_2}^2}{n_0} & I_m \left(\gamma +\sigma_{W_2}^2 \left(\eta '-\zeta \right)\right) & \frac{\sqrt{\eta -\zeta } \Theta_F^\top}{n_1} & \frac{\sqrt{\zeta } X^\top}{\sqrt{n_0} n_1} \\
 0 & 0 & 0 & 0 & -\frac{\sqrt{\zeta } W_1}{\sqrt{n_0}} & -\sqrt{\eta -\zeta } \Theta_F & I_{n_1} & 0 \\
 0 & 0 & 0 & 0 & 0 & 0 & -W_1^\top & I_{n_0} 
\end{smallmatrix}
\right)\,.
}
\subsubsection{$E_{32}$}
A linear pencil for $E_{32}$ follows from the representation,
\eq{
E_{32} = \tr(U_{32}^T Q_{32}^{-1} V_{32})\,,
}
where,
\eq{
    U_{32}^T  = 
\begin{pmatrix}
0 & I_{n_0} & 0 & 0 & 0 & 0 & 0 & 0 & 0
\end{pmatrix}\,,\quad
V_{32}^T = 
\begin{pmatrix}
 0 & 0 & 0 & 0 & 0 & 0 & 0 & 0 & -\frac{\sqrt{n_0} n_1 I_{n_0}}{\sqrt{\zeta }}
\end{pmatrix}
}
and, for $\beta = \left(n_0 (\zeta -\eta )-\zeta  n_1 \sigma_{W_2}^2\right)$
\eq{
Q_{32} =
\left(
\begin{smallmatrix}
I_m \left(\gamma +\sigma_{W_2}^2 \left(\eta '-\zeta \right)\right) & 0 & \frac{\zeta  X^\top \sigma_{W_2}^2}{n_0} & \frac{\sqrt{\eta -\zeta } \Theta_F^\top}{n_1} & \frac{\sqrt{\zeta } X^\top}{\sqrt{n_0} n_1} & -\frac{\zeta ^2 X^\top \sigma_{W_2}^4}{n_0^2} & 0 & \frac{\sqrt{\eta -\zeta } \Theta_F^\top \beta}{n_0 n_1^2} & 0 \\
 -X & I_{n_0} & 0 & 0 & 0 & 0 & 0 & 0 & 0 \\
 -X & 0 & I_{n_0} & 0 & 0 & 0 & 0 & \frac{\sqrt{\zeta } W_1^\top}{\sqrt{n_0} n_1} & 0 \\
 -\sqrt{\eta -\zeta } \Theta_F & 0 & -\frac{\sqrt{\zeta } W_1}{\sqrt{n_0}} & I_{n_1} & 0 & 0 & -\frac{\zeta  \sqrt{\eta -\zeta } \Theta_F \sigma_{W_2}^2}{n_0} & 0 & 0 \\
 0 & 0 & 0 & -W_1^\top & I_{n_0} & 0 & 0 & W_1^\top \left(\frac{\eta -\zeta }{n_1}+\frac{\zeta  \sigma_{W_2}^2}{n_0}\right) & 0 \\
 0 & 0 & 0 & 0 & 0 & I_{n_0} & -X & 0 & 0 \\
 0 & 0 & 0 & 0 & 0 & \frac{\zeta  X^\top \sigma_{W_2}^2}{n_0} & I_m \left(\gamma +\sigma_{W_2}^2 \left(\eta '-\zeta \right)\right) & \frac{\sqrt{\eta -\zeta } \Theta_F^\top}{n_1} & \frac{\sqrt{\zeta } X^\top}{\sqrt{n_0} n_1} \\
 0 & 0 & 0 & 0 & 0 & -\frac{\sqrt{\zeta } W_1}{\sqrt{n_0}} & -\sqrt{\eta -\zeta } \Theta_F & I_{n_1} & 0 \\
 0 & 0 & 0 & 0 & 0 & 0 & 0 & -W_1^\top & I_{n_0}
\end{smallmatrix}
\right)\,.
}

\subsubsection{$E_{33}$}
A linear pencil for $E_{33}$ follows from the representation,
\eq{
E_{33} = \tr(U_{33}^T Q_{33}^{-1} V_{33})\,,
}
where,
\begin{align}
    U_{33}^T  & = 
\left(
\begin{array}{ccccccccccc}
 0 & I_{n_1} \sigma_{W_2}^2 & 0 & 0 & 0 & 0 & 0 & 0 & 0 & 0 & 0 \\
\end{array}
\right)\\
V_{33}^T &= 
\left(
\begin{array}{ccccccccccc}
 0 & 0 & 0 & 0 & 0 & 0 & 0 & 0 & 0 & -n_1 I_{n_1} & 0 \\
\end{array}
\right)
\end{align}
and, for $\beta = \left(n_0 (\zeta -\eta )-\zeta  n_1 \sigma_{W_2}^2\right)$,
\eq{
Q_{33} =
\left(
\begin{smallmatrix}
I_m \left(\gamma +\sigma_{W_2}^2 \left(\eta '-\zeta \right)\right) & 0 & 0 & \frac{\zeta  X^\top \sigma_{W_2}^2}{n_0} & \frac{\sqrt{\eta -\zeta } \Theta_F^\top}{n_1} & \frac{\sqrt{\zeta } X^\top}{\sqrt{n_0} n_1} & -\frac{\zeta ^2 X^\top \sigma_{W_2}^4}{n_0^2} & 0 & \frac{\sqrt{\eta -\zeta } \Theta_F^\top \beta}{n_0 n_1^2} & 0 & 0 \\
 -\sqrt{\eta -\zeta } \Theta_F & I_{n_1} & -\frac{\sqrt{\zeta } W_1}{\sqrt{n_0}} & 0 & 0 & 0 & 0 & 0 & 0 & 0 & 0 \\
 -X & 0 & I_{n_0} & 0 & 0 & 0 & 0 & 0 & 0 & 0 & 0 \\
 -X & 0 & 0 & I_{n_0} & 0 & 0 & 0 & 0 & \frac{\sqrt{\zeta } W_1^\top}{\sqrt{n_0} n_1} & 0 & 0 \\
 -\sqrt{\eta -\zeta } \Theta_F & 0 & 0 & -\frac{\sqrt{\zeta } W_1}{\sqrt{n_0}} & I_{n_1} & 0 & 0 & -\frac{\zeta  \sqrt{\eta -\zeta } \Theta_F \sigma_{W_2}^2}{n_0} & 0 & 0 & 0 \\
 0 & 0 & 0 & 0 & -W_1^\top & I_{n_0} & 0 & 0 & W_1^\top \left(\frac{\eta -\zeta }{n_1}+\frac{\zeta  \sigma_{W_2}^2}{n_0}\right) & 0 & 0 \\
 0 & 0 & 0 & 0 & 0 & 0 & I_{n_0} & -X & 0 & 0 & 0 \\
 0 & 0 & 0 & 0 & 0 & 0 & \frac{\zeta  X^\top \sigma_{W_2}^2}{n_0} & I_m \left(\gamma +\sigma_{W_2}^2 \left(\eta '-\zeta \right)\right) & 0 & \frac{\sqrt{\eta -\zeta } \Theta_F^\top}{n_1} & \frac{\sqrt{\zeta } X^\top}{\sqrt{n_0} n_1} \\
 0 & 0 & 0 & 0 & 0 & 0 & -\frac{\sqrt{\zeta } W_1}{\sqrt{n_0}} & -\sqrt{\eta -\zeta } \Theta_F & I_{n_1} & 0 & 0 \\
 0 & 0 & 0 & 0 & 0 & 0 & -\frac{\sqrt{\zeta } W_1}{\sqrt{n_0}} & -\sqrt{\eta -\zeta } \Theta_F & 0 & I_{n_1} & 0 \\
 0 & 0 & 0 & 0 & 0 & 0 & 0 & 0 & 0 & -W_1^\top & I_{n_0} 
\end{smallmatrix}
\right)\,.
}
\subsection{Operator-valued Stieltjes transform}
Even though the individual error terms $E_{21}, E_{22}, E_{31}, E_{32}, E_{33}$ can be written as the trace of self-adjoint matrices, the individual $Q$ matrices are not themselves self-adjoint. However, by enlarging the dimensionality by a factor of two, equivalent self-adjoint representations can easily be constructed. To do so, we simply utilize the identity,
\eq{
U^T Q V = \bar{U}^\top \bar{Q} \bar{V} \equiv
\begin{pmatrix}
\frac{1}{2} U^\top & V^\top
\end{pmatrix}
\begin{pmatrix}
0 & Q^\top \\
Q & 0
\end{pmatrix}
\begin{pmatrix}
\frac{1}{2} U\\
V
\end{pmatrix}\,.
}
Observe that $\bar{Q}_{21},\bar{Q}_{22},\bar{Q}_{31}, \bar{Q}_{32}$ and $\bar{Q}_{33}$ are all self-adjoint block matrices whose blocks are either constants or proportional to one of $\{X,X^\top,W_1,W_1^\top,\Theta_F,\Theta_F^\top\}$; let us denote the constant terms as $Z$. As such, we can directly utilize the results of~\cite{far2006spectra,mingo2017free} to compute the error terms in question.

For each linear pencil, the corresponding error term can be extracted from the operator-valued Stieltjes transform $G:M_d(\mathbb{C})^+\to M_d(\mathbb{C})^+$, which is a solution of the equation,
\eq{\label{eqn_Geqn}
ZG = I_d + \eta(G) G\,,
}
where $d$ is the number of blocks, $\eta: M_d(\mathbb{C})\to M_d(\mathbb{C})$ defined by
\eq{\label{eqn_Deqn2}
[\eta(D)]_{ij} = \sum_{kl} \sigma(i,k;l,j) \alpha_k D_{kl} \,,
}
where $\alpha_k$ is dimensionality of the $k$th block and $\sigma(i,k;l,k)$ denotes the covariance between the entries of the $ij$ block of $\bar{Q}$ and entries of the $kl$ block of $\bar{Q}$. Eqn.~\eqref{eqn_Geqn} may admit many solutions, but there is a unique solution such that $\text{Im}G \succ 0$ for $\text{Im} Z\succ 0$.

The constants $Z$, the entries of $\sigma$, and therefore the equations~\eqref{eqn_Deqn2} are manifest by inspection of the block matrix representations for $Q$. Although the matrix representations are too large to reproduce here, we can nevertheless extract the equations satisfied by each entry of $G$, which we present in the subsequent sections.

\subsubsection{$E_{21}$}
The equations satisfied by the operator-valued Stieltjes transform $G$ of $\bar{Q}  _{21}$ induce the following structure on $G$,
\eq{
G = \begin{pmatrix} 0 & G_{12} \\ G_{12}^\top & 0 \end{pmatrix}\,,
}
where,
\eq{
G_{12} = \left(
\begin{array}{cccccccccccccc}
 g_8 & 0 & 0 & 0 & 0 & 0 & 0 & 0 & 0 & 0 & 0 & 0 & 0 & 0 \\
 0 & g_9 & 0 & g_6 & 0 & 0 & 0 & 0 & 0 & 0 & 0 & 0 & 0 & 0 \\
 0 & 0 & g_{11} & 0 & 0 & 0 & 0 & 0 & 0 & 0 & 0 & 0 & 0 & 0 \\
 0 & g_{12} & 0 & g_{10} & 0 & 0 & 0 & 0 & 0 & 0 & 0 & 0 & 0 & 0 \\
 0 & 0 & 0 & 0 & g_8 & 0 & 0 & 0 & 0 & 0 & 0 & 0 & 0 & 0 \\
 0 & 0 & 0 & 0 & 0 & g_1 & 0 & g_5 & 0 & g_4 & 0 & g_7 & 0 & 0 \\
 0 & 0 & 0 & 0 & 0 & 0 & g_9 & 0 & g_6 & 0 & 0 & 0 & 0 & 0 \\
 0 & 0 & 0 & 0 & 0 & 0 & 0 & g_{11} & 0 & g_3 & 0 & 0 & 0 & 0 \\
 0 & 0 & 0 & 0 & 0 & 0 & g_{12} & 0 & g_{10} & 0 & 0 & 0 & 0 & 0 \\
 0 & 0 & 0 & 0 & 0 & 0 & 0 & 0 & 0 & g_1 & 0 & 0 & 0 & 0 \\
 0 & 0 & 0 & 0 & 0 & 0 & 0 & 0 & 0 & 0 & g_8 & 0 & 0 & 0 \\
 0 & 0 & 0 & 0 & 0 & 0 & 0 & 0 & 0 & g_2 & 0 & g_{11} & 0 & 0 \\
 0 & 0 & 0 & 0 & 0 & 0 & 0 & 0 & 0 & 0 & 0 & 0 & g_9 & g_6 \\
 0 & 0 & 0 & 0 & 0 & 0 & 0 & 0 & 0 & 0 & 0 & 0 & g_{12} & g_{10} \\
\end{array}
\right)\,,
}
and the independent entry-wise component functions $g_i$ combine to produce the error $E_{21}$ through the relation,
\eq{
E_{21} = \frac{g_4 (\eta -\zeta )}{n_0}+\frac{2\sqrt{\zeta } g_6 \sqrt{n_0} \sigma _{W_2}^2}{\psi }-\frac{g_2}{n_0}\,,
}
and themselves satisfy the following system of polynomial equations,
{\small
\begin{subequations}
\begin{align}
0 &= 1-g_1\\
0 &= \sqrt{\zeta } g_9 g_{11} \sqrt{n_0}-g_{12} \psi\\
0 &= \sqrt{\zeta } g_6 g_{11} \sqrt{n_0}-g_{10} \psi +\psi\\
0 &= g_7 (\eta -\zeta )+\sqrt{\zeta } g_6 g_{11} \sqrt{n_0}\\
0 &= g_8 g_{11} n_0 \sqrt{\eta -\zeta }-g_3 \phi  \big(\gamma +\sigma _{W_2}^2 \big(\eta '-\zeta \big)\big)\\
0 &= -\sqrt{\zeta } g_8 g_9 \psi -g_6 \sqrt{n_0} \phi  \big(\gamma +\sigma _{W_2}^2 \big(\eta '-\zeta \big)\big)\\
0 &= -\sqrt{\zeta } g_8 g_{12} \psi -\big(g_{10}-1\big) \sqrt{n_0} \phi  \big(\gamma +\sigma _{W_2}^2 \big(\eta '-\zeta \big)\big)\\
0 &= g_6 \sqrt{n_0} \phi  \big(\gamma +\sigma _{W_2}^2 \big(\eta '-\zeta \big)\big)+g_8 \big(\sqrt{\zeta } g_{10} \psi +\zeta  g_6 \sqrt{n_0} \sigma _{W_2}^2\big)\\
0 &= g_8 g_{11} \psi  (\eta -\zeta )-\phi  \big(g_5 \sqrt{\eta -\zeta }-\sqrt{\zeta } g_6 g_{11} \sqrt{n_0}\big) \big(\sigma _{W_2}^2 \big(\zeta -\eta '\big)-\gamma \big)\\
0 &= \big(g_9-1\big) \sqrt{n_0} \phi  \big(\gamma +\sigma _{W_2}^2 \big(\eta '-\zeta \big)\big)+g_8 \big(\sqrt{\zeta } g_{12} \psi +\zeta  g_9 \sqrt{n_0} \sigma _{W_2}^2\big)\\
0 &= g_1 g_8 n_0 \sqrt{\eta -\zeta }+g_3 \big(g_8 \psi  (\zeta -\eta )+\phi  \big(\sqrt{\zeta } g_6 \sqrt{n_0}-1\big) \big(\gamma +\sigma _{W_2}^2 \big(\eta '-\zeta \big)\big)\big)\\
0 &= \sqrt{\zeta } g_{10} g_{11} \sqrt{n_0} \phi  \big(\sigma _{W_2}^2 \big(\zeta -\eta '\big)-\gamma \big)+g_{12} \psi  \big(\gamma  \phi +\sigma _{W_2}^2 \big(-\zeta  \phi +\phi  \eta '+\zeta  g_8\big)\big)\\
0 &= g_{11} \big(g_8 \psi  (\zeta -\eta )+\phi  \big(\sqrt{\zeta } g_6 \sqrt{n_0}-1\big) \big(\gamma +\sigma _{W_2}^2 \big(\eta '-\zeta \big)\big)\big)+\phi  \big(\gamma +\sigma _{W_2}^2 \big(\eta '-\zeta \big)\big)\\
0 &= g_{11} n_0 \big(g_8 \psi  (\eta -\zeta )+\sqrt{\zeta } g_6 \sqrt{n_0} \phi  \big(\sigma _{W_2}^2 \big(\zeta -\eta '\big)-\gamma \big)\big)-g_2 \psi  \phi  \big(\gamma +\sigma _{W_2}^2 \big(\eta '-\zeta \big)\big)\\
0 &= g_9 \psi  \big(\gamma  \phi +\sigma _{W_2}^2 \big(\phi  \big(\eta '-\zeta \big)+\zeta  g_8\big)\big)-\phi  \big(\sqrt{\zeta } g_6 g_{11} \sqrt{n_0}+\psi \big) \big(\gamma +\sigma _{W_2}^2 \big(\eta '-\zeta \big)\big)\\
0 &= g_8 \big(-\sqrt{\zeta } g_{12} \psi -\sqrt{n_0} \big(\gamma +g_{11} (\eta -\zeta )+\sigma _{W_2}^2 \big(\eta '+\zeta  \big(g_9-1\big)\big)\big)\big)+\sqrt{n_0} \big(\gamma +\sigma _{W_2}^2 \big(\eta '-\zeta \big)\big)\\
0 &= \sqrt{\zeta } g_1 g_6 \sqrt{n_0} \phi  \big(\sigma _{W_2}^2 \big(\zeta -\eta '\big)-\gamma \big)-g_7 (\zeta -\eta ) \big(g_8 \psi  (\zeta -\eta )+\phi  \big(\sqrt{\zeta } g_6 \sqrt{n_0}-1\big) \big(\gamma +\sigma _{W_2}^2 \big(\eta '-\zeta \big)\big)\big)\\
0 &= g_1 n_0 \big(g_8 \psi  (\eta -\zeta )+\sqrt{\zeta } g_6 \sqrt{n_0} \phi  \big(\sigma _{W_2}^2 \big(\zeta -\eta '\big)-\gamma \big)\big)+g_2 \psi  \big(g_8 \psi  (\zeta -\eta )\nonumber\\
&\quad\;+\phi  \big(\sqrt{\zeta } g_6 \sqrt{n_0}-1\big) \big(\gamma +\sigma _{W_2}^2 \big(\eta '-\zeta \big)\big)\big)\\
0 &= g_1 \big(g_8 \psi  (\eta -\zeta )+\sqrt{\zeta } g_6 \sqrt{n_0} \phi  \big(\sigma _{W_2}^2 \big(\zeta -\eta '\big)-\gamma \big)\big)+g_5 \sqrt{\eta -\zeta } \big(g_8 \psi  (\eta -\zeta )\nonumber\\
&\quad\;-\phi  \big(\sqrt{\zeta } g_6 \sqrt{n_0}-1\big) \big(\gamma +\sigma _{W_2}^2 \big(\eta '-\zeta \big)\big)\big)\\
0 &= n_0 \big(-\zeta  g_5 g_8 \psi  \sqrt{\eta -\zeta }+\eta  g_5 g_8 \psi  \sqrt{\eta -\zeta }+g_8 \psi  (\zeta -\eta ) \big(g_7 (\zeta -\eta )-g_1\big)\nonumber\\
&\quad\; +\sqrt{\zeta } g_6 \sqrt{n_0} \phi  \big(g_7 (\zeta -\eta )+g_1\big) \big(\gamma +\sigma _{W_2}^2 \big(\eta '-\zeta \big)\big)\big)+g_4 \psi  \phi  (\zeta -\eta ) \big(\gamma +\sigma _{W_2}^2 \big(\eta '-\zeta \big)\big)\\
0 &= \sqrt{n_0} \sqrt{\eta -\zeta } \big(g_1 g_8 \sqrt{n_0} \psi  (\eta -\zeta )+\sqrt{\zeta } g_1 g_6 n_0 \phi  \big(\gamma +\sigma _{W_2}^2 \big(\eta '-\zeta \big)\big)-\sqrt{\zeta } g_2 g_6 \psi  \phi  \big(\gamma +\sigma _{W_2}^2 \big(\eta '-\zeta \big)\big)\big)\nonumber\\
&\quad\; +g_3 \psi  (\zeta -\eta ) \big(g_8 \psi  (\eta -\zeta )+\sqrt{\zeta } g_6 \sqrt{n_0} \phi  \big(\sigma _{W_2}^2 \big(\zeta -\eta '\big)-\gamma \big)\big)+g_4 \psi  (-\phi ) (\eta -\zeta )^{3/2} \big(\gamma +\sigma _{W_2}^2 \big(\eta '-\zeta \big)\big)\,.
\end{align}
\end{subequations}
}
After some straightforward algebra, one can eliminate all $g_i$ except for $g_6$ and $g_{8}$, which satisfy coupled polynomial equations. Those equations can be shown to be identical to eqn.~(\ref{eqn:tau}) by invoking the change of variables,
\begin{equation}
    g_6 = - \frac{\sqrt{\zeta} \psi}{\sqrt{n_0} \phi} \tau_2\,,\quad\text{and}\quad
    g_8 = \big(\gamma +\sigma _{W_2}^2 \big(\eta '-\zeta \big)\big)\tau_1\,.
\end{equation}
In terms of these variables, the error $E_{21}$ is given by,
\begin{equation}
    E_{21} = 2(\tau_2/\tau_1 - 1)\,.
\end{equation}

\subsubsection{$E_{22}$}
The equations satisfied by the operator-valued Stieltjes transform $G$ of $\bar{Q}_{22}$ induce the following structure on $G$,
\eq{
G = \begin{pmatrix} 0 & G_{12} \\ G_{12}^\top & 0 \end{pmatrix}\,,
}
where,
\eq{
G_{12} = \left(
\begin{array}{ccccccc}
 g_{11} & 0 & 0 & 0 & 0 & 0 & g_7 \\
 0 & g_5 & 0 & g_2 & 0 & g_9 & 0 \\
 0 & 0 & g_{10} & 0 & 0 & 0 & 0 \\
 0 & g_3 & 0 & g_4 & 0 & g_8 & 0 \\
 g_{14} & 0 & 0 & 0 & g_1 & 0 & g_6 \\
 0 & 0 & 0 & 0 & 0 & g_{13} & 0 \\
 g_{14} & 0 & 0 & 0 & 0 & 0 & g_{12} \\
\end{array}
\right)\,,
}
and the independent entry-wise component functions $g_i$ combine to produce the error $E_{22}$ through the relation,
\eq{
E_{22} = \frac{2\sqrt{\zeta } g_9 \sigma _{W_2}^2 \big(\psi  (\eta -\zeta )+\zeta  \sigma _{W_2}^2\big)}{\sqrt{n_0} \psi }+2g_8 (\eta -\zeta ) \sigma _{W_2}^2\,,
}
and themselves satisfy the following system of polynomial equations,
{\small
\begin{subequations}
\begin{align}
0 &= \sqrt{\zeta } g_{11} g_{13} \sqrt{n_0}-g_{14} \psi\\
0 &= \sqrt{\zeta } g_7 g_{13} \sqrt{n_0}-g_{12} \psi +\psi\\
0 &= g_1 \psi  \big(g_3 \sqrt{n_0}-\sqrt{\zeta } g_4\big)-g_3 \sqrt{n_0} \sigma _{W_2}^2\\ 
0 &= -g_1 \psi  \big(\sqrt{\zeta } g_5+g_3 \sqrt{n_0}\big)-g_3 \sqrt{n_0} \sigma _{W_2}^2\\
0 &= g_1 \psi  \big(g_5 \sqrt{n_0}-\sqrt{\zeta } g_2\big)-\sqrt{\zeta } g_2 \sigma _{W_2}^2\\
0 &= g_1 \psi  \big(\sqrt{\zeta } g_2+g_4 \sqrt{n_0}\big)-\sqrt{\zeta } g_2 \sigma _{W_2}^2\\
0 &= g_1 \psi  \big(g_5 \sqrt{n_0}-\sqrt{\zeta } g_2\big)-\big(g_5-1\big) \sqrt{n_0} \sigma _{W_2}^2\\
0 &= -g_1 \psi  \big(\sqrt{\zeta } g_2+g_4 \sqrt{n_0}\big)-\big(g_4-1\big) \sqrt{n_0} \sigma _{W_2}^2\\
0 &= g_1 \psi  \big(g_3 \sqrt{n_0}-\sqrt{\zeta } g_4\big)-\sqrt{\zeta } \big(g_4-1\big) \sigma _{W_2}^2\\
0 &= g_1 \psi  \big(\sqrt{\zeta } g_5+g_3 \sqrt{n_0}\big)-\sqrt{\zeta } \big(g_5-1\big) \sigma _{W_2}^2\\
0 &= -\sqrt{\zeta } g_{10} g_{11} \psi -g_7 \sqrt{n_0} \phi  \big(\gamma +\sigma _{W_2}^2 \big(\eta '-\zeta \big)\big)\\
0 &= -\sqrt{\zeta } g_{10} g_{14} \psi -g_6 \sqrt{n_0} \phi  \big(\gamma +\sigma _{W_2}^2 \big(\eta '-\zeta \big)\big)\\
0 &= -\sqrt{\zeta } g_{10} g_{14} \psi -\big(g_{12}-1\big) \sqrt{n_0} \phi  \big(\gamma +\sigma _{W_2}^2 \big(\eta '-\zeta \big)\big)\\
0 &= g_1 \big(-\zeta  g_2+\sqrt{\zeta } \big(g_5-g_4\big) \sqrt{n_0}+g_3 n_0\big)-\sqrt{\zeta } \big(g_1-1\big) \sqrt{n_0} \sigma _{W_2}^2\\
0 &= g_1 \psi  \big(\sqrt{\zeta } g_9+g_8 \sqrt{n_0}\big)+\sqrt{\zeta } \big(g_7 g_{13} n_0-g_9\big) \sigma _{W_2}^2+g_6 g_{13} \sqrt{n_0} \psi\\
0 &= g_7 \sqrt{n_0} \phi  \big(\gamma +\sigma _{W_2}^2 \big(\eta '-\zeta \big)\big)+g_{10} \big(\sqrt{\zeta } g_{12} \psi +\zeta  g_7 \sqrt{n_0} \sigma _{W_2}^2\big)\\
0 &= \big(g_{11}-1\big) \sqrt{n_0} \phi  \big(\gamma +\sigma _{W_2}^2 \big(\eta '-\zeta \big)\big)+g_{10} \big(\sqrt{\zeta } g_{14} \psi +\zeta  g_{11} \sqrt{n_0} \sigma _{W_2}^2\big)\\
0 &= \sqrt{\zeta } g_{12} g_{13} \sqrt{n_0} \phi  \big(\sigma _{W_2}^2 \big(\zeta -\eta '\big)-\gamma \big)+g_{14} \psi  \big(\gamma  \phi +\sigma _{W_2}^2 \big(-\zeta  \phi +\phi  \eta '+\zeta  g_{10}\big)\big)\\
0 &= g_{13} \big(g_{10} \psi  (\zeta -\eta )+\phi  \big(\sqrt{\zeta } g_7 \sqrt{n_0}-1\big) \big(\gamma +\sigma _{W_2}^2 \big(\eta '-\zeta \big)\big)\big)+\phi  \big(\gamma +\sigma _{W_2}^2 \big(\eta '-\zeta \big)\big)\\
0 &= g_6 \psi  \big(-\zeta  g_2+\sqrt{\zeta } \big(g_5-g_4\big) \sqrt{n_0}+g_3 n_0\big)+\sqrt{\zeta } \sqrt{n_0} \sigma _{W_2}^2 \big(g_7 \big(\zeta  g_9+\sqrt{\zeta } \big(g_5+g_8\big) \sqrt{n_0}+g_3 n_0\big)-g_6 \psi \big)\\
0 &= g_{11} \psi  \big(\gamma  \phi +\sigma _{W_2}^2 \big(\phi  \big(\eta '-\zeta \big)+\zeta  g_{10}\big)\big)-\phi  \big(\sqrt{\zeta } g_7 g_{13} \sqrt{n_0}+\psi \big) \big(\gamma +\sigma _{W_2}^2 \big(\eta '-\zeta \big)\big)\\
0 &= g_{10} \big(-\sqrt{\zeta } g_{14} \psi -\sqrt{n_0} \big(\gamma +g_{13} (\eta -\zeta )+\sigma _{W_2}^2 \big(\eta '+\zeta  \big(g_{11}-1\big)\big)\big)\big)+\sqrt{n_0} \big(\gamma +\sigma _{W_2}^2 \big(\eta '-\zeta \big)\big)\\
0 &= g_{14} \psi  \big(-\zeta  g_2+\sqrt{\zeta } \big(g_5-g_4\big) \sqrt{n_0}+g_3 n_0\big)+\sqrt{\zeta } \sqrt{n_0} \sigma _{W_2}^2 \big(g_{11} \big(\zeta  g_9\nonumber\\
&\quad\quad+\sqrt{\zeta } \big(g_5+g_8\big) \sqrt{n_0}+g_3 n_0\big)-g_{14} \psi \big)\\
0 &= \sqrt{\zeta } g_6 g_{13} \sqrt{n_0} \phi  \big(\sigma _{W_2}^2 \big(\zeta -\eta '\big)-\gamma \big)-g_1 \phi  \big(\zeta  g_9+\sqrt{\zeta } \big(g_5+g_8\big) \sqrt{n_0}+g_3 n_0\big) \big(\gamma +\sigma _{W_2}^2 \big(\eta '-\zeta \big)\big)\nonumber\\
&\quad\quad+g_{14} \psi  \big(\gamma  \phi +\sigma _{W_2}^2 \big(-\zeta  \phi +\phi  \eta '+\zeta  g_{10}\big)\big)\\
0 &= g_1 \psi  \phi  \big(\sqrt{\zeta } g_9+g_8 \sqrt{n_0}\big) \big(\gamma +\sigma _{W_2}^2 \big(\eta '-\zeta \big)\big)+\sqrt{n_0} \big(\sigma _{W_2}^2 \big(g_{10} g_{13} \psi  (\eta -\zeta )+g_8 \phi  \big(\gamma +\sigma _{W_2}^2 \big(\eta '-\zeta \big)\big)\big)\nonumber\\
&\quad\quad+g_6 g_{13} \psi  \phi  \big(\gamma +\sigma _{W_2}^2 \big(\eta '-\zeta \big)\big)\big)\\
0 &= \sqrt{\zeta } g_8 \sigma _{W_2}^2 \big(g_{10} \psi  (\eta -\zeta )-\phi  \big(\sqrt{\zeta } g_7 \sqrt{n_0}-1\big) \big(\gamma +\sigma _{W_2}^2 \big(\eta '-\zeta \big)\big)\big)-g_3 \sqrt{n_0} \phi  \big(\sqrt{\zeta } g_7 \sqrt{n_0} \sigma _{W_2}^2+g_6 \psi \big)\nonumber\\
&\quad\quad \big(\gamma +\sigma _{W_2}^2 \big(\eta '-\zeta \big)\big)+\sqrt{\zeta } g_4 \psi  \big(g_6 \phi  \big(\gamma +\sigma _{W_2}^2 \big(\eta '-\zeta \big)\big)+g_{10} (\eta -\zeta ) \sigma _{W_2}^2\big)\\
0 &= \sqrt{\zeta } g_9 \sigma _{W_2}^2 \big(g_{10} \psi  (\eta -\zeta )-\phi  \big(\sqrt{\zeta } g_7 \sqrt{n_0}-1\big)\big(\gamma +\sigma _{W_2}^2 \big(\eta '-\zeta \big)\big)\big)-g_5 \sqrt{n_0} \phi  \big(\sqrt{\zeta } g_7 \sqrt{n_0} \sigma _{W_2}^2 + g_6 \psi \big)\nonumber\\
&\quad\quad \big(\gamma +\sigma _{W_2}^2 \big(\eta '-\zeta \big)\big)+\sqrt{\zeta } g_2 \psi  \big(g_6 \phi  \big(\gamma +\sigma _{W_2}^2 \big(\eta '-\zeta \big)\big)+g_{10} (\eta -\zeta ) \sigma _{W_2}^2\big)
\end{align}
\end{subequations}
}

After some straightforward algebra, one can eliminate all $g_i$ except for $g_7$ and $g_{10}$, which satisfy coupled polynomial equations. Those equations can be shown to be identical to eqn.~(\ref{eqn:tau}) by invoking the change of variables,
\begin{equation}
    g_7 = - \frac{\sqrt{\zeta} \psi}{\sqrt{n_0} \phi} \tau_2\,,\quad\text{and}\quad
    g_{10} = \big(\gamma +\sigma _{W_2}^2 \big(\eta '-\zeta \big)\big)\tau_1\,.
\end{equation}
The error $E_{22}$ is then given by,
\begin{equation}
    E_{22} = 2 \zeta  \pa{\frac{\tau _2}{\tau _1}-1}+\frac{2\psi  \big(\zeta  \big(\tau _2-\tau _1\big)+\eta  \tau _1\big){}^2 \big(\big(\tau _2-\tau _1\big) \phi +\zeta  \tau _1 \tau _2 \sigma _{W_2}^2\big)}{\zeta  \tau _1^2 \tau _2 \phi }\,.
\end{equation}

\subsubsection{$E_{31}$}
The equations satisfied by the operator-valued Stieltjes transform $G$ of $\bar{Q}_{31}$ induce the following structure on $G$,
\eq{
G = \begin{pmatrix} 0 & G_{12} \\ G_{12}^\top & 0 \end{pmatrix}\,,
}
where,
\eq{
G_{12} =\left(
\begin{array}{cccccccc}
 g_5 & 0 & 0 & 0 & 0 & g_2 & 0 & 0 \\
 0 & g_6 & 0 & g_1 & g_3 & 0 & 0 & g_4 \\
 0 & 0 & g_8 & 0 & 0 & 0 & g_{12} & 0 \\
 0 & g_{11} & 0 & g_7 & g_{10} & 0 & 0 & g_9 \\
 0 & 0 & 0 & 0 & g_6 & 0 & 0 & g_1 \\
 0 & 0 & 0 & 0 & 0 & g_5 & 0 & 0 \\
 0 & 0 & 0 & 0 & 0 & 0 & g_8 & 0 \\
 0 & 0 & 0 & 0 & g_{11} & 0 & 0 & g_7 \\
\end{array}
\right)\,,
}
and the independent entry-wise component functions $g_i$ give the error $E_{31}$ through the relation,
\eq{
E_{31} = \frac{g_2 n_0 \sigma _{\varepsilon }^2}{\phi  \big(\gamma +\sigma _{W_2}^2 \big(\eta '-\zeta \big)\big)}\,,
}
and themselves satisfy the following system of polynomial equations,
{\small
\begin{subequations}
\begin{align}
0 &= \sqrt{\zeta } g_6 g_8 \sqrt{n_0}-g_{11} \psi\\
0 &= \sqrt{\zeta } g_1 g_8 \sqrt{n_0}-g_7 \psi +\psi\\
0 &= -\sqrt{\zeta } g_5 g_6 \psi -g_1 \sqrt{n_0} \phi  \big(\gamma +\sigma _{W_2}^2 \big(\eta '-\zeta \big)\big)\\
0 &= -\sqrt{\zeta } g_5 g_{11} \psi -\big(g_7-1\big) \sqrt{n_0} \phi  \big(\gamma +\sigma _{W_2}^2 \big(\eta '-\zeta \big)\big)\\
0 &= -\zeta  g_7 g_8 \psi +\sqrt{\zeta } \sqrt{n_0} \big(\big(g_4 g_8+g_1 g_{12}\big) n_0-\zeta  g_1 g_8 \sigma _{W_2}^2\big)-g_9 n_0 \psi\\
0 &= g_1 \sqrt{n_0} \phi  \big(\gamma +\sigma _{W_2}^2 \big(\eta '-\zeta \big)\big)+g_5 \big(\sqrt{\zeta } g_7 \psi +\zeta  g_1 \sqrt{n_0} \sigma _{W_2}^2\big)\\
0 &= \sqrt{\zeta } g_6 g_{12} n_0^{3/2}-g_8 \big(\zeta  g_{11} \psi +\sqrt{\zeta } \sqrt{n_0} \big(\zeta  g_6 \sigma _{W_2}^2-g_3 n_0\big)\big)-g_{10} n_0 \psi\\
0 &= \big(g_6-1\big) \sqrt{n_0} \phi  \big(\gamma +\sigma _{W_2}^2 \big(\eta '-\zeta \big)\big)+g_5 \big(\sqrt{\zeta } g_{11} \psi +\zeta  g_6 \sqrt{n_0} \sigma _{W_2}^2\big)\\
0 &= \sqrt{\zeta } g_7 g_8 \sqrt{n_0} \phi  \big(\sigma _{W_2}^2 \big(\zeta -\eta '\big)-\gamma \big)+g_{11} \psi  \big(\gamma  \phi +\sigma _{W_2}^2 \big(-\zeta  \phi +\phi  \eta '+\zeta  g_5\big)\big)\\
0 &= g_8 \big(g_5 \psi  (\zeta -\eta )+\phi  \big(\sqrt{\zeta } g_1 \sqrt{n_0}-1\big) \big(\gamma +\sigma _{W_2}^2 \big(\eta '-\zeta \big)\big)\big)+\phi  \big(\gamma +\sigma _{W_2}^2 \big(\eta '-\zeta \big)\big)\\
0 &= g_6 \psi  \big(\gamma  \phi +\sigma _{W_2}^2 \big(-\zeta  \phi +\phi  \eta '+\zeta  g_5\big)\big)-\phi  \big(\sqrt{\zeta } g_1 g_8 \sqrt{n_0}+\psi \big) \big(\gamma +\sigma _{W_2}^2 \big(\eta '-\zeta \big)\big)\\
0 &= g_5 \big(\sqrt{\zeta } g_{11} \psi +\sqrt{n_0} \big(\gamma +g_8 (\eta -\zeta )+\sigma _{W_2}^2 \big(\eta '+\zeta  \big(g_6-1\big)\big)\big)\big)-\sqrt{n_0} \big(\gamma +\sigma _{W_2}^2 \big(\eta '-\zeta \big)\big)\\
0 &= \sqrt{\zeta } g_5 \psi  \big(g_6 \big(\psi  (\eta -\zeta )+\zeta  \sigma _{W_2}^2\big)-g_3 n_0\big)-\sqrt{n_0} \big(\sqrt{\zeta } g_2 g_6 \sqrt{n_0} \psi +g_4 n_0 \phi  \big(\gamma +\sigma _{W_2}^2 \big(\eta '-\zeta \big)\big)\nonumber\\
&\quad\quad+\zeta  g_1 g_8 \phi  \big(\gamma +\sigma _{W_2}^2 \big(\eta '-\zeta \big)\big)\big)\\
0 &= \sqrt{\zeta } g_5 \psi  \big(g_{11} \big(\psi  (\eta -\zeta )+\zeta  \sigma _{W_2}^2\big)-g_{10} n_0\big)-\sqrt{n_0} \big(\sqrt{\zeta } g_2 g_{11} \sqrt{n_0} \psi +g_9 n_0 \phi  \big(\gamma +\sigma _{W_2}^2 \big(\eta '-\zeta \big)\big)\nonumber\\
&\quad\quad+\zeta  g_7 g_8 \phi  \big(\gamma +\sigma _{W_2}^2 \big(\eta '-\zeta \big)\big)\big)\\
0 &= g_5 \big(-\sqrt{\zeta } g_9 n_0 \psi +\zeta  \sqrt{n_0} \sigma _{W_2}^2 \big(\zeta  g_1 \sigma _{W_2}^2-g_4 n_0\big)+\sqrt{\zeta } g_7 \psi  \big(\psi  (\eta -\zeta )+\zeta  \sigma _{W_2}^2\big)\big)\nonumber\\
&\quad\quad-n_0 \big(g_4 \sqrt{n_0} \phi  \big(\gamma +\sigma _{W_2}^2 \big(\eta '-\zeta \big)\big)+g_2 \big(\sqrt{\zeta } g_7 \psi +\zeta  g_1 \sqrt{n_0} \sigma _{W_2}^2\big)\big)\\
0 &= g_5 \big(-\sqrt{\zeta } g_{10} n_0 \psi +\zeta  \sqrt{n_0} \sigma _{W_2}^2 \big(\zeta  g_6 \sigma _{W_2}^2-g_3 n_0\big)+\sqrt{\zeta } g_{11} \psi  \big(\psi  (\eta -\zeta )+\zeta  \sigma _{W_2}^2\big)\big)\nonumber\\
&\quad\quad-n_0 \big(g_3 \sqrt{n_0} \phi  \big(\gamma +\sigma _{W_2}^2 \big(\eta '-\zeta \big)\big)+g_2 \big(\sqrt{\zeta } g_{11} \psi +\zeta  g_6 \sqrt{n_0} \sigma _{W_2}^2\big)\big)\\
0 &= g_2 g_8 n_0 \psi  (\eta -\zeta )-g_5 \psi  (\zeta -\eta ) \big(g_8 \psi  (\zeta -\eta )+g_{12} n_0\big)-\sqrt{n_0} \phi  \big(g_{12} \big(\sqrt{\zeta } g_1 n_0-\sqrt{n_0}\big)\nonumber\\
&\quad\quad+\sqrt{\zeta } g_8 \big(g_4 n_0-\zeta  g_1 \sigma _{W_2}^2\big)\big) \big(\gamma +\sigma _{W_2}^2 \big(\eta '-\zeta \big)\big)+\zeta  g_7 g_8 \psi  \phi  \big(\gamma +\sigma _{W_2}^2 \big(\eta '-\zeta \big)\big)\\
0 &= g_2 n_0 \big(-\sqrt{\zeta } g_{11} \psi -\sqrt{n_0} \big(\gamma +g_8 (\eta -\zeta )+\sigma _{W_2}^2 \big(\eta '+\zeta  \big(g_6-1\big)\big)\big)\big)+g_5 \big(\sqrt{n_0} \big(g_8 \psi  (\zeta -\eta )^2\nonumber\\
&\quad\quad+g_{12} n_0 (\zeta -\eta )-\sqrt{\zeta } g_{10} \sqrt{n_0} \psi -\zeta  g_3 n_0 \sigma _{W_2}^2+\zeta ^2 g_6 \sigma _{W_2}^4\big)+\sqrt{\zeta } g_{11} \psi  \big(\psi  (\eta -\zeta )+\zeta  \sigma _{W_2}^2\big)\big)\\
0 &= g_3 n_0 \psi  \big(\gamma  \phi +\sigma _{W_2}^2 \big(\phi  \big(\eta '-\zeta \big)+\zeta  g_5\big)\big)-\sqrt{\zeta } \big(g_4 g_8 n_0^{3/2} \phi  \big(\gamma +\sigma _{W_2}^2 \big(\eta '-\zeta \big)\big)\nonumber\\
&\quad\quad+g_1 \sqrt{n_0} \phi  \big(g_{12} n_0-\zeta  g_8 \sigma _{W_2}^2\big) \big(\gamma +\sigma _{W_2}^2 \big(\eta '-\zeta \big)\big)+\sqrt{\zeta } g_6 \psi  \sigma _{W_2}^2 \big(\zeta  g_5 \sigma _{W_2}^2-g_2 n_0\big)\big)\\
0 &= g_{10} n_0 \psi  \big(\gamma  \phi +\sigma _{W_2}^2 \big(\phi  \big(\eta '-\zeta \big)+\zeta  g_5\big)\big)-\sqrt{\zeta } \big(g_7 g_{12} n_0^{3/2} \phi  \big(\gamma +\sigma _{W_2}^2 \big(\eta '-\zeta \big)\big)\nonumber\\
&\quad\quad+g_8 \sqrt{n_0} \phi  \big(g_9 n_0-\zeta  g_7 \sigma _{W_2}^2\big) \big(\gamma +\sigma _{W_2}^2 \big(\eta '-\zeta \big)\big)+\sqrt{\zeta } g_{11} \psi  \sigma _{W_2}^2 \big(\zeta  g_5 \sigma _{W_2}^2-g_2 n_0\big)\big)
\end{align}
\end{subequations}
}
After some straightforward algebra, one can eliminate all $g_i$ except for $g_1$ and $g_{5}$, which satisfy coupled polynomial equations. Those equations can be shown to be identical to eqn.~(\ref{eqn:tau}) by invoking the change of variables,
\begin{equation}
    g_1 = - \frac{\sqrt{\zeta} \psi}{\sqrt{n_0} \phi} \tau_2\,,\quad\text{and}\quad
    g_5 = \big(\gamma +\sigma _{W_2}^2 \big(\eta '-\zeta \big)\big)\tau_1\,.
\end{equation}
The error $E_{31}$ can then be written in terms of $\tau_1$ and its derivative $\tau_1'$~(\ref{eq:dtau1}),
\begin{equation}
    E_{31} = \sigma_{\varepsilon}^2\big(-\tau_1'/\tau_1^2-1\big)\,.
\end{equation}

\subsubsection{$E_{32}$}
The equations satisfied by the operator-valued Stieltjes transform $G$ of $\bar{Q}_{32}$ induce the following structure on $G$,
\eq{
G = \begin{pmatrix} 0 & G_{12} \\ G_{12}^\top & 0 \end{pmatrix}\,,
}
where,
\eq{
G_{12} = \left(
\begin{array}{ccccccccc}
 g_9 & 0 & 0 & 0 & 0 & 0 & g_6 & 0 & 0 \\
 0 & g_1 & g_3 & 0 & g_4 & g_7 & 0 & 0 & g_2 \\
 0 & 0 & g_{10} & 0 & g_4 & g_{13} & 0 & 0 & g_5 \\
 0 & 0 & 0 & g_{12} & 0 & 0 & 0 & g_{16} & 0 \\
 0 & 0 & g_{15} & 0 & g_{11} & g_{14} & 0 & 0 & g_8 \\
 0 & 0 & 0 & 0 & 0 & g_{10} & 0 & 0 & g_4 \\
 0 & 0 & 0 & 0 & 0 & 0 & g_9 & 0 & 0 \\
 0 & 0 & 0 & 0 & 0 & 0 & 0 & g_{12} & 0 \\
 0 & 0 & 0 & 0 & 0 & g_{15} & 0 & 0 & g_{11} \\
\end{array}
\right)\,,
}
and the independent entry-wise component functions $g_i$ give the error $E_{32}$ through the relation,
\eq{
E_{32} = -g_2 n_0^{3/2}/(\sqrt{\zeta } \psi )\,,
}
and themselves satisfy the following system of polynomial equations,
{\small
\begin{subequations}
\begin{align}
0 &= \sqrt{\zeta } g_{10} g_{12} \sqrt{n_0}-g_{15} \psi\\
0 &= \sqrt{\zeta } g_4 g_{12} \sqrt{n_0}-g_{11} \psi +\psi\\
0 &= -\sqrt{\zeta } g_9 g_{10} \psi -g_4 \sqrt{n_0} \phi  \big(\gamma +\sigma _{W_2}^2 \big(\eta '-\zeta \big)\big)\\
0 &= -\sqrt{\zeta } g_9 g_{15} \psi -\big(g_{11}-1\big) \sqrt{n_0} \phi  \big(\gamma +\sigma _{W_2}^2 \big(\eta '-\zeta \big)\big)\\
0 &= -\sqrt{\zeta } g_9 \psi -\sqrt{\zeta } g_3 g_9 \psi -g_4 \sqrt{n_0} \phi  \big(\gamma +\sigma _{W_2}^2 \big(\eta '-\zeta \big)\big)\\
0 &= -\sqrt{\zeta } g_6 g_{10} \psi -\sqrt{\zeta } g_9 g_{13} \psi -g_5 \sqrt{n_0} \phi  \big(\gamma +\sigma _{W_2}^2 \big(\eta '-\zeta \big)\big)\\
0 &= -\sqrt{\zeta } g_9 g_{14} \psi -\sqrt{\zeta } g_6 g_{15} \psi -g_8 \sqrt{n_0} \phi  \big(\gamma +\sigma _{W_2}^2 \big(\eta '-\zeta \big)\big)\\
0 &= \sqrt{\zeta } g_5 g_{12} n_0+\sqrt{\zeta } g_4 \big(g_{16} n_0+g_{12} \big(\zeta  \psi -\eta  \psi -\zeta  \sigma _{W_2}^2\big)\big)+g_8 \sqrt{n_0} (-\psi )\\
0 &= g_4 \sqrt{n_0} \phi  \big(\gamma +\sigma _{W_2}^2 \big(\eta '-\zeta \big)\big)+g_9 \big(\sqrt{\zeta } g_{11} \psi +\zeta  g_4 \sqrt{n_0} \sigma _{W_2}^2\big)\\
0 &= g_3 \sqrt{n_0} \phi  \big(\gamma +\sigma _{W_2}^2 \big(\eta '-\zeta \big)\big)+g_9 \big(\sqrt{\zeta } g_{15} \psi +\zeta  g_{10} \sqrt{n_0} \sigma _{W_2}^2\big)\\
0 &= \sqrt{\zeta } g_{12} g_{13} n_0+\sqrt{\zeta } g_{10} \big(g_{16} n_0+g_{12} \big(\zeta  \psi -\eta  \psi -\zeta  \sigma _{W_2}^2\big)\big)+g_{14} \sqrt{n_0} (-\psi )\\
0 &= \big(g_{10}-1\big) \sqrt{n_0} \phi  \big(\gamma +\sigma _{W_2}^2 \big(\eta '-\zeta \big)\big)+g_9 \big(\sqrt{\zeta } g_{15} \psi +\zeta  g_{10} \sqrt{n_0} \sigma _{W_2}^2\big)\\
0 &= -\sqrt{\zeta } \big(\big(g_1+g_3\big) g_6+g_7 g_9\big) \psi -\gamma  g_2 \sqrt{n_0} \phi +\zeta  g_2 \sqrt{n_0} \phi  \sigma _{W_2}^2+g_2 \sqrt{n_0} (-\phi ) \eta ' \sigma _{W_2}^2\\
0 &= \sqrt{\zeta } g_{11} g_{12} \sqrt{n_0} \phi  \big(\sigma _{W_2}^2 \big(\zeta -\eta '\big)-\gamma \big)+g_{15} \psi  \big(\gamma  \phi +\sigma _{W_2}^2 \big(-\zeta  \phi +\phi  \eta '+\zeta  g_9\big)\big)\\
0 &= g_{12} \big(g_9 \psi  (\zeta -\eta )+\phi  \big(\sqrt{\zeta } g_4 \sqrt{n_0}-1\big) \big(\gamma +\sigma _{W_2}^2 \big(\eta '-\zeta \big)\big)\big)+\phi  \big(\gamma +\sigma _{W_2}^2 \big(\eta '-\zeta \big)\big)\\
0 &= g_{10} \psi  \big(\gamma  \phi +\sigma _{W_2}^2 \big(-\zeta  \phi +\phi  \eta '+\zeta  g_9\big)\big)-\phi  \big(\sqrt{\zeta } g_4 g_{12} \sqrt{n_0}+\psi \big) \big(\gamma +\sigma _{W_2}^2 \big(\eta '-\zeta \big)\big)\\
0 &= g_9 \big(\sqrt{\zeta } g_{15} \psi +\sqrt{n_0} \big(\gamma +g_{12} (\eta -\zeta )+\sigma _{W_2}^2 \big(\eta '+\zeta  \big(g_{10}-1\big)\big)\big)\big)-\sqrt{n_0} \big(\gamma +\sigma _{W_2}^2 \big(\eta '-\zeta \big)\big)\\
0 &= -\sqrt{\zeta } g_4 g_{12} \sqrt{n_0} \phi  \big(\gamma +\sigma _{W_2}^2 \big(\eta '-\zeta \big)\big)+g_3 \psi  \big(\gamma  \phi +\sigma _{W_2}^2 \big(-\zeta  \phi +\phi  \eta '+\zeta  g_9\big)\big)+\zeta  g_9 \psi  \sigma _{W_2}^2\\
0 &= g_7 n_0 \phi  \big(\gamma +\sigma _{W_2}^2 \big(\eta '-\zeta \big)\big)+g_6 \big(\sqrt{\zeta } g_{15} \sqrt{n_0} \psi +\zeta  g_{10} n_0 \sigma _{W_2}^2\big)+g_9 \big(\sqrt{\zeta } g_{14} \sqrt{n_0} \psi +\zeta  \sigma _{W_2}^2 \big(g_{13} n_0-\zeta  g_{10} \sigma _{W_2}^2\big)\big)\\
0 &= \gamma  g_2 n_0 \phi +\sqrt{\zeta } g_8 g_9 \sqrt{n_0} \psi +g_6 \big(\sqrt{\zeta } g_{11} \sqrt{n_0} \psi +\zeta  g_4 n_0 \sigma _{W_2}^2\big)-\zeta  g_2 n_0 \phi  \sigma _{W_2}^2+\zeta  g_5 g_9 n_0 \sigma _{W_2}^2\nonumber\\
&\quad\quad+g_2 n_0 \phi  \eta ' \sigma _{W_2}^2-\zeta ^2 g_4 g_9 \sigma _{W_2}^4\\
0 &= g_6 \big(-\sqrt{\zeta } g_{15} \sqrt{n_0} \psi -n_0 \big(\gamma +g_{12} (\eta -\zeta )+\sigma _{W_2}^2 \big(\eta '+\zeta  \big(g_{10}-1\big)\big)\big)\big)+g_9 \big(g_{12} \psi  (\zeta -\eta )^2\nonumber\\
&\quad\quad+g_{16} n_0 (\zeta -\eta )-\sqrt{\zeta } g_{14} \sqrt{n_0} \psi -\zeta  g_{13} n_0 \sigma _{W_2}^2+\zeta ^2 g_{10} \sigma _{W_2}^4\big)\\
0 &= \gamma  g_5 n_0 \phi +\sqrt{\zeta } g_8 g_9 \sqrt{n_0} \psi +\sqrt{\zeta } g_6 g_{11} \sqrt{n_0} \psi +\zeta  g_4 \big(g_6 n_0 \sigma _{W_2}^2+g_{12} \phi  \big(\gamma +\sigma _{W_2}^2 \big(\eta '-\zeta \big)\big)-\zeta  g_9 \sigma _{W_2}^4\big)\nonumber\\
&\quad\quad-\zeta  g_5 n_0 \phi  \sigma _{W_2}^2+\zeta  g_5 g_9 n_0 \sigma _{W_2}^2+g_5 n_0 \phi  \eta ' \sigma _{W_2}^2\\
0 &= \gamma  g_{13} n_0 \phi +\sqrt{\zeta } g_6 g_{15} \sqrt{n_0} \psi +\zeta  g_{10} \big(g_6 n_0 \sigma _{W_2}^2+g_{12} \phi  \big(\gamma +\sigma _{W_2}^2 \big(\eta '-\zeta \big)\big)-\zeta  g_9 \sigma _{W_2}^4\big)\nonumber\\
&\quad\quad+g_9 \big(\sqrt{\zeta } g_{14} \sqrt{n_0} \psi +\zeta  g_{13} n_0 \sigma _{W_2}^2\big)-\zeta  g_{13} n_0 \phi  \sigma _{W_2}^2+g_{13} n_0 \phi  \eta ' \sigma _{W_2}^2\\
0 &= -\sqrt{\zeta } g_{12} \phi  \big(\gamma +\sigma _{W_2}^2 \big(\eta '-\zeta \big)\big) \big(\sqrt{n_0} \big(g_8 n_0+g_{11} \big(\zeta  \psi -\eta  \psi -\zeta  \sigma _{W_2}^2\big)\big)-\sqrt{\zeta } g_{15} \psi \big)\nonumber\\
&\quad\quad+g_{14} n_0 \psi  \big(\gamma  \phi +\sigma _{W_2}^2 \big(-\zeta  \phi +\phi  \eta '+\zeta  g_9\big)\big)-\sqrt{\zeta } g_{11} g_{16} n_0^{3/2} \phi  \big(\gamma +\sigma _{W_2}^2 \big(\eta '-\zeta \big)\big)+\zeta  g_{15} \psi  \sigma _{W_2}^2 \big(g_6 n_0-\zeta  g_9 \sigma _{W_2}^2\big)\\
0 &= g_9 \psi  (-(\zeta -\eta )) \big(g_{12} \psi  (\zeta -\eta )+g_{16} n_0\big)-\sqrt{\zeta } g_4 \sqrt{n_0} \phi  \big(\gamma +\sigma _{W_2}^2 \big(\eta '-\zeta \big)\big) \big(g_{16} n_0+g_{12} \big(\zeta  \psi -\eta  \psi -\zeta  \sigma _{W_2}^2\big)\big)\nonumber\\
&\quad\quad+n_0 \big(g_6 g_{12} \psi  (\eta -\zeta )+\phi  \big(g_{16}-\sqrt{\zeta } g_5 g_{12} \sqrt{n_0}\big) \big(\gamma +\sigma _{W_2}^2 \big(\eta '-\zeta \big)\big)\big)+\zeta  g_{10} g_{12} \psi  \phi  \big(\gamma +\sigma_{W_2}^2 \big(\eta '-\zeta \big)\big)\\
0 &= g_{13} n_0 \psi  \big(\gamma  \phi +\sigma _{W_2}^2 \big(-\zeta  \phi +\phi  \eta '+\zeta  g_9\big)\big)-\sqrt{\zeta } g_4 \sqrt{n_0} \phi  \big(\gamma +\sigma _{W_2}^2 \big(\eta '-\zeta \big)\big)+\sqrt{\zeta } g_5 g_{12} n_0^{3/2} \phi  \big(\sigma _{W_2}^2 \big(\zeta -\eta '\big)-\gamma \big) \nonumber\\
&\quad\quad\big(g_{16} n_0+g_{12} \big(\zeta  \psi -\eta  \psi -\zeta  \sigma _{W_2}^2\big)\big)+\zeta  g_{10} \psi  \big(g_6 n_0 \sigma _{W_2}^2+g_{12} \phi  \big(\gamma +\sigma _{W_2}^2 \big(\eta '-\zeta \big)\big)-\zeta  g_9 \sigma _{W_2}^4\big)\\
0 &= -\gamma  \sqrt{\zeta } g_2 g_{12} n_0^{3/2} \phi +\gamma  g_7 n_0 \psi  \phi -\sqrt{\zeta } g_4 \sqrt{n_0} \phi  \big(\gamma +\sigma _{W_2}^2 \big(\eta '-\zeta \big)\big) \big(g_{16} n_0+g_{12} \big(\zeta  \psi -\eta  \psi -\zeta  \sigma _{W_2}^2\big)\big)\nonumber\\
&\quad\quad+\zeta  g_3 \psi  \big(g_6 n_0 \sigma _{W_2}^2+g_{12} \phi  \big(\gamma +\sigma _{W_2}^2 \big(\eta '-\zeta \big)\big)-\zeta  g_9 \sigma _{W_2}^4\big)+\zeta ^{3/2} g_2 g_{12} n_0^{3/2} \phi  \sigma _{W_2}^2-\zeta ^2 g_9 \psi  \sigma _{W_2}^4\nonumber\\
&\quad\quad+n_0 \phi  \eta ' \sigma _{W_2}^2 \big(g_7 \psi -\sqrt{\zeta } g_2 g_{12} \sqrt{n_0}\big)+\zeta  g_6 n_0 \psi  \sigma _{W_2}^2+\zeta  g_7 g_9 n_0 \psi  \sigma _{W_2}^2-\zeta  g_7 n_0 \psi  \phi  \sigma _{W_2}^2
\end{align}
\end{subequations}
}
After some straightforward algebra, one can eliminate all $g_i$ except for $g_4$ and $g_9$, which satisfy coupled polynomial equations. Those equations can be shown to be identical to eqn.~(\ref{eqn:tau}) by invoking the change of variables,
\begin{equation}
    g_4 = - \frac{\sqrt{\zeta} \psi}{\sqrt{n_0} \phi} \tau_2\,,\quad\text{and}\quad
    g_9 = \big(\gamma +\sigma _{W_2}^2 \big(\eta '-\zeta \big)\big)\tau_1\,.
\end{equation}
In terms of $\tau_1$, $\tau_2$, and $\tau_2'$~(\ref{eq:dtau2}), the error $E_{32}$ is given by, 
\begin{equation}
    E_{32} = 1-2\tau_2/\tau_1-\tau_2'/\tau_1^2\,.
\end{equation}

\subsubsection{$E_{33}$}
The equations satisfied by the operator-valued Stieltjes transform $G$ of $\bar{Q}_{32}$ induce the following structure on $G$,
\eq{
G = \begin{pmatrix} 0 & G_{12} \\ G_{12}^\top & 0 \end{pmatrix}\,,
}
where,
\eq{
G_{12} = \left(
\begin{array}{ccccccccccc}
 g_{13} & 0 & 0 & 0 & 0 & 0 & 0 & g_8 & 0 & 0 & 0 \\
 0 & g_1 & 0 & 0 & g_5 & 0 & 0 & 0 & g_{11} & g_3 & 0 \\
 0 & 0 & g_1 & g_4 & 0 & g_6 & g_9 & 0 & 0 & 0 & g_2 \\
 0 & 0 & 0 & g_{14} & 0 & g_6 & g_{17} & 0 & 0 & 0 & g_7 \\
 0 & 0 & 0 & 0 & g_{16} & 0 & 0 & 0 & g_{20} & g_{12} & 0 \\
 0 & 0 & 0 & g_{19} & 0 & g_{15} & g_{18} & 0 & 0 & 0 & g_{10} \\
 0 & 0 & 0 & 0 & 0 & 0 & g_{14} & 0 & 0 & 0 & g_6 \\
 0 & 0 & 0 & 0 & 0 & 0 & 0 & g_{13} & 0 & 0 & 0 \\
 0 & 0 & 0 & 0 & 0 & 0 & 0 & 0 & g_1 & g_5 & 0 \\
 0 & 0 & 0 & 0 & 0 & 0 & 0 & 0 & 0 & g_{16} & 0 \\
 0 & 0 & 0 & 0 & 0 & 0 & g_{19} & 0 & 0 & 0 & g_{15} \\
\end{array}
\right)\,,
}
and the independent entry-wise component functions $g_i$ give the error $E_{32}$ through the relation,
\eq{
E_{33} =-g_3 n_0 \sigma _{W_2}^2/\psi\,,
}
and themselves satisfy the following system of polynomial equations,
{\small
\begin{subequations}
\begin{align}
0 &= \sqrt{\zeta } g_{14} g_{16} \sqrt{n_0}-g_{19} \psi\\
0 &= \sqrt{\zeta } g_6 g_{16} \sqrt{n_0}-g_{15} \psi +\psi\\
0 &= -\sqrt{\zeta } g_{13} g_{14} \psi -g_6 \sqrt{n_0} \phi  \big(\gamma +\sigma _{W_2}^2 \big(\eta '-\zeta \big)\big)\\
0 &= -\sqrt{\zeta } g_{13} g_{19} \psi -\big(g_{15}-1\big) \sqrt{n_0} \phi  \big(\gamma +\sigma _{W_2}^2 \big(\eta '-\zeta \big)\big)\\
0 &= -\sqrt{\zeta } g_{13} \psi -\sqrt{\zeta } g_4 g_{13} \psi -g_6 \sqrt{n_0} \phi  \big(\gamma +\sigma _{W_2}^2 \big(\eta '-\zeta \big)\big)\\
0 &= -\sqrt{\zeta } g_8 g_{14} \psi -\sqrt{\zeta } g_{13} g_{17} \psi -g_7 \sqrt{n_0} \phi  \big(\gamma +\sigma _{W_2}^2 \big(\eta '-\zeta \big)\big)\\
0 &= -\sqrt{\zeta } g_{13} g_{18} \psi -\sqrt{\zeta } g_8 g_{19} \psi -g_{10} \sqrt{n_0} \phi  \big(\gamma +\sigma _{W_2}^2 \big(\eta '-\zeta \big)\big)\\
0 &= g_{13} g_{16} \psi  (\zeta -\eta )-\phi  \big(g_5-\sqrt{\zeta } g_6 g_{16} \sqrt{n_0}\big) \big(\gamma +\sigma _{W_2}^2 \big(\eta '-\zeta \big)\big)\\
0 &= g_6 \sqrt{n_0} \phi  \big(\gamma +\sigma _{W_2}^2 \big(\eta '-\zeta \big)\big)+g_{13} \big(\sqrt{\zeta } g_{15} \psi +\zeta  g_6 \sqrt{n_0} \sigma _{W_2}^2\big)\\
0 &= g_4 \sqrt{n_0} \phi  \big(\gamma +\sigma _{W_2}^2 \big(\eta '-\zeta \big)\big)+g_{13} \big(\sqrt{\zeta } g_{19} \psi +\zeta  g_{14} \sqrt{n_0} \sigma _{W_2}^2\big)\\
0 &= \big(g_{14}-1\big) \sqrt{n_0} \phi  \big(\gamma +\sigma _{W_2}^2 \big(\eta '-\zeta \big)\big)+g_{13} \big(\sqrt{\zeta } g_{19} \psi +\zeta  g_{14} \sqrt{n_0} \sigma _{W_2}^2\big)\\
0 &= -\sqrt{\zeta } \big(\big(g_4+1\big) g_8+g_9 g_{13}\big) \psi -\gamma  g_2 \sqrt{n_0} \phi +\zeta  g_2 \sqrt{n_0} \phi  \sigma _{W_2}^2+g_2 \sqrt{n_0} (-\phi ) \eta ' \sigma _{W_2}^2\\
0 &= \sqrt{\zeta } g_{15} g_{16} \sqrt{n_0} \phi  \big(\sigma _{W_2}^2 \big(\zeta -\eta '\big)-\gamma \big)+g_{19} \psi  \big(\gamma  \phi +\sigma _{W_2}^2 \big(-\zeta  \phi +\phi  \eta '+\zeta  g_{13}\big)\big)\\
0 &= g_{16} \big(g_{13} \psi  (\zeta -\eta )+\phi  \big(\sqrt{\zeta } g_6 \sqrt{n_0}-1\big) \big(\gamma +\sigma _{W_2}^2 \big(\eta '-\zeta \big)\big)\big)+\phi  \big(\gamma +\sigma _{W_2}^2 \big(\eta '-\zeta \big)\big)\\
0 &= g_{13} \big(\sqrt{\zeta } g_{19} \psi +\sqrt{n_0} \big(\gamma +g_{16} (\eta -\zeta )+\sigma _{W_2}^2 \big(\eta '+\zeta  \big(g_{14}-1\big)\big)\big)\big)-\sqrt{n_0} \big(\gamma +\sigma _{W_2}^2 \big(\eta '-\zeta \big)\big)\\
0 &= g_{14} \psi  \big(\gamma  \phi +\sigma _{W_2}^2 \big(\phi  \big(\eta '-\zeta \big)+\zeta  g_{13}\big)\big)-\phi  \big(\sqrt{\zeta } g_6 g_{16} \sqrt{n_0}+\psi \big) \big(\gamma +\sigma _{W_2}^2 \big(\eta '-\zeta \big)\big)\\
0 &= -\sqrt{\zeta } g_6 g_{16} \sqrt{n_0} \phi  \big(\gamma +\sigma _{W_2}^2 \big(\eta '-\zeta \big)\big)+g_4 \psi  \big(\gamma  \phi +\sigma _{W_2}^2 \big(-\zeta  \phi +\phi  \eta '+\zeta  g_{13}\big)\big)+\zeta  g_{13} \psi  \sigma _{W_2}^2\\
0 &= \sqrt{\zeta } \big(g_7 g_{16}+g_6 \big(g_{12}+g_{20}\big)\big) n_0+g_{10} \sqrt{n_0} (-\psi )+\sqrt{\zeta } g_6 \big(\psi  (\zeta -\eta )-\zeta  \sigma _{W_2}^2\big)+\sqrt{\zeta } g_5 g_6 \big(\zeta  \psi -\eta  \psi -\zeta  \sigma _{W_2}^2\big)\\
0 &= \sqrt{\zeta } \big(g_{16} g_{17}+g_{14} \big(g_{12}+g_{20}\big)\big) n_0+g_{18} \sqrt{n_0} (-\psi )+\sqrt{\zeta } g_{14} \big(\psi  (\zeta -\eta )-\zeta  \sigma _{W_2}^2\big)+\sqrt{\zeta } g_5 g_{14} \big(\zeta  \psi -\eta  \psi -\zeta  \sigma _{W_2}^2\big)\\
0 &= g_{13} \psi  (\zeta -\eta )+g_5 \big(g_{13} \psi  (\zeta -\eta )+\phi  \big(\sqrt{\zeta } g_6 \sqrt{n_0}-1\big) \big(\gamma +\sigma _{W_2}^2 \big(\eta '-\zeta \big)\big)\big)+\sqrt{\zeta } g_6 \sqrt{n_0} \phi  \big(\gamma +\sigma _{W_2}^2 \big(\eta '-\zeta \big)\big)\\
0 &= g_9 n_0 \phi  \big(\gamma +\sigma _{W_2}^2 \big(\eta '-\zeta \big)\big)+g_8 \big(\sqrt{\zeta } g_{19} \sqrt{n_0} \psi +\zeta  g_{14} n_0 \sigma _{W_2}^2\big)+g_{13} \big(\sqrt{\zeta } g_{18} \sqrt{n_0} \psi +\zeta  \sigma _{W_2}^2 \big(g_{17} n_0-\zeta  g_{14} \sigma _{W_2}^2\big)\big)\\
0 &= \gamma  g_2 n_0 \phi +\sqrt{\zeta } g_{10} g_{13} \sqrt{n_0} \psi +g_8 \big(\sqrt{\zeta } g_{15} \sqrt{n_0} \psi +\zeta  g_6 n_0 \sigma _{W_2}^2\big)-\zeta  g_2 n_0 \phi  \sigma _{W_2}^2\nonumber\\
&\quad\quad+\zeta  g_7 g_{13} n_0 \sigma _{W_2}^2+g_2 n_0 \phi  \eta ' \sigma _{W_2}^2-\zeta ^2 g_6 g_{13} \sigma _{W_2}^4\\
0 &= g_{13} g_{16} \psi  (-(\zeta -\eta )) \big(\psi  (\zeta -\eta )-\zeta  \sigma _{W_2}^2\big)\nonumber\\
&\quad\quad-\phi  \big(\gamma +\sigma _{W_2}^2 \big(\eta '-\zeta \big)\big) \big(-\zeta  g_{14} g_{16} \psi +\sqrt{\zeta } g_6 g_{16} \sqrt{n_0} \big(\zeta  \psi -\eta  \psi -\zeta  \sigma _{W_2}^2\big)-g_{20} n_0\big)\\
0 &= -\sqrt{\zeta } \phi  \big(\gamma +\sigma _{W_2}^2 \big(\eta '-\zeta \big)\big) \big(g_6 \sqrt{n_0} \big(\zeta  \psi -\eta  \psi -\zeta  \sigma _{W_2}^2\big)-\sqrt{\zeta } g_{14} \psi \big)+g_{20} n_0 \big(g_{13} \psi  (\eta -\zeta )\nonumber\\
&\quad\quad-\phi  \big(\sqrt{\zeta } g_6 \sqrt{n_0}-1\big) \big(\gamma +\sigma _{W_2}^2 \big(\eta '-\zeta \big)\big)\big)+g_{13} \psi  (-(\zeta -\eta )) \big(\psi  (\zeta -\eta )-\zeta  \sigma _{W_2}^2\big)\\
0 &= \big(\psi  (\zeta -\eta )-\zeta  \sigma _{W_2}^2\big) \big(g_{13} \psi  (\eta -\zeta )+\sqrt{\zeta } g_6 \sqrt{n_0} \phi  \big(\sigma _{W_2}^2 \big(\zeta -\eta '\big)-\gamma \big)\big)+n_0 \big(g_{13} g_{20} \psi  (\eta -\zeta )\nonumber\\
&\quad\quad+\phi  \big(g_{11}-\sqrt{\zeta } g_6 g_{20} \sqrt{n_0}\big) \big(\gamma +\sigma _{W_2}^2 \big(\eta '-\zeta \big)\big)\big)+\zeta  g_4 \psi  \phi  \big(\gamma +\sigma _{W_2}^2 \big(\eta '-\zeta \big)\big)\\
0 &= \gamma  g_7 n_0 \phi +\sqrt{\zeta } g_{10} g_{13} \sqrt{n_0} \psi +\sqrt{\zeta } g_8 g_{15} \sqrt{n_0} \psi -\zeta  g_7 n_0 \phi  \sigma _{W_2}^2+\zeta  g_6 g_8 n_0 \sigma _{W_2}^2+\zeta  g_7 g_{13} n_0 \sigma _{W_2}^2+g_7 n_0 \phi  \eta ' \sigma _{W_2}^2\nonumber\\
&\quad\quad+\zeta  g_6 \phi  \big(\gamma +\sigma _{W_2}^2 \big(\eta '-\zeta \big)\big)+\zeta  g_5 g_6 \phi  \big(\gamma +\sigma _{W_2}^2 \big(\eta '-\zeta \big)\big)-\zeta ^2 g_6 g_{13} \sigma _{W_2}^4\\
0 &= \gamma  g_{17} n_0 \phi +\sqrt{\zeta } g_{13} g_{18} \sqrt{n_0} \psi +\sqrt{\zeta } g_8 g_{19} \sqrt{n_0} \psi -\zeta  g_{17} n_0 \phi  \sigma _{W_2}^2+\zeta  g_8 g_{14} n_0 \sigma _{W_2}^2+\zeta  g_{13} g_{17} n_0 \sigma _{W_2}^2+g_{17} n_0 \phi  \eta ' \sigma _{W_2}^2\nonumber\\
&\quad\quad+\zeta  g_{14} \phi  \big(\gamma +\sigma _{W_2}^2 \big(\eta '-\zeta \big)\big)+\zeta  g_5 g_{14} \phi  \big(\gamma +\sigma _{W_2}^2 \big(\eta '-\zeta \big)\big)-\zeta ^2 g_{13} g_{14} \sigma _{W_2}^4\\
0 &= g_5 \big(\psi  (\zeta -\eta )-\zeta  \sigma _{W_2}^2\big) \big(g_{13} \psi  (\eta -\zeta )+\sqrt{\zeta } g_6 \sqrt{n_0} \phi  \big(\sigma _{W_2}^2 \big(\zeta -\eta '\big)-\gamma \big)\big)+n_0 \big(\phi  \big(g_3-\sqrt{\zeta } \big(g_6 g_{12}+g_2 g_{16}\big) \sqrt{n_0}\big) \nonumber\\
&\quad\quad\big(\gamma +\sigma _{W_2}^2 \big(\eta '-\zeta \big)\big)-\big(g_{12} g_{13}+g_8 g_{16}\big) \psi  (\zeta -\eta )\big)+\zeta  g_4 g_5 \psi  \phi  \big(\gamma +\sigma _{W_2}^2 \big(\eta '-\zeta \big)\big)\\
0 &= \big(\psi  (\zeta -\eta )-\zeta  \sigma _{W_2}^2\big) \big(g_{13} \psi  (\eta -\zeta )+\sqrt{\zeta } g_6 \sqrt{n_0} \phi  \big(\sigma _{W_2}^2 \big(\zeta -\eta '\big)-\gamma \big)\big)+g_5 \big(g_{13} \psi  (-(\zeta -\eta )) \big(\psi  (\zeta -\eta )-\zeta  \sigma _{W_2}^2\big)\nonumber\\
&\quad\quad-\sqrt{\zeta } \phi  \big(\gamma +\sigma _{W_2}^2 \big(\eta '-\zeta \big)\big) \big(g_6 \sqrt{n_0} \big(\zeta  \psi -\eta  \psi -\zeta  \sigma _{W_2}^2\big)-\sqrt{\zeta } g_{14} \psi \big)\big)+g_{11} n_0 \phi  \big(\gamma +\sigma _{W_2}^2 \big(\eta '-\zeta \big)\big)\nonumber\\
&\quad\quad+\zeta  g_4 \psi  \phi  \big(\gamma +\sigma _{W_2}^2 \big(\eta '-\zeta \big)\big)\\
0 &= g_{12} n_0 \big(g_{13} \psi  (\eta -\zeta )-\phi  \big(\sqrt{\zeta } g_6 \sqrt{n_0}-1\big) \big(\gamma +\sigma _{W_2}^2 \big(\eta '-\zeta \big)\big)\big)-g_{16} \big(g_8 n_0 \psi  (\zeta -\eta )+\sqrt{\zeta } g_7 n_0^{3/2} \phi  \big(\gamma +\sigma _{W_2}^2 \big(\eta '-\zeta \big)\big)\nonumber\\
&\quad\quad+\zeta  g_{13} \psi  (\zeta -\eta ) \sigma _{W_2}^2\big)+g_5 \big(g_{13} \psi  (-(\zeta -\eta )) \big(\psi  (\zeta -\eta )-\zeta  \sigma _{W_2}^2\big)\nonumber\\
&\quad\quad-\sqrt{\zeta } \phi  \big(\gamma +\sigma _{W_2}^2 \big(\eta '-\zeta \big)\big) \big(g_6 \sqrt{n_0} \big(\zeta  \psi -\eta  \psi -\zeta  \sigma _{W_2}^2\big)-\sqrt{\zeta } g_{14} \psi \big)\big)\\
0 &= \gamma  \sqrt{\zeta } g_7 g_{16} n_0^{3/2} \phi +\gamma  \sqrt{\zeta } g_6 g_{20} n_0^{3/2} \phi +g_8 g_{16} n_0 \psi  (\zeta -\eta )+\zeta  g_{13} g_{20} n_0 \psi -\eta  g_{13} g_{20} n_0 \psi \nonumber\\
&\quad\quad+g_{12} n_0 \big(g_{13} \psi  (\zeta -\eta )+\phi  \big(\sqrt{\zeta } g_6 \sqrt{n_0}-1\big) \big(\gamma +\sigma _{W_2}^2 \big(\eta '-\zeta \big)\big)\big)-\zeta ^{3/2} g_7 g_{16} n_0^{3/2} \phi  \sigma _{W_2}^2-\zeta ^{3/2} g_6 g_{20} n_0^{3/2} \phi  \sigma _{W_2}^2\nonumber\\
&\quad\quad+\sqrt{\zeta } \big(g_7 g_{16}+g_6 g_{20}\big) n_0^{3/2} \phi  \eta ' \sigma _{W_2}^2+\zeta ^2 g_{13} g_{16} \psi  \sigma _{W_2}^2-\zeta  \eta  g_{13} g_{16} \psi  \sigma _{W_2}^2\\
0 &= -\gamma  g_8 n_0-\sqrt{\zeta } g_{13} g_{18} \sqrt{n_0} \psi -\sqrt{\zeta } g_8 g_{19} \sqrt{n_0} \psi +\zeta  g_{12} g_{13} n_0+\zeta  g_8 g_{16} n_0+\zeta  g_{13} g_{20} n_0-\eta  g_{12} g_{13} n_0\nonumber\\
&\quad\quad-\eta  g_8 g_{16} n_0-\eta  g_{13} g_{20} n_0+\zeta  g_8 n_0 \sigma _{W_2}^2-\zeta  g_8 g_{14} n_0 \sigma _{W_2}^2-\zeta  g_{13} g_{17} n_0 \sigma _{W_2}^2-g_8 n_0 \eta ' \sigma _{W_2}^2+\zeta ^2 g_{13} g_{14} \sigma _{W_2}^4\nonumber\\
&\quad\quad+\zeta ^2 g_{13} g_{16} \sigma _{W_2}^2+g_{13} (\zeta -\eta ) \big(\psi  (\zeta -\eta )-\zeta  \sigma _{W_2}^2\big)+g_5 g_{13} (\zeta -\eta ) \big(\zeta  \psi -\eta  \psi -\zeta  \sigma _{W_2}^2\big)-\zeta  \eta  g_{13} g_{16} \sigma _{W_2}^2\\
0 &= \gamma  \sqrt{\zeta } g_5 g_7 n_0^{3/2} \phi +\gamma  \sqrt{\zeta } g_6 g_{11} n_0^{3/2} \phi +\zeta  g_5 g_8 n_0 \psi +\zeta  g_{11} g_{13} n_0 \psi -\eta  g_5 g_8 n_0 \psi -\eta  g_{11} g_{13} n_0 \psi\nonumber\\
&\quad\quad +g_3 n_0 \big(g_{13} \psi  (\zeta -\eta )+\phi  \big(\sqrt{\zeta } g_6 \sqrt{n_0}-1\big) \big(\gamma +\sigma _{W_2}^2 \big(\eta '-\zeta \big)\big)\big)+n_0 \big(g_8 \psi  (\zeta -\eta )\nonumber\\
&\quad\quad+\sqrt{\zeta } g_2 \sqrt{n_0} \phi  \big(\gamma +\sigma _{W_2}^2 \big(\eta '-\zeta \big)\big)\big)-\zeta ^{3/2} g_5 g_7 n_0^{3/2} \phi  \sigma _{W_2}^2-\zeta ^{3/2} g_6 g_{11} n_0^{3/2} \phi  \sigma _{W_2}^2+\sqrt{\zeta } g_5 g_7 n_0^{3/2} \phi  \eta ' \sigma _{W_2}^2\nonumber\\
&\quad\quad+\sqrt{\zeta } g_6 g_{11} n_0^{3/2} \phi  \eta ' \sigma _{W_2}^2+\zeta ^2 g_5 g_{13} \psi  \sigma _{W_2}^2-\zeta  \eta  g_5 g_{13} \psi  \sigma _{W_2}^2\\
0 &= -\sqrt{\zeta } g_6 \sqrt{n_0} \phi  \big(\psi  (\zeta -\eta )-\zeta  \sigma _{W_2}^2\big) \big(\gamma +\sigma _{W_2}^2 \big(\eta '-\zeta \big)\big)-\sqrt{\zeta } g_5 g_6 \sqrt{n_0} \phi  \big(\psi  (\zeta -\eta )-\zeta  \sigma _{W_2}^2\big) \big(\gamma +\sigma _{W_2}^2 \big(\eta '-\zeta \big)\big)\nonumber\\
&\quad\quad+g_9 n_0 \psi  \big(\gamma  \phi +\sigma _{W_2}^2 \big(\phi  \big(\eta '-\zeta \big)+\zeta  g_{13}\big)\big)+\zeta  g_4 \psi  \big(g_8 n_0 \sigma _{W_2}^2+g_5 \phi  \big(\gamma +\sigma _{W_2}^2 \big(\eta '-\zeta \big)\big)-\zeta  g_{13} \sigma _{W_2}^4\big)\nonumber\\
&\quad\quad-\sqrt{\zeta } \big(g_2 g_{16}+g_6 \big(g_{12}+g_{20}\big)\big) n_0^{3/2} \phi  \big(\gamma +\sigma _{W_2}^2 \big(\eta '-\zeta \big)\big)+\zeta  \psi  \sigma _{W_2}^2 \big(g_8 n_0-\zeta  g_{13} \sigma _{W_2}^2\big)\nonumber\\
&\quad\quad+\zeta  g_4 \psi  \phi  \big(\gamma +\sigma _{W_2}^2 \big(\eta '-\zeta \big)\big)\\
0 &= -\gamma  \sqrt{\zeta } g_6 g_{12} n_0^{3/2} \phi -\gamma  \sqrt{\zeta } g_7 g_{16} n_0^{3/2} \phi -\gamma  \sqrt{\zeta } g_6 g_{20} n_0^{3/2} \phi +\gamma  g_{17} n_0 \psi  \phi -\sqrt{\zeta } \phi  \big(\gamma +\sigma _{W_2}^2 \big(\eta '-\zeta \big)\big) \nonumber\\
&\quad\quad\big(g_6 \sqrt{n_0} \big(\zeta  \psi -\eta  \psi -\zeta  \sigma _{W_2}^2\big)-\sqrt{\zeta } g_{14} \psi \big)-\sqrt{\zeta } g_5 \phi  \big(\gamma +\sigma _{W_2}^2 \big(\eta '-\zeta \big)\big) \big(g_6 \sqrt{n_0} \big(\zeta  \psi -\eta  \psi -\zeta  \sigma _{W_2}^2\big)-\sqrt{\zeta } g_{14} \psi \big)\nonumber\\
&\quad\quad+\zeta ^{3/2} g_6 g_{12} n_0^{3/2} \phi  \sigma _{W_2}^2+\zeta ^{3/2} g_7 g_{16} n_0^{3/2} \phi  \sigma _{W_2}^2+\zeta ^{3/2} g_6 g_{20} n_0^{3/2} \phi  \sigma _{W_2}^2-n_0 \phi  \eta ' \sigma _{W_2}^2 \big(\sqrt{\zeta } \big(g_7 g_{16}+g_6 \big(g_{12}+g_{20}\big)\big) \sqrt{n_0}\nonumber\\
&\quad\quad-g_{17} \psi \big)+\zeta  g_8 g_{14} n_0 \psi  \sigma _{W_2}^2+\zeta  g_{13} g_{17} n_0 \psi  \sigma _{W_2}^2-\zeta  g_{17} n_0 \psi  \phi  \sigma _{W_2}^2-\zeta ^2 g_{13} g_{14} \psi  \sigma _{W_2}^4\\
0 &= -\gamma  \sqrt{\zeta } g_{12} g_{15} n_0^{3/2} \phi -\gamma  \sqrt{\zeta } g_{10} g_{16} n_0^{3/2} \phi -\gamma  \sqrt{\zeta } g_{15} g_{20} n_0^{3/2} +\gamma  g_{18} n_0 \psi  \phi -\sqrt{\zeta } \phi  \big(\gamma +\sigma _{W_2}^2 \big(\eta '-\zeta \big)\big) \nonumber\\
&\quad\quad\big(g_{15} \sqrt{n_0} \big(\zeta  \psi -\eta  \psi -\zeta  \sigma _{W_2}^2\big)-\sqrt{\zeta } g_{19} \psi \big)-\sqrt{\zeta } g_5 \phi  \big(\gamma +\sigma _{W_2}^2 \big(\eta '-\zeta \big)\big) \big(g_{15} \sqrt{n_0} \big(\zeta  \psi -\eta  \psi -\zeta  \sigma _{W_2}^2\big)-\zeta  g_{18} n_0 \psi  \phi  \sigma _{W_2}^2\nonumber\\
&\quad\quad-\sqrt{\zeta } g_{19} \psi \big)+\zeta ^{3/2} g_{12} g_{15} n_0^{3/2} \phi  \sigma _{W_2}^2+\zeta ^{3/2} g_{10} g_{16} n_0^{3/2} \phi  \sigma _{W_2}^2+\zeta ^{3/2} g_{15} g_{20} n_0^{3/2} \phi  \sigma _{W_2}^2-\zeta ^2 g_{13} g_{19} \psi  \sigma _{W_2}^4\nonumber\\
&\quad\quad-n_0 \phi  \eta ' \sigma _{W_2}^2 \big(\sqrt{\zeta } \big(g_{10} g_{16}+g_{15} \big(g_{12}+g_{20}\big)\big) \sqrt{n_0}-g_{18} \psi \big)+\zeta  g_{13} g_{18} n_0 \psi  \sigma _{W_2}^2+\zeta  g_8 g_{19} n_0 \psi  \sigma _{W_2}^2
\end{align}
\end{subequations}
}
After some straightforward algebra, one can eliminate all $g_i$ except for $g_6$ and $g_{13}$, which satisfy coupled polynomial equations. Those equations can be shown to be identical to eqn.~(\ref{eqn:tau}) by invoking the change of variables,
\begin{equation}
    g_6 = - \frac{\sqrt{\zeta} \psi}{\sqrt{n_0} \phi} \tau_2\,,\quad\text{and}\quad
    g_{13} = \big(\gamma +\sigma _{W_2}^2 \big(\eta '-\zeta \big)\big)\tau_1\,.
\end{equation}
In terms of $\tau_1$, $\tau_2$, and their derivatives $\tau_1'$~(\ref{eq:dtau1}), $\tau_2'$~(\ref{eq:dtau2}), the error $E_{33}$ is given by,
\begin{equation}
    E_{33} = \sigma_{W_2}^2 \left[\left(\tau_1 + (\sigma_{W_2}^2(\eta'-\zeta) + \gamma)\tau_1'+\sigma_{W_2}^2\zeta \tau_2'\right)/\tau_1^2 -\eta\right] - E_{22}\,.
\end{equation}

\section{Exact asymptotics for bias and variance terms}
Following Sec.~\ref{sec_var_decomp_sm}, for each random variable in question we introduce an iid copy of it denoted by a tilde. Using this simplifying notation and recalling $P=\{W_1,W_2\}$ we have,
\begin{align}
\bias &= \mathbb{E}_{(\bfx,y)} (y - \mathbb{E}_{(P,X,\varepsilon)}\hat{y}(\bfx; P,X,\varepsilon))^2\\
&= \mathbb{E}_{(\bfx,y)}\mathbb{E}_{(P,X,\varepsilon)}\mathbb{E}_{(\tilde{P},\tilde{X},\tilde{\varepsilon})}(y - \hat{y}(\bfx; P,X,\varepsilon))(y - \hat{y}(\bfx; \tilde{P},\tilde{X},\tilde{\varepsilon}))\\
&= 1 + E_{21} + H_{000}\,,
\end{align}
where $E_{21}$ was computed previously and $H_{000}$ and the other $H_{ijk}$(also defined above) are, \jp{Need to reconcile the ordering of the indices with the previous sections.}
\begin{align}
H_{000} &= \mathbb{E}\hat{y}(\bfx; P,X,\varepsilon)\hat{y}(\bfx; \tilde{P},\tilde{X},\tilde{\varepsilon})\\
H_{001} &= \mathbb{E}\hat{y}(\bfx; P,X,\varepsilon)\hat{y}(\bfx; \tilde{P},\tilde{X},\varepsilon)\\
H_{010} &= \mathbb{E}\hat{y}(\bfx; P,X,\varepsilon)\hat{y}(\bfx; \tilde{P},X,\tilde{\varepsilon})\\
H_{011} &= \mathbb{E}\hat{y}(\bfx; P,X,\varepsilon)\hat{y}(\bfx; \tilde{P},X,\varepsilon)\\
H_{100} &= \mathbb{E}\hat{y}(\bfx; P,X,\varepsilon)\hat{y}(\bfx; P,\tilde{X},\tilde{\varepsilon})\\
H_{101} &= \mathbb{E}\hat{y}(\bfx; P,X,\varepsilon)\hat{y}(\bfx; P,\tilde{X},\varepsilon)\\
H_{110} &= \mathbb{E}\hat{y}(\bfx; P,X,\varepsilon)\hat{y}(\bfx; P,X,\tilde{\varepsilon})\\
H_{111} &= \mathbb{E}\hat{y}(\bfx; P,X,\varepsilon)\hat{y}(\bfx; P,X,\varepsilon)\,,
\end{align}
where the expectations are over $\bfx,P,X,\varepsilon, \tilde{P},\tilde{X}$, and $\tilde{\varepsilon}$. Recalling the definition of $\hat{y}$,
\eq{
\hat{y}(\bfx; P, X, \varepsilon) \deq N_0(\bfx; P) + (Y(X,\epsilon) - N_0(X; P))K(X,X;P)^{-1}K(X,\bfx;P)\,
}
and the techniques described in the previous section, it is straightforward to analyze each of the above terms, which we do in the following subsections. To aid those calculations, we first note that, similar to above, we can write,
\begin{align}
\mathbb{E}_\bfx K(X,\bfx;P)K(\bfx,\tilde{X};\tilde{P}) & = \frac{\sigma_{W_2}^4 \zeta^2}{n_0^2} X^\top \tilde{X} + \frac{\sigma_{W_2}^2 \zeta^{3/2}}{n_0^{3/2} n_1}(X^\top W_1^T \tilde{F} + F^\top \tilde{W}_1 \tilde{X}) + \frac{\zeta}{n_0 n_1^2}F^\top W_1 \tilde{W}_1^\top\tilde{F}\\
& = (\frac{\sigma_{W_2}^2 \zeta}{n_0} X^\top+\frac{\zeta}{\sqrt{n_0}n_1}F^\top W_1)(\frac{\sigma_{W_2}^2 \zeta}{n_0} \tilde{X}^\top+\frac{\sqrt{\zeta}}{\sqrt{n_0}n_1}\tilde{F}^\top \tilde{W}_1)^\top\\
\mathbb{E}_\bfx K(X,\bfx;P)K(\bfx,X;\tilde{P}) & = (\frac{\sigma_{W_2}^2 \zeta}{n_0} X^\top+\frac{\zeta}{\sqrt{n_0}n_1}F^\top W_1)(\frac{\sigma_{W_2}^2 \zeta}{n_0} X^\top+\frac{\sqrt{\zeta}}{\sqrt{n_0}n_1}f(\tilde{W}_1X)^\top \tilde{W}_1)^\top\\
\mathbb{E}_\bfx K(X,\bfx;P)K(\bfx,\tilde{X};P) & = (\frac{\sigma_{W_2}^2 \zeta}{n_0} X^\top+\frac{\zeta}{\sqrt{n_0}n_1}F^\top W_1)(\frac{\sigma_{W_2}^2 \zeta}{n_0} \tilde{X}^\top+\frac{\sqrt{\zeta}}{\sqrt{n_0}n_1}f(W_1\tilde{X})^\top W_1)^\top + \frac{\eta-\zeta}{n_1^2}F^\top f(W_1\tilde{X})
\end{align}

\subsection{$H_{000}$}
\begin{align}
H_{000} &= \mathbb{E}\hat{y}(\bfx; P,X,\varepsilon)\hat{y}(\bfx; \tilde{P},\tilde{X},\tilde{\varepsilon})\\
&= \mathbb{E} K(\bfx,\tilde{X};\tilde{P})K(\tilde{X},\tilde{X};\tilde{P})^{-1} Y(\tilde{X},\tilde{\varepsilon})^\top Y(X,\varepsilon)K(X,X;P)^{-1}K(X,\bfx;P)\\
&= \mathbb{E} \tr \big(K(\tilde{X},\tilde{X};\tilde{P})^{-1} \tilde{X}^\top X K(X,X;P)^{-1}K(X,\bfx;P)K(\bfx,\tilde{X};\tilde{P})\big)\\
&= \mathbb{E} \tr \big(K(\tilde{X},\tilde{X};\tilde{P})^{-1} \tilde{X}^\top X K(X,X;P)^{-1}  (\frac{\sigma_{W_2}^2 \zeta}{n_0} X^\top+\frac{\sqrt{\zeta}}{\sqrt{n_0}n_1}F^\top W_1)(\frac{\sigma_{W_2}^2 \zeta}{n_0} \tilde{X}^\top+\frac{\sqrt{\zeta}}{\sqrt{n_0}n_1}\tilde{F}^\top \tilde{W}_1)^\top\big)\\
&=\tr\big(X K^{-1} (\frac{\sigma_{W_2}^2 \zeta}{n_0} X^\top+\frac{\sqrt{\zeta}}{\sqrt{n_0}n_1}F^\top W_1)\big)^2\\
&\equiv E_4
\end{align}
A linear pencil for $E_4$ follows from the representation,
\eq{
E_{4} = \tr(U_{4}^T Q_{4}^{-1} V_{4})^2\,,
}
where,
\eq{
    U_{4}^T  = 
\begin{pmatrix}
0 & \frac{\zeta  I_{n_0} \sigma_{W_2}^2}{n_0} & 0 & 0 & \frac{\sqrt{\zeta } I_{m}}{\sqrt{n_0} n_1}
\end{pmatrix}\,,\quad
V_{4}^T = 
\begin{pmatrix}
0 & -\frac{n_0 I_{n_0}}{\zeta  \sigma_{W_2}^2} & 0 & 0 & \frac{\sqrt{n_0} n_1 I_{m}}{\sqrt{\zeta }}
\end{pmatrix}
}
and,
\eq{
Q_{4} =
\left(
\begin{smallmatrix}
I_m \left(\gamma +\sigma_{W_2}^2 \left(\eta '-\zeta \right)\right) & \frac{\zeta  X^\top \sigma_{W_2}^2}{n_0} & \frac{\sqrt{\eta -\zeta } \Theta_F^\top}{n_1} & \frac{\sqrt{\zeta } X^\top}{\sqrt{n_0} n_1} & 0 \\
 -X & I_{n_0} & 0 & 0 & 0 \\
 -\sqrt{\eta -\zeta } \Theta_F & -\frac{\sqrt{\zeta } W_1}{\sqrt{n_0}} & I_{n_1} & 0 & 0 \\
 0 & 0 & -W_1^\top & I_{n_0} & 0 \\
 0 & 0 & 0 & 0 & I_{m}
\end{smallmatrix}
\right)\,.
}
The equations satisfied by the operator-valued Stieltjes transform $G$ of $\bar{Q}_{4}$ induce the following structure on $G$,
\eq{
G = \begin{pmatrix} 0 & G_{12} \\ G_{12}^\top & 0 \end{pmatrix}\,,
}
where,
\eq{
G_{12} =\left(
\begin{array}{ccccc}
 g_3 & 0 & 0 & 0 & 0 \\
 0 & g_4 & 0 & g_2 & 0 \\
 0 & 0 & g_6 & 0 & 0 \\
 0 & g_7 & 0 & g_5 & 0 \\
 0 & 0 & 0 & 0 & g_1 \\
\end{array}
\right)\,,
}
and the independent entry-wise component functions $g_i$ give the error $E_4$ through the relation,
\eq{
E_4 = (g_1-g_4)^2\,,
}
and themselves satisfy the following system of polynomial equations,
\begin{subequations}
\begin{align}
0 &= 1-g_1\\
0 &= \sqrt{\zeta } g_4 g_6 \sqrt{n_0}-g_7 \psi\\
0 &= \sqrt{\zeta } g_2 g_6 \sqrt{n_0}-g_5 \psi +\psi\\
0 &= -\sqrt{\zeta } g_3 g_4 \psi -g_2 \sqrt{n_0} \phi  \big(\gamma +\sigma _{W_2}^2 \big(\eta '-\zeta \big)\big)\\
0 &= -\sqrt{\zeta } g_3 g_7 \psi -\big(g_5-1\big) \sqrt{n_0} \phi  \big(\gamma +\sigma _{W_2}^2 \big(\eta '-\zeta \big)\big)\\
0 &= g_2 \sqrt{n_0} \phi  \big(\gamma +\sigma _{W_2}^2 \big(\eta '-\zeta \big)\big)+g_3 \big(\sqrt{\zeta } g_5 \psi +\zeta  g_2 \sqrt{n_0} \sigma _{W_2}^2\big)\\
0 &= \big(g_4-1\big) \sqrt{n_0} \phi  \big(\gamma +\sigma _{W_2}^2 \big(\eta '-\zeta \big)\big)+g_3 \big(\sqrt{\zeta } g_7 \psi +\zeta  g_4 \sqrt{n_0} \sigma _{W_2}^2\big)\\
0 &= \sqrt{\zeta } g_5 g_6 \sqrt{n_0} \phi  \big(\sigma _{W_2}^2 \big(\zeta -\eta '\big)-\gamma \big)+g_7 \psi  \big(\gamma  \phi +\sigma _{W_2}^2 \big(-\zeta  \phi +\phi  \eta '+\zeta  g_3\big)\big)\\
0 &= g_6 \big(g_3 \psi  (\zeta -\eta )+\phi  \big(\sqrt{\zeta } g_2 \sqrt{n_0}-1\big) \big(\gamma +\sigma _{W_2}^2 \big(\eta '-\zeta \big)\big)\big)+\phi  \big(\gamma +\sigma _{W_2}^2 \big(\eta '-\zeta \big)\big)\\
0 &= g_3 \big(\sqrt{\zeta } g_7 \psi +\sqrt{n_0} \big(\gamma +g_6 (\eta -\zeta )+\sigma _{W_2}^2 \big(\eta '+\zeta  \big(g_4-1\big)\big)\big)\big)-\sqrt{n_0} \big(\gamma +\sigma _{W_2}^2 \big(\eta '-\zeta \big)\big)\\
0 &= g_4 \psi  \big(\gamma  \phi +\sigma _{W_2}^2 \big(-\zeta  \phi +\phi  \eta '+\zeta  g_3\big)\big)-\phi  \big(\sqrt{\zeta } g_2 g_6 \sqrt{n_0}+\psi \big) \big(\gamma +\sigma _{W_2}^2 \big(\eta '-\zeta \big)\big)
\end{align}

After some straightforward algebra, one can eliminate all $g_i$ except for $g_2$ and $g_{3}$, which satisfy coupled polynomial equations. Those equations can be shown to be identical to eqn.~\eqref{eqn:tau} by invoking the change of variables,
\begin{equation}
    g_2 = - \frac{\sqrt{\zeta} \psi}{\sqrt{n_0} \phi} \tau_2\,,\quad\text{and}\quad
    g_3 = \big(\gamma +\sigma _{W_2}^2 \big(\eta '-\zeta \big)\big)\tau_1\,.
\end{equation}

In terms of the related variables defined in eqn.~\eqref{eq:ttau}, the error $E_{4}$ is given by,
\begin{equation}
    E_{4} = \ttau_2^2\,.
\end{equation}

\end{subequations}

\subsection{$H_{001}$}
\begin{align}
H_{001} &= \mathbb{E}\hat{y}(\bfx; P,X,\varepsilon)\hat{y}(\bfx; \tilde{P},\tilde{X},\varepsilon)\\
&= \mathbb{E} K(\bfx,\tilde{X};\tilde{P})K(\tilde{X},\tilde{X};\tilde{P})^{-1} Y(\tilde{X},\varepsilon)^\top Y(X,\varepsilon)K(X,X;P)^{-1}K(X,\bfx;P)\\
&= \mathbb{E} \tr \big(K(\tilde{X},\tilde{X};\tilde{P})^{-1} \tilde{X}^\top X K(X,X;P)^{-1}K(X,\bfx;P)K(\bfx,\tilde{X};\tilde{P})\big)\\
& = H_{000}
\end{align}

\subsection{$H_{010}$}
\begin{align}
H_{010} &= \mathbb{E}\hat{y}(\bfx; P,X,\varepsilon)\hat{y}(\bfx; \tilde{P},X,\tilde{\varepsilon})\\
&= \mathbb{E} K(\bfx,X;\tilde{P})K(X,X;\tilde{P})^{-1} Y(X,\tilde{\varepsilon})^\top Y(X,\varepsilon)K(X,X;P)^{-1}K(X,\bfx;P)\\
&= \mathbb{E} \tr \big(K(X,X;\tilde{P})^{-1} X^\top X K(X,X;P)^{-1}K(X,\bfx;P)K(\bfx,X;\tilde{P})\big)\\
&= \mathbb{E} \tr \Big(K(X,X;\tilde{P})^{-1} X^\top X K(X,X;P)^{-1}K(X,\bfx;P)K(\bfx,X;\tilde{P})\\
&\qquad\qquad\; \times (\frac{\sigma_{W_2}^2 \zeta}{n_0} X^\top+\frac{\zeta}{\sqrt{n_0}n_1}F^\top W_1)(\frac{\sigma_{W_2}^2 \zeta}{n_0} X^\top+\frac{\sqrt{\zeta}}{\sqrt{n_0}n_1}f(\tilde{W}_1X)^\top \tilde{W}_1)^\top\Big)\\
&\equiv E_5\,.
\end{align}
A linear pencil for $E_5$ follows from the representation,
\eq{
E_{5} = \tr(U_{5}^T Q_{5}^{-1} V_{5})\,,
}
where,
\eq{
    U_{5}^T  = 
\begin{pmatrix}
 0 & \frac{I_{n_0}}{m} & 0 & 0 & 0 & 0 & 0 & 0 & 0 
\end{pmatrix}\,,\quad
V_{5}^T = \begin{pmatrix}
0 & 0 & 0 & 0 & 0 & 0 & 0 & 0 & -\frac{\sqrt{n_0} n_1 I_{n_0}}{\sqrt{\zeta }}
\end{pmatrix}
}
and,
\eq{
Q_{5} =
\left(\begin{smallmatrix}
 I_m \left(\gamma +\sigma_{W_2}^2 \left(\eta '-\zeta \right)\right) & 0 & \frac{\zeta  X^\top \sigma_{W_2}^2}{n_0} & \frac{\sqrt{\eta -\zeta } \Theta_F^\top}{n_1} & \frac{\sqrt{\zeta } X^\top}{\sqrt{n_0} n_1} & -\frac{\zeta ^2 m X^\top \sigma_{W_2}^4}{n_0^2} & 0 & 0 & 0 \\
 -X & I_{n_0} & 0 & 0 & 0 & 0 & 0 & 0 & 0 \\
 -X & 0 & I_{n_0} & 0 & 0 & 0 & 0 & \frac{\sqrt{\zeta } m \tilde{W}_1^\top}{\sqrt{n_0} n_1} & 0 \\
 -\sqrt{\eta -\zeta } \Theta_F & 0 & -\frac{\sqrt{\zeta } W_1}{\sqrt{n_0}} & I_{n_1} & 0 & \frac{\zeta ^{3/2} m W_1 \sigma_{W_2}^2}{n_0^{3/2}} & 0 & 0 & 0 \\
 0 & 0 & 0 & -W_1^\top & I_{n_0} & 0 & 0 & 0 & 0 \\
 0 & 0 & 0 & 0 & 0 & I_{n_0} & -X & 0 & 0 \\
 0 & 0 & 0 & 0 & 0 & \frac{\zeta  X^\top \sigma_{W_2}^2}{n_0} & I_m \left(\gamma +\sigma_{W_2}^2 \left(\eta '-\zeta \right)\right) & \frac{\sqrt{\eta -\zeta } \tilde{\Theta }_F^\top}{n_1} & \frac{\sqrt{\zeta } X^\top}{\sqrt{n_0} n_1} \\
 0 & 0 & 0 & 0 & 0 & -\frac{\sqrt{\zeta } \tilde{W}_1}{\sqrt{n_0}} & -\sqrt{\eta -\zeta } \tilde{\Theta }_F & I_{n_1} & 0 \\
 0 & 0 & 0 & 0 & 0 & 0 & 0 & -\tilde{W}_1^\top & I_{n_0} \\
\end{smallmatrix}\right)\,.
}
The equations satisfied by the operator-valued Stieltjes transform $G$ of $\bar{Q}_{5}$ induce the following structure on $G$,
\eq{
G = \begin{pmatrix} 0 & G_{12} \\ G_{12}^\top & 0 \end{pmatrix}\,,
}
where,
\eq{
G_{12} =\left(
\begin{array}{ccccccccc}
 g_9 & 0 & 0 & 0 & 0 & 0 & g_6 & 0 & 0 \\
 0 & g_1 & g_5 & 0 & g_8 & g_3 & 0 & 0 & g_2 \\
 0 & 0 & g_{10} & 0 & g_8 & g_{13} & 0 & 0 & g_7 \\
 0 & 0 & 0 & g_{12} & 0 & 0 & 0 & 0 & 0 \\
 0 & 0 & g_{15} & 0 & g_{11} & g_{14} & 0 & 0 & g_4 \\
 0 & 0 & 0 & 0 & 0 & g_{10} & 0 & 0 & g_8 \\
 0 & 0 & 0 & 0 & 0 & 0 & g_9 & 0 & 0 \\
 0 & 0 & 0 & 0 & 0 & 0 & 0 & g_{12} & 0 \\
 0 & 0 & 0 & 0 & 0 & g_{15} & 0 & 0 & g_{11} \\
\end{array}
\right)\,,
}
and the independent entry-wise component functions $g_i$ give the error $E_5$ through the relation,
\eq{
E_5 = -\frac{g_2 \sqrt{n_0} \phi }{\sqrt{\zeta } \psi }\,,
}
and themselves satisfy the following system of polynomial equations,
{\small
\begin{subequations}
\begin{align}
0 &= 1-g_1\\
0 &= \sqrt{\zeta } g_{10} g_{12} \sqrt{n_0}-g_{15} \psi\\
0 &= \sqrt{\zeta } g_8 g_{12} \sqrt{n_0}-g_{11} \psi +\psi\\
0 &= \sqrt{\zeta } g_{12} \sqrt{n_0} \big(g_7 \phi -\zeta  g_8 \sigma _{W_2}^2\big)-g_4 \psi  \phi\\
0 &= \sqrt{\zeta } g_{12} \sqrt{n_0} \big(g_{13} \phi -\zeta  g_{10} \sigma _{W_2}^2\big)-g_{14} \psi  \phi\\
0 &= -\sqrt{\zeta } g_9 g_{10} \psi -g_8 \sqrt{n_0} \phi  \big(\gamma +\sigma _{W_2}^2 \big(\eta '-\zeta \big)\big)\\
0 &= -\sqrt{\zeta } g_9 g_{15} \psi -\big(g_{11}-1\big) \sqrt{n_0} \phi  \big(\gamma +\sigma _{W_2}^2 \big(\eta '-\zeta \big)\big)\\
0 &= -\sqrt{\zeta } g_1 g_9 \psi -\sqrt{\zeta } g_5 g_9 \psi -g_8 \sqrt{n_0} \phi  \big(\gamma +\sigma _{W_2}^2 \big(\eta '-\zeta \big)\big)\\
0 &= -\sqrt{\zeta } g_6 g_{10} \psi -\sqrt{\zeta } g_9 g_{13} \psi -g_7 \sqrt{n_0} \phi  \big(\gamma +\sigma _{W_2}^2 \big(\eta '-\zeta \big)\big)\\
0 &= -\sqrt{\zeta } g_9 g_{14} \psi -\sqrt{\zeta } g_6 g_{15} \psi -g_4 \sqrt{n_0} \phi  \big(\gamma +\sigma _{W_2}^2 \big(\eta '-\zeta \big)\big)\\
0 &= g_8 \sqrt{n_0} \phi  \big(\gamma +\sigma _{W_2}^2 \big(\eta '-\zeta \big)\big)+g_9 \big(\sqrt{\zeta } g_{11} \psi +\zeta  g_8 \sqrt{n_0} \sigma _{W_2}^2\big)\\
0 &= g_5 \sqrt{n_0} \phi  \big(\gamma +\sigma _{W_2}^2 \big(\eta '-\zeta \big)\big)+g_9 \big(\sqrt{\zeta } g_{15} \psi +\zeta  g_{10} \sqrt{n_0} \sigma _{W_2}^2\big)\\
0 &= \big(g_{10}-1\big) \sqrt{n_0} \phi  \big(\gamma +\sigma _{W_2}^2 \big(\eta '-\zeta \big)\big)+g_9 \big(\sqrt{\zeta } g_{15} \psi +\zeta  g_{10} \sqrt{n_0} \sigma _{W_2}^2\big)\\
0 &= -\sqrt{\zeta } \big(\big(g_1+g_5\big) g_6+g_3 g_9\big) \psi -\gamma  g_2 \sqrt{n_0} \phi +\zeta  g_2 \sqrt{n_0} \phi  \sigma _{W_2}^2+g_2 \sqrt{n_0} (-\phi ) \eta ' \sigma _{W_2}^2\\
0 &= \sqrt{\zeta } g_{11} g_{12} \sqrt{n_0} \phi  \big(\sigma _{W_2}^2 \big(\zeta -\eta '\big)-\gamma \big)+g_{15} \psi  \big(\gamma  \phi +\sigma _{W_2}^2 \big(-\zeta  \phi +\phi  \eta '+\zeta  g_9\big)\big)\\
0 &= g_{12} \big(g_9 \psi  (\zeta -\eta )+\phi  \big(\sqrt{\zeta } g_8 \sqrt{n_0}-1\big) \big(\gamma +\sigma _{W_2}^2 \big(\eta '-\zeta \big)\big)\big)+\phi  \big(\gamma +\sigma _{W_2}^2 \big(\eta '-\zeta \big)\big)\\
0 &= g_{10} \psi  \big(\gamma  \phi +\sigma _{W_2}^2 \big(-\zeta  \phi +\phi  \eta '+\zeta  g_9\big)\big)-\phi  \big(\sqrt{\zeta } g_8 g_{12} \sqrt{n_0}+\psi \big) \big(\gamma +\sigma _{W_2}^2 \big(\eta '-\zeta \big)\big)\\
0 &= g_9 \big(\sqrt{\zeta } g_{15} \psi +\sqrt{n_0} \big(\gamma +g_{12} (\eta -\zeta )+\sigma _{W_2}^2 \big(\eta '+\zeta  \big(g_{10}-1\big)\big)\big)\big)-\sqrt{n_0} \big(\gamma +\sigma _{W_2}^2 \big(\eta '-\zeta \big)\big)\\
0 &= -\sqrt{\zeta } g_8 g_{12} \sqrt{n_0} \phi  \big(\gamma +\sigma _{W_2}^2 \big(\eta '-\zeta \big)\big)+g_5 \psi  \big(\gamma  \phi +\sigma _{W_2}^2 \big(-\zeta  \phi +\phi  \eta '+\zeta  g_9\big)\big)+\zeta  g_1 g_9 \psi  \sigma _{W_2}^2\\
0 &= \sqrt{\zeta } g_4 g_9 \psi  \phi +\sqrt{n_0} \big(g_2 \phi ^2 \big(\gamma +\sigma _{W_2}^2 \big(\eta '-\zeta \big)\big)+\zeta  g_9 \sigma _{W_2}^2 \big(g_7 \phi -\zeta  g_8 \sigma _{W_2}^2\big)\big)+g_6 \big(\sqrt{\zeta } g_{11} \psi  \phi +\zeta  g_8 \sqrt{n_0} \phi  \sigma _{W_2}^2\big)\\
0 &= g_6 \phi  \big(\sqrt{\zeta } g_{15} \psi +\sqrt{n_0} \big(\gamma +g_{12} (\eta -\zeta )+\sigma _{W_2}^2 \big(\eta '+\zeta  \big(g_{10}-1\big)\big)\big)\big)\nonumber\\
&\quad\quad+g_9 \big(\sqrt{\zeta } g_{14} \psi  \phi +\zeta  \sqrt{n_0} \sigma _{W_2}^2 \big(g_{13} \phi -\zeta  g_{10} \sigma _{W_2}^2\big)\big)\\
0 &= \phi  \big(g_3 \sqrt{n_0} \phi  \big(\gamma +\sigma _{W_2}^2 \big(\eta '-\zeta \big)\big)+g_6 \big(\sqrt{\zeta } g_{15} \psi +\zeta  g_{10} \sqrt{n_0} \sigma _{W_2}^2\big)\big)+g_9 \big(\sqrt{\zeta } g_{14} \psi  \phi +\zeta  \sqrt{n_0} \sigma _{W_2}^2 \big(g_{13} \phi -\zeta  g_{10} \sigma _{W_2}^2\big)\big)\\
0 &= \phi  \big(\sqrt{n_0} \big(\zeta  g_{10} g_{12}+g_{13} \phi \big) \big(\gamma +\sigma _{W_2}^2 \big(\eta '-\zeta \big)\big)+g_6 \big(\sqrt{\zeta } g_{15} \psi +\zeta  g_{10} \sqrt{n_0} \sigma _{W_2}^2\big)\big)\nonumber\\
&\quad\quad+g_9 \big(\sqrt{\zeta } g_{14} \psi  \phi +\zeta  \sqrt{n_0} \sigma _{W_2}^2 \big(g_{13} \phi -\zeta  g_{10} \sigma _{W_2}^2\big)\big)\\
0 &= \sqrt{\zeta } g_4 g_9 \psi  \phi +\sqrt{n_0} \big(g_7 \phi  \big(\gamma  \phi +\sigma _{W_2}^2 \big(-\zeta  \phi +\phi  \eta '+\zeta  g_9\big)\big)+\zeta  g_8 \big(g_{12} \phi  \big(\gamma +\sigma _{W_2}^2 \big(\eta '-\zeta \big)\big)-\zeta  g_9 \sigma _{W_2}^4\big)\big)\nonumber\\
&\quad\quad+g_6 \big(\sqrt{\zeta } g_{11} \psi  \phi +\zeta  g_8 \sqrt{n_0} \phi  \sigma _{W_2}^2\big)\\
0 &= \sqrt{\zeta } g_{12} \phi  \big(\gamma +\sigma _{W_2}^2 \big(\eta '-\zeta \big)\big) \big(\sqrt{\zeta } g_{15} \psi +\sqrt{n_0} \big(\zeta  g_{11} \sigma _{W_2}^2-g_4 \phi \big)\big)+g_{14} \psi  \phi  \big(\gamma  \phi +\sigma _{W_2}^2 \big(-\zeta  \phi +\phi  \eta '+\zeta  g_9\big)\big)\nonumber\\
&\quad\quad+\zeta  g_{15} \psi  \sigma _{W_2}^2 \big(g_6 \phi -\zeta  g_9 \sigma _{W_2}^2\big)\\
0 &= \phi  \big(g_{13} \psi  \big(\gamma  \phi +\sigma _{W_2}^2 \big(-\zeta  \phi +\phi  \eta '+\zeta  g_9\big)\big)-\sqrt{\zeta } g_{12} \sqrt{n_0} \big(g_7 \phi -\zeta  g_8 \sigma _{W_2}^2\big) \big(\gamma +\sigma _{W_2}^2 \big(\eta '-\zeta \big)\big)\big)\nonumber\\
&\quad\quad+\zeta  g_{10} \psi  \big(g_{12} \phi  \big(\gamma +\sigma _{W_2}^2 \big(\eta '-\zeta \big)\big)-\zeta  g_9 \sigma _{W_2}^4+g_6 \phi  \sigma _{W_2}^2\big)\\
0 &= -\sqrt{\zeta } g_2 g_{12} \sqrt{n_0} \phi ^2 \big(\gamma +\sigma _{W_2}^2 \big(\eta '-\zeta \big)\big)+\zeta  \sigma _{W_2}^2 \big(\sqrt{\zeta } g_8 g_{12} \sqrt{n_0} \phi  \big(\gamma +\sigma _{W_2}^2 \big(\eta '-\zeta \big)\big)+g_1 \psi  \big(g_6 \phi -\zeta  g_9 \sigma _{W_2}^2\big)\big)\nonumber\\
&\quad\quad+g_3 \psi  \phi  \big(\gamma  \phi +\sigma _{W_2}^2 \big(\phi  \big(\eta '-\zeta \big)+\zeta  g_9\big)\big)+\zeta  g_5 \psi  \big(g_{12} \phi  \big(\gamma +\sigma _{W_2}^2 \big(\eta '-\zeta \big)\big)-\zeta  g_9 \sigma _{W_2}^4+g_6 \phi  \sigma _{W_2}^2\big)
\end{align}
\end{subequations}
}

After some straightforward algebra, one can eliminate all $g_i$ except for $g_8$ and $g_9$, which satisfy coupled polynomial equations. Those equations can be shown to be identical to eqn.~\eqref{eqn:tau} by invoking the change of variables,
\begin{equation}
    g_8 = - \frac{\sqrt{\zeta} \psi}{\sqrt{n_0} \phi} \tau_2\,,\quad\text{and}\quad
    g_9 = \big(\gamma +\sigma _{W_2}^2 \big(\eta '-\zeta \big)\big)\tau_1\,.
\end{equation}

In terms of the related variables defined in eqn.~\eqref{eq:ttau}, the error $E_{4}$ is given by,
\begin{equation}
    E_{5} = \ttau_2^2 (1+\phi+2 \ttau_2 \phi)/(1- \ttau_2^2 \phi)\,.
\end{equation}

\subsection{$H_{011}$}
\begin{align}
H_{011} &= \mathbb{E}\hat{y}(\bfx; P,X,\varepsilon)\hat{y}(\bfx; \tilde{P},X,\tilde{\varepsilon})\\
&= \mathbb{E} K(\bfx,X;\tilde{P})K(X,X;\tilde{P})^{-1} Y(X,\tilde{\varepsilon})^\top Y(X,\varepsilon)K(X,X;P)^{-1}K(X,\bfx;P)\\
&= \mathbb{E} \tr \big(K(X,X;\tilde{P})^{-1} (X^\top X + \sigma_{\varepsilon}^2 n_1 I_m) K(X,X;P)^{-1}K(X,\bfx;P)K(\bfx,X;\tilde{P})\big)\\
&= \mathbb{E} \tr \Big(K(X,X;\tilde{P})^{-1} (X^\top X + \sigma_{\varepsilon}^2 n_1 I_m) K(X,X;P)^{-1}K(X,\bfx;P)K(\bfx,X;\tilde{P})\\
&\qquad\qquad\; \times (\frac{\sigma_{W_2}^2 \zeta}{n_0} X^\top+\frac{\zeta}{\sqrt{n_0}n_1}F^\top W_1)(\frac{\sigma_{W_2}^2 \zeta}{n_0} X^\top+\frac{\sqrt{\zeta}}{\sqrt{n_0}n_1}f(\tilde{W}_1X)^\top \tilde{W}_1)^\top\Big)\\
&\equiv H_{010}+E_6\,,
\end{align}
where,
\begin{align}
E_6 &= \sigma_{\varepsilon}^2 n_1 \mathbb{E} \tr \Big(K(X,X;\tilde{P})^{-1} K(X,X;P)^{-1} \\ &\quad\quad\times(\frac{\sigma_{W_2}^2 \zeta}{n_0} X^\top+\frac{\zeta}{\sqrt{n_0}n_1}F^\top W_1)(\frac{\sigma_{W_2}^2 \zeta}{n_0} X^\top+\frac{\sqrt{\zeta}}{\sqrt{n_0}n_1}f(\tilde{W}_1X)^\top \tilde{W}_1)^\top\Big)\,.
\end{align}
A linear pencil for $E_6$ follows from the representation,
\eq{
E_{6} = \tr(U_{6}^T Q_{6}^{-1} V_{6})\,,
}
where,
\eq{
    U_{6}^T  = 
\begin{pmatrix}
\sigma_{\varepsilon }^2 I_m & 0 & 0 & 0 & 0 & 0 & 0 & 0 
\end{pmatrix}\,,\quad
V_{6}^T = \begin{pmatrix}
 0 & 0 & 0 & 0 & 0 & \frac{I_m}{\gamma +\sigma_{W_2}^2 \left(\eta '-\zeta \right)} & 0 & 0 \\
\end{pmatrix}
}
and,
\eq{
Q_{6} = \left(\begin{smallmatrix}
I_m \left(\gamma +\sigma_{W_2}^2 \left(\eta '-\zeta \right)\right) & \frac{\zeta  X^\top \sigma_{W_2}^2}{n_0} & \frac{\sqrt{\eta -\zeta } \Theta_F^\top}{n_1} & \frac{\sqrt{\zeta } X^\top}{\sqrt{n_0} n_1} & -\frac{\zeta ^2 m X^\top \sigma_{W_2}^4}{n_0^2} & 0 & -\frac{\zeta ^{3/2} m X^\top \sigma_{W_2}^2}{n_0^{3/2} n_1} & 0 \\
 -X & I_{n_0} & 0 & 0 & 0 & 0 & 0 & 0 \\
 -\sqrt{\eta -\zeta } \Theta_F & -\frac{\sqrt{\zeta } W_1}{\sqrt{n_0}} & I_{n_1} & 0 & \frac{\zeta ^{3/2} m W_1 \sigma_{W_2}^2}{n_0^{3/2}} & 0 & \frac{\zeta  m W_1}{n_0 n_1} & 0 \\
 0 & 0 & -W_1^\top & I_{n_0} & 0 & 0 & 0 & 0 \\
 0 & 0 & 0 & 0 & I_{n_0} & -X & 0 & 0 \\
 0 & 0 & 0 & 0 & \frac{\zeta  X^\top \sigma_{W_2}^2}{n_0} & I_m \left(\gamma +\sigma_{W_2}^2 \left(\eta '-\zeta \right)\right) & \frac{\sqrt{\zeta } X^\top}{\sqrt{n_0} n_1} & \frac{\sqrt{\eta -\zeta } \tilde{\Theta }_F^\top}{n_1} \\
 0 & 0 & 0 & 0 & 0 & 0 & I_{n_0} & -\tilde{W}_1^\top \\
 0 & 0 & 0 & 0 & -\frac{\sqrt{\zeta } \tilde{W}_1}{\sqrt{n_0}} & -\sqrt{\eta -\zeta } \tilde{\Theta }_F & 0 & I_{n_1} \\
\end{smallmatrix}\right)\,.
}
The equations satisfied by the operator-valued Stieltjes transform $G$ of $\bar{Q}_{6}$ induce the following structure on $G$,
\eq{
G = \begin{pmatrix} 0 & G_{12} \\ G_{12}^\top & 0 \end{pmatrix}\,,
}
where,
\eq{
G_{12} =\left(
\begin{array}{cccccccc}
 g_5 & 0 & 0 & 0 & 0 & g_2 & 0 & 0 \\
 0 & g_6 & 0 & g_3 & g_1 & 0 & g_4 & 0 \\
 0 & 0 & g_8 & 0 & 0 & 0 & 0 & 0 \\
 0 & g_{11} & 0 & g_7 & g_{10} & 0 & g_9 & 0 \\
 0 & 0 & 0 & 0 & g_6 & 0 & g_3 & 0 \\
 0 & 0 & 0 & 0 & 0 & g_5 & 0 & 0 \\
 0 & 0 & 0 & 0 & g_{11} & 0 & g_7 & 0 \\
 0 & 0 & 0 & 0 & 0 & 0 & 0 & g_8 \\
\end{array}
\right)\,,
}
and the independent entry-wise component functions $g_i$ give the error $E_6$ through the relation,
\eq{
E_6 = \frac{g_2 \sigma _{\varepsilon }^2}{\big(\gamma +\sigma _{W_2}^2 \big(\eta '-\zeta \big)\big)}\,,
}
and themselves satisfy the following system of polynomial equations,
{\small
\begin{subequations}
\begin{align}
0 &= \sqrt{\zeta } g_6 g_8 \sqrt{n_0}-g_{11} \psi\\
0 &= \sqrt{\zeta } g_3 g_8 \sqrt{n_0}-g_7 \psi +\psi\\
0 &= -\sqrt{\zeta } g_5 g_6 \psi -g_3 \sqrt{n_0} \phi  \big(\gamma +\sigma _{W_2}^2 \big(\eta '-\zeta \big)\big)\\
0 &= -\zeta  g_7 g_8 \psi +\sqrt{\zeta } g_8 \sqrt{n_0} \big(g_4 \phi -\zeta  g_3 \sigma _{W_2}^2\big)-g_9 \psi  \phi\\
0 &= -\sqrt{\zeta } g_5 g_{11} \psi -\big(g_7-1\big) \sqrt{n_0} \phi  \big(\gamma +\sigma _{W_2}^2 \big(\eta '-\zeta \big)\big)\\
0 &= -\zeta  g_8 g_{11} \psi +\sqrt{\zeta } g_8 \sqrt{n_0} \big(g_1 \phi -\zeta  g_6 \sigma _{W_2}^2\big)+g_{10} \psi  (-\phi )\\
0 &= g_3 \sqrt{n_0} \phi  \big(\gamma +\sigma _{W_2}^2 \big(\eta '-\zeta \big)\big)+g_5 \big(\sqrt{\zeta } g_7 \psi +\zeta  g_3 \sqrt{n_0} \sigma _{W_2}^2\big)\\
0 &= \big(g_6-1\big) \sqrt{n_0} \phi  \big(\gamma +\sigma _{W_2}^2 \big(\eta '-\zeta \big)\big)+g_5 \big(\sqrt{\zeta } g_{11} \psi +\zeta  g_6 \sqrt{n_0} \sigma _{W_2}^2\big)\\
0 &= \sqrt{\zeta } g_7 g_8 \sqrt{n_0} \phi  \big(\sigma _{W_2}^2 \big(\zeta -\eta '\big)-\gamma \big)+g_{11} \psi  \big(\gamma  \phi +\sigma _{W_2}^2 \big(-\zeta  \phi +\phi  \eta '+\zeta  g_5\big)\big)\\
0 &= g_8 \big(g_5 \psi  (\zeta -\eta )+\phi  \big(\sqrt{\zeta } g_3 \sqrt{n_0}-1\big) \big(\gamma +\sigma _{W_2}^2 \big(\eta '-\zeta \big)\big)\big)+\phi  \big(\gamma +\sigma _{W_2}^2 \big(\eta '-\zeta \big)\big)\\
0 &= g_6 \psi  \big(\gamma  \phi +\sigma _{W_2}^2 \big(-\zeta  \phi +\phi  \eta '+\zeta  g_5\big)\big)-\phi  \big(\sqrt{\zeta } g_3 g_8 \sqrt{n_0}+\psi \big) \big(\gamma +\sigma _{W_2}^2 \big(\eta '-\zeta \big)\big)\\
0 &= g_5 \big(\sqrt{\zeta } g_{11} \psi +\sqrt{n_0} \big(\gamma +g_8 (\eta -\zeta )+\sigma _{W_2}^2 \big(\eta '+\zeta  \big(g_6-1\big)\big)\big)\big)-\sqrt{n_0} \big(\gamma +\sigma _{W_2}^2 \big(\eta '-\zeta \big)\big)\\
0 &= \sqrt{\zeta } g_2 g_{11} \psi  \phi +\sqrt{n_0} \phi  \big(\zeta  g_7 g_8+g_9 \phi \big) \big(\gamma +\sigma _{W_2}^2 \big(\eta '-\zeta \big)\big)+\sqrt{\zeta } g_5 \psi  \big(g_{10} \phi -\zeta  g_{11} \sigma _{W_2}^2\big)\\
0 &= g_2 \phi  \big(\sqrt{\zeta } g_{11} \psi +\sqrt{n_0} \big(\gamma +g_8 (\eta -\zeta )+\sigma _{W_2}^2 \big(\eta '+\zeta  \big(g_6-1\big)\big)\big)\big)\nonumber\\
&\quad\quad+g_5 \big(\sqrt{\zeta } g_{10} \psi  \phi -\zeta  \sigma _{W_2}^2 \big(\sqrt{\zeta } g_{11} \psi +\sqrt{n_0} \big(\zeta  g_6 \sigma _{W_2}^2-g_1 \phi \big)\big)\big)\\
0 &= g_4 \sqrt{n_0} \phi ^2 \big(\gamma +\sigma _{W_2}^2 \big(\eta '-\zeta \big)\big)+g_2 \big(\sqrt{\zeta } g_7 \psi  \phi +\zeta  g_3 \sqrt{n_0} \phi  \sigma _{W_2}^2\big)+g_5 \big(\sqrt{\zeta } g_9 \psi  \phi\nonumber\\
&\quad\quad -\zeta  \sigma _{W_2}^2 \big(\sqrt{\zeta } g_7 \psi +\sqrt{n_0} \big(\zeta  g_3 \sigma _{W_2}^2-g_4 \phi \big)\big)\big)\\
0 &= -\sqrt{\zeta } g_8 \sqrt{n_0} \phi  \big(g_9 \phi -\zeta  g_7 \sigma _{W_2}^2\big) \big(\gamma +\sigma _{W_2}^2 \big(\eta '-\zeta \big)\big)+g_{10} \psi  \phi  \big(\gamma  \phi +\sigma _{W_2}^2 \big(-\zeta  \phi +\phi  \eta '+\zeta  g_5\big)\big)\nonumber\\
&\quad\quad+\zeta  g_{11} \psi  \sigma _{W_2}^2 \big(g_2 \phi -\zeta  g_5 \sigma _{W_2}^2\big)\\
0 &= \phi  \big(g_1 \sqrt{n_0} \phi  \big(\gamma +\sigma _{W_2}^2 \big(\eta '-\zeta \big)\big)+g_2 \big(\sqrt{\zeta } g_{11} \psi +\zeta  g_6 \sqrt{n_0} \sigma _{W_2}^2\big)\big)+g_5 \big(\sqrt{\zeta } g_{10} \psi  \phi \nonumber\\
&\quad\quad-\zeta  \sigma _{W_2}^2 \big(\sqrt{\zeta } g_{11} \psi +\sqrt{n_0} \big(\zeta  g_6 \sigma _{W_2}^2-g_1 \phi \big)\big)\big)\\
0 &= \sqrt{\zeta } g_1 g_5 \psi  \phi +\sqrt{\zeta } g_2 g_6 \psi  \phi +\gamma  \zeta  g_3 g_8 \sqrt{n_0} \phi +\gamma  g_4 \sqrt{n_0} \phi ^2-\zeta ^2 g_3 g_8 \sqrt{n_0} \phi  \sigma _{W_2}^2+\sqrt{n_0} \phi  \eta ' \sigma _{W_2}^2 \big(\zeta  g_3 g_8+g_4 \phi \big)\nonumber\\
&\quad\quad-\zeta  g_4 \sqrt{n_0} \phi ^2 \sigma _{W_2}^2-\zeta ^{3/2} g_5 g_6 \psi  \sigma _{W_2}^2\\
0 &= -\sqrt{\zeta } g_4 g_8 \sqrt{n_0} \phi ^2 \big(\gamma +\sigma _{W_2}^2 \big(\eta '-\zeta \big)\big)+\zeta  \sigma _{W_2}^2 \big(\sqrt{\zeta } g_3 g_8 \sqrt{n_0} \phi  \big(\gamma +\sigma _{W_2}^2 \big(\eta '-\zeta \big)\big)-\zeta  g_5 g_6 \psi  \sigma _{W_2}^2+g_2 g_6 \psi  \phi \big)\nonumber\\
&\quad\quad+g_1 \psi  \phi  \big(\gamma  \phi +\sigma _{W_2}^2 \big(\phi  \big(\eta '-\zeta \big)+\zeta  g_5\big)\big)\\
\end{align}
\end{subequations}
}

After some straightforward algebra, one can eliminate all $g_i$ except for $g_3$ and $g_5$, which satisfy coupled polynomial equations. Those equations can be shown to be identical to eqn.~\eqref{eqn:tau} by invoking the change of variables,
\begin{equation}
    g_3 = - \frac{\sqrt{\zeta} \psi}{\sqrt{n_0} \phi} \tau_2\,,\quad\text{and}\quad
    g_5 = \big(\gamma +\sigma _{W_2}^2 \big(\eta '-\zeta \big)\big)\tau_1\,.
\end{equation}

In terms of the related variables defined in eqn.~\eqref{eq:ttau}, the error $E_{6}$ is given by,
\begin{align}
    E_{6} &= \sigma_{\varepsilon}^2 \phi \ttau_2^2/(1-\ttau_2^2 \phi)
\end{align}

\subsection{$H_{100}$}
\begin{align}
H_{100} &= \mathbb{E}\hat{y}(\bfx; P,X,\varepsilon)\hat{y}(\bfx; P,\tilde{X},\tilde{\varepsilon})\\
&= \mathbb{E}\Big[ N_0(\bfx; P)N_0(\bfx; P)^\top + K(\bfx,\tilde{X};P)K(\tilde{X},\tilde{X};P)^{-1} Y(\tilde{X},\tilde{\varepsilon})^\top Y(X,\varepsilon)K(X,X;P)^{-1}K(X,\bfx;P)\nonumber\\
&\quad + K(\bfx,\tilde{X};\P)K(\tilde{X},\tilde{X};P)^{-1} N_0(\tilde{X})^\top N_0(X)K(X,X;P)^{-1}K(X,\bfx;P)\nonumber\\
&\quad - N_0(\bfx; P) N_0(X; P) K(X,X;P)^{-1}K(X,\bfx;P) - N_0(\bfx; P) N_0(\tilde{X}; P) K(\tilde{X},\tilde{X};P)^{-1}K(\tilde{X},\bfx;P)\Big]\\
&= \nu \sigma_{W_2}^2 \eta + \nu E_{22} + \mathbb{E} \tr \big(K(\tilde{X},\tilde{X};P)^{-1} (\tilde{X}^\top X + \nu \frac{\sigma_{W_2}^2}{n_1}f(W_1 \tilde{X})^T F) K(X,X;P)^{-1}K(X,\bfx;P)K(\bfx,\tilde{X};P)\big)\\
&= \nu \sigma_{W_2}^2 \eta + \nu E_{22}+\mathbb{E} \tr \Big(K(\tilde{X},\tilde{X};P)^{-1} (\tilde{X}^\top X + \nu \frac{\sigma_{W_2}^2}{n_1}f(W_1 \tilde{X})^T F) K(X,X;P)^{-1}\\
&\qquad\qquad\; \times(\frac{\sigma_{W_2}^2 \zeta}{n_0} X^\top+\frac{\zeta}{\sqrt{n_0}n_1}F^\top W_1)(\frac{\sigma_{W_2}^2 \zeta}{n_0} \tilde{X}^\top+\frac{\sqrt{\zeta}}{\sqrt{n_0}n_1}f(W_1\tilde{X})^\top W_1)^\top\Big)\\
&\equiv \nu \sigma_{W_2}^2 \eta + \nu E_{22} + E_{71} + \nu E_{72}\,,
\end{align}
where, $E_{22}$ is given above and,
\begin{align}
E_{71} &= \mathbb{E} \tr \Big(K(\tilde{X},\tilde{X};P)^{-1} \tilde{X}^\top X K(X,X;P)^{-1}\big( \frac{\eta-\zeta}{ n_1^2}f(W_1X)^\top f(W_1\tilde{X})\\
&\qquad\qquad\; + (\frac{\sigma_{W_2}^2 \zeta}{n_0} X^\top+\frac{\zeta}{\sqrt{n_0}n_1}F^\top W_1)(\frac{\sigma_{W_2}^2 \zeta}{n_0} \tilde{X}^\top+\frac{\sqrt{\zeta}}{\sqrt{n_0}n_1}f(W_1\tilde{X})^\top W_1)^\top\big)\Big)\\
E_{72} &= \frac{\sigma_{W_2}^2}{n_1}\mathbb{E} \tr \Big(K(\tilde{X},\tilde{X};P)^{-1} f(W_1 \tilde{X})^T F K(X,X;P)^{-1}\big( \frac{\eta-\zeta}{ n_1^2}f(W_1X)^\top f(W_1\tilde{X})\\
&\qquad\qquad\; + (\frac{\sigma_{W_2}^2 \zeta}{n_0} X^\top+\frac{\zeta}{\sqrt{n_0}n_1}F^\top W_1)(\frac{\sigma_{W_2}^2 \zeta}{n_0} \tilde{X}^\top+\frac{\sqrt{\zeta}}{\sqrt{n_0}n_1}f(W_1\tilde{X})^\top W_1)^\top\big)\Big)\,.
\end{align}

\subsubsection{$E_{71}$}
A linear pencil for $E_{71}$ follows from the representation,
\eq{
E_{71} = \tr(U_{71}^T Q_{71}^{-1} V_{71})\,,
}
where,
\begin{align}
    U_{71}^T  &= \left(
\begin{array}{ccccccccccc}
  0 & \frac{I_{n_0}}{m} & 0 & 0 & 0 & 0 & 0 & 0 & 0 \\
\end{array}
\right)\\
V_{71}^T &= 
\left(
\begin{array}{ccccccccccc}
 0 & 0 & 0 & 0 & 0 & 0 & 0 & 0 & -\frac{\sqrt{n_0} n_1 I_{n_0}}{\sqrt{\zeta }} \\
\end{array}
\right)
\end{align}
and, for $\beta = \left(n_0 (\zeta -\eta )-\zeta  n_1 \sigma_{W_2}^2\right)$,
\eq{
Q_{71} = \left(\begin{smallmatrix}
I_m \left(\gamma +\sigma_{W_2}^2 \left(\eta '-\zeta \right)\right) & 0 & \frac{\zeta  X^\top \sigma_{W_2}^2}{n_0} & \frac{\sqrt{\eta -\zeta } \Theta_F^\top}{n_1} & \frac{\sqrt{\zeta } X^\top}{\sqrt{n_0} n_1} & -\frac{\zeta ^2 m X^\top \sigma_{W_2}^4}{n_0^2} & 0 & 0 & 0 \\
 -X & I_{n_0} & 0 & 0 & 0 & 0 & 0 & 0 & 0 \\
 -X & 0 & I_{n_0} & 0 & 0 & 0 & 0 & \frac{\sqrt{\zeta } m W_1^\top}{\sqrt{n_0} n_1} & 0 \\
 -\sqrt{\eta -\zeta } \Theta_F & 0 & -\frac{\sqrt{\zeta } W_1}{\sqrt{n_0}} & I_{n_1} & 0 & -\frac{\sqrt{\zeta } m W_1 \beta}{n_0^{3/2} n_1} & \frac{m (\eta -\zeta )^{3/2} \tilde{\Theta }_F}{n_1} & 0 & 0 \\
 0 & 0 & 0 & -W_1^\top & I_{n_0} & 0 & 0 & 0 & 0 \\
 0 & 0 & 0 & 0 & 0 & I_{n_0} & -\tilde{X} & 0 & 0 \\
 0 & 0 & 0 & 0 & 0 & \frac{\zeta  \sigma_{W_2}^2 \tilde{X}^\top}{n_0} & I_m \left(\gamma +\sigma_{W_2}^2 \left(\eta '-\zeta \right)\right) & \frac{\sqrt{\eta -\zeta } \tilde{\Theta }_F^\top}{n_1} & \frac{\sqrt{\zeta } \tilde{X}^\top}{\sqrt{n_0} n_1} \\
 0 & 0 & 0 & 0 & 0 & -\frac{\sqrt{\zeta } W_1}{\sqrt{n_0}} & -\sqrt{\eta -\zeta } \tilde{\Theta }_F & I_{n_1} & 0 \\
 0 & 0 & 0 & 0 & 0 & 0 & 0 & -W_1^\top & I_{n_0}
\end{smallmatrix}\right)\,.
}
The equations satisfied by the operator-valued Stieltjes transform $G$ of $\bar{Q}_{71}$ induce the following structure on $G$,
\eq{
G = \begin{pmatrix} 0 & G_{12} \\ G_{12}^\top & 0 \end{pmatrix}\,,
}
where,
\eq{
G_{12} =\left(
\begin{array}{ccccccccc}
 g_8 & 0 & 0 & 0 & 0 & 0 & 0 & 0 & 0 \\
 0 & g_1 & g_3 & 0 & g_5 & g_6 & 0 & 0 & g_2 \\
 0 & 0 & g_9 & 0 & g_5 & g_{12} & 0 & 0 & g_4 \\
 0 & 0 & 0 & g_{11} & 0 & 0 & 0 & g_{15} & 0 \\
 0 & 0 & g_{14} & 0 & g_{10} & g_{13} & 0 & 0 & g_7 \\
 0 & 0 & 0 & 0 & 0 & g_9 & 0 & 0 & g_5 \\
 0 & 0 & 0 & 0 & 0 & 0 & g_8 & 0 & 0 \\
 0 & 0 & 0 & 0 & 0 & 0 & 0 & g_{11} & 0 \\
 0 & 0 & 0 & 0 & 0 & g_{14} & 0 & 0 & g_{10} \\
\end{array}
\right)
\,,
}
and the independent entry-wise component functions $g_i$ give the error $E_{71}$ through the relation,
\eq{
E_{71} = -\frac{g_2 \sqrt{n_0} \phi }{\sqrt{\zeta } \psi }\,,
}
and themselves satisfy the following system of polynomial equations,
{\small
\begin{subequations}
\begin{align}
0 &= 1-g_1\\
0 &= \sqrt{\zeta } g_9 g_{11} \sqrt{n_0}-g_{14} \psi\\
0 &= \sqrt{\zeta } g_5 g_{11} \sqrt{n_0}-g_{10} \psi +\psi\\
0 &= -\sqrt{\zeta } g_6 g_8 \psi -g_2 \sqrt{n_0} \phi  \big(\gamma +\sigma _{W_2}^2 \big(\eta '-\zeta \big)\big)\\
0 &= -\sqrt{\zeta } g_8 g_9 \psi -g_5 \sqrt{n_0} \phi  \big(\gamma +\sigma _{W_2}^2 \big(\eta '-\zeta \big)\big)\\
0 &= -\sqrt{\zeta } g_8 g_{12} \psi -g_4 \sqrt{n_0} \phi  \big(\gamma +\sigma _{W_2}^2 \big(\eta '-\zeta \big)\big)\\
0 &= -\sqrt{\zeta } g_8 g_{13} \psi -g_7 \sqrt{n_0} \phi  \big(\gamma +\sigma _{W_2}^2 \big(\eta '-\zeta \big)\big)\\
0 &= -\sqrt{\zeta } g_8 g_{14} \psi -\big(g_{10}-1\big) \sqrt{n_0} \phi  \big(\gamma +\sigma _{W_2}^2 \big(\eta '-\zeta \big)\big)\\
0 &= -\sqrt{\zeta } g_1 g_8 \psi -\sqrt{\zeta } g_3 g_8 \psi -g_5 \sqrt{n_0} \phi  \big(\gamma +\sigma _{W_2}^2 \big(\eta '-\zeta \big)\big)\\
0 &= g_3 \sqrt{n_0} \phi  \big(\gamma +\sigma _{W_2}^2 \big(\eta '-\zeta \big)\big)+g_8 \big(\sqrt{\zeta } g_{14} \psi +\zeta  g_9 \sqrt{n_0} \sigma _{W_2}^2\big)\\
0 &= g_5 \sqrt{n_0} \phi  \big(\gamma +\sigma _{W_2}^2 \big(\eta '-\zeta \big)\big)+g_8 \big(\sqrt{\zeta } g_{10} \psi +\zeta  g_5 \sqrt{n_0} \sigma _{W_2}^2\big)\\
0 &= \sqrt{\zeta } \sqrt{n_0} \big(g_5 \big(g_{11} \big(\zeta  \psi -\eta  \psi -\zeta  \sigma _{W_2}^2\big)+g_{15} \phi \big)+g_4 g_{11} \phi \big)-g_7 \psi  \phi\\
0 &= \sqrt{\zeta } \sqrt{n_0} \big(g_9 \big(g_{11} \big(\zeta  \psi -\eta  \psi -\zeta  \sigma _{W_2}^2\big)+g_{15} \phi \big)+g_{11} g_{12} \phi \big)-g_{13} \psi  \phi\\
0 &= \big(g_9-1\big) \sqrt{n_0} \phi  \big(\gamma +\sigma _{W_2}^2 \big(\eta '-\zeta \big)\big)+g_8 \big(\sqrt{\zeta } g_{14} \psi +\zeta  g_9 \sqrt{n_0} \sigma _{W_2}^2\big)\\
0 &= \sqrt{\zeta } g_7 g_8 \psi  \phi +\sqrt{n_0} \big(g_2 \phi ^2 \big(\gamma +\sigma _{W_2}^2 \big(\eta '-\zeta \big)\big)+\zeta  g_8 \sigma _{W_2}^2 \big(g_4 \phi -\zeta  g_5 \sigma _{W_2}^2\big)\big)\\
0 &= \sqrt{\zeta } g_{10} g_{11} \sqrt{n_0} \phi  \big(\sigma _{W_2}^2 \big(\zeta -\eta '\big)-\gamma \big)+g_{14} \psi  \big(\gamma  \phi +\sigma _{W_2}^2 \big(-\zeta  \phi +\phi  \eta '+\zeta  g_8\big)\big)\\
0 &= g_{11} \big(g_8 \psi  (\zeta -\eta )+\phi  \big(\sqrt{\zeta } g_5 \sqrt{n_0}-1\big) \big(\gamma +\sigma _{W_2}^2 \big(\eta '-\zeta \big)\big)\big)+\phi  \big(\gamma +\sigma _{W_2}^2 \big(\eta '-\zeta \big)\big)\\
0 &= g_6 \sqrt{n_0} \phi ^2 \big(\gamma +\sigma _{W_2}^2 \big(\eta '-\zeta \big)\big)+g_8 \big(\sqrt{\zeta } g_{13} \psi  \phi +\zeta  \sqrt{n_0} \sigma _{W_2}^2 \big(g_{12} \phi -\zeta  g_9 \sigma _{W_2}^2\big)\big)\\
0 &= g_9 \psi  \big(\gamma  \phi +\sigma _{W_2}^2 \big(-\zeta  \phi +\phi  \eta '+\zeta  g_8\big)\big)-\phi  \big(\sqrt{\zeta } g_5 g_{11} \sqrt{n_0}+\psi \big) \big(\gamma +\sigma _{W_2}^2 \big(\eta '-\zeta \big)\big)\\
0 &= g_8 \big(\sqrt{\zeta } g_{14} \psi +\sqrt{n_0} \big(\gamma +g_{11} (\eta -\zeta )+\sigma _{W_2}^2 \big(\eta '+\zeta  \big(g_9-1\big)\big)\big)\big)-\sqrt{n_0} \big(\gamma +\sigma _{W_2}^2 \big(\eta '-\zeta \big)\big)\\
0 &= -\sqrt{\zeta } g_5 g_{11} \sqrt{n_0} \phi  \big(\gamma +\sigma _{W_2}^2 \big(\eta '-\zeta \big)\big)+g_3 \psi  \big(\gamma  \phi +\sigma _{W_2}^2 \big(-\zeta  \phi +\phi  \eta '+\zeta  g_8\big)\big)+\zeta  g_1 g_8 \psi  \sigma _{W_2}^2\\
0 &= \sqrt{n_0} \phi  \big(\zeta  g_9 g_{11}+g_{12} \phi \big) \big(\gamma +\sigma _{W_2}^2 \big(\eta '-\zeta \big)\big)+g_8 \big(\sqrt{\zeta } g_{13} \psi  \phi +\zeta  \sqrt{n_0} \sigma _{W_2}^2 \big(g_{12} \phi -\zeta  g_9 \sigma _{W_2}^2\big)\big)\\
0 &= \sqrt{\zeta } g_7 g_8 \psi  \phi +\sqrt{n_0} \big(g_4 \phi  \big(\gamma  \phi +\sigma _{W_2}^2 \big(-\zeta  \phi +\phi  \eta '+\zeta  g_8\big)\big)+\zeta  g_5 \big(g_{11} \phi  \big(\gamma +\sigma _{W_2}^2 \big(\eta '-\zeta \big)\big)-\zeta  g_8 \sigma _{W_2}^4\big)\big)\\
0 &= g_8 \psi  (-(\zeta -\eta )) \big(g_{11} \psi  (\zeta -\eta )+g_{15} \phi \big)-\phi  \big(\gamma +\sigma _{W_2}^2 \big(\eta '-\zeta \big)\big)\nonumber\\
&\quad\quad \big(-\zeta  g_9 g_{11} \psi +\sqrt{\zeta } \sqrt{n_0} \big(g_5 \big(g_{11} \big(\zeta  \psi -\eta  \psi -\zeta  \sigma _{W_2}^2\big)+g_{15} \phi \big)+g_4 g_{11} \phi \big)-g_{15} \phi \big)\\
0 &= \phi  \big(g_{12} \psi  \big(\gamma  \phi +\sigma _{W_2}^2 \big(-\zeta  \phi +\phi  \eta '+\zeta  g_8\big)\big)-\sqrt{\zeta } \sqrt{n_0} \big(\gamma +\sigma _{W_2}^2 \big(\eta '-\zeta \big)\big)\nonumber\\
&\quad\quad \big(g_5 \big(g_{11} \big(\zeta  \psi -\eta  \psi -\zeta  \sigma _{W_2}^2\big)+g_{15} \phi \big)+g_4 g_{11} \phi \big)\big)+\zeta  g_9 \psi  \big(g_{11} \phi  \big(\gamma +\sigma _{W_2}^2 \big(\eta '-\zeta \big)\big)-\zeta  g_8 \sigma _{W_2}^4\big)\\
0 &= \sqrt{\zeta } \big(-g_{10} g_{15} \sqrt{n_0} \phi ^2 \big(\gamma +\sigma _{W_2}^2 \big(\eta '-\zeta \big)\big)-g_{11} \phi  \big(\gamma +\sigma _{W_2}^2 \big(\eta '-\zeta \big)\big) \big(\sqrt{n_0} \big(g_{10} \big(\zeta  \psi -\eta  \psi -\zeta  \sigma _{W_2}^2\big)+g_7 \phi \big)\nonumber\\
&\quad\quad-\sqrt{\zeta } g_{14} \psi \big)-\zeta ^{3/2} g_8 g_{14} \psi  \sigma _{W_2}^4\big)+g_{13} \psi  \phi  \big(\gamma  \phi +\sigma _{W_2}^2 \big(\phi  \big(\eta '-\zeta \big)+\zeta  g_8\big)\big)\\
0 &= \sqrt{\zeta } \big(-\gamma  \zeta  g_5 g_{11} \sqrt{n_0} \psi  \phi +\gamma  \eta  g_5 g_{11} \sqrt{n_0} \psi  \phi -\gamma  g_5 g_{15} \sqrt{n_0} \phi ^2-g_2 g_{11} \sqrt{n_0} \phi ^2 \big(\gamma +\sigma _{W_2}^2 \big(\eta '-\zeta \big)\big)\nonumber\\
&\quad\quad+\gamma  \zeta  g_5 g_{11} \sqrt{n_0} \phi  \sigma _{W_2}^2+\zeta ^2 g_5 g_{11} \sqrt{n_0} \psi  \phi  \sigma _{W_2}^2-\zeta ^2 g_5 g_{11} \sqrt{n_0} \phi  \sigma _{W_2}^4\nonumber\\
&\quad\quad+g_5 \sqrt{n_0} \phi  \eta ' \sigma _{W_2}^2 \big(g_{11} \big(-\zeta  \psi +\eta  \psi +\zeta  \sigma _{W_2}^2\big)-g_{15} \phi \big)-\zeta  \eta  g_5 g_{11} \sqrt{n_0} \psi  \phi  \sigma _{W_2}^2+\zeta  g_5 g_{15} \sqrt{n_0} \phi ^2 \sigma _{W_2}^2\nonumber\\
&\quad\quad+\sqrt{\zeta } g_3 \psi  \big(g_{11} \phi  \big(\gamma +\sigma _{W_2}^2 \big(\eta '-\zeta \big)\big)-\zeta  g_8 \sigma _{W_2}^4\big)-\zeta ^{3/2} g_1 g_8 \psi  \sigma _{W_2}^4\big)\nonumber\\
&\quad\quad+g_6 \psi  \phi  \big(\gamma  \phi +\sigma _{W_2}^2 \big(\phi  \big(\eta '-\zeta \big)+\zeta  g_8\big)\big)
\end{align}
\end{subequations}
}

After some straightforward algebra, one can eliminate all $g_i$ except for $g_5$ and $g_8$, which satisfy coupled polynomial equations. Those equations can be shown to be identical to eqn.~\eqref{eqn:tau} by invoking the change of variables,
\begin{equation}
    g_5 = - \frac{\sqrt{\zeta} \psi}{\sqrt{n_0} \phi} \tau_2\,,\quad\text{and}\quad
    g_8 = \big(\gamma +\sigma _{W_2}^2 \big(\eta '-\zeta \big)\big)\tau_1\,.
\end{equation}

In terms of the related variables defined in eqn.~\eqref{eq:ttau}, the error $E_{71}$ is given by,
\begin{align}
    E_{71} &= \frac{\psi  \tilde{\tau }_1^2 \big(2 \zeta  \tilde{\tau }_2+\eta \big)+\zeta  \phi ^2 \tilde{\tau }_2^2}{\zeta  \big(\phi ^2-\psi  \tilde{\tau }_1^2\big)}\\
    &= \tau_2'/\tau_1'-2\tau_2/\tau_1+1\,.
\end{align}

\subsubsection{$E_{72}$}
A linear pencil for $E_{72}$ follows from the representation,
\eq{
E_{72} = \tr(U_{72}^T Q_{72}^{-1} V_{72})\,,
}
where,
\begin{align}
    U_{72}^T  &= \left(
\begin{array}{ccccccccccc}
 0 & \frac{I_{n_1}}{m} & 0 & 0 & 0 & 0 & 0 & 0 & 0 & 0 & 0 \\
\end{array}
\right)\\
V_{72}^T &= 
\left(
\begin{array}{ccccccccccc}
 0 & 0 & 0 & 0 & 0 & 0 & 0 & 0 & 0 & -n_1 I_{n_1} & 0 \\
\end{array}
\right)
\end{align}
and, for $\beta = \left(n_0 (\zeta -\eta )-\zeta  n_1 \sigma_{W_2}^2\right)$,
\eq{
Q_{72} = \left(\begin{smallmatrix}
I_m \left(\gamma +\sigma_{W_2}^2 \left(\eta '-\zeta \right)\right) & 0 & 0 & \frac{\zeta  X^\top \sigma_{W_2}^2}{n_0} & \frac{\sqrt{\eta -\zeta } \Theta_F^\top}{n_1} & \frac{\sqrt{\zeta } X^\top}{\sqrt{n_0} n_1} & -\frac{\zeta ^2 m X^\top \sigma_{W_2}^4}{n_0^2} & 0 & 0 & 0 & 0 \\
 -\sqrt{\eta -\zeta } \Theta_F & I_{n_1} & -\frac{\sqrt{\zeta } W_1}{\sqrt{n_0}} & 0 & 0 & 0 & 0 & 0 & 0 & 0 & 0 \\
 -X & 0 & I_{n_0} & 0 & 0 & 0 & 0 & 0 & 0 & 0 & 0 \\
 -X & 0 & 0 & I_{n_0} & 0 & 0 & 0 & 0 & \frac{\sqrt{\zeta } m W_1^\top}{\sqrt{n_0} n_1} & 0 & 0 \\
 -\sqrt{\eta -\zeta } \Theta_F & 0 & 0 & -\frac{\sqrt{\zeta } W_1}{\sqrt{n_0}} & I_{n_1} & 0 & -\frac{\sqrt{\zeta } m W_1 \beta}{n_0^{3/2} n_1} & \frac{m (\eta -\zeta )^{3/2} \tilde{\Theta }_F}{n_1} & 0 & 0 & 0 \\
 0 & 0 & 0 & 0 & -W_1^\top & I_{n_0} & 0 & 0 & 0 & 0 & 0 \\
 0 & 0 & 0 & 0 & 0 & 0 & I_{n_0} & -\tilde{X} & 0 & 0 & 0 \\
 0 & 0 & 0 & 0 & 0 & 0 & \frac{\zeta  \sigma_{W_2}^2 \tilde{X}^\top}{n_0} & I_m \left(\gamma +\sigma_{W_2}^2 \left(\eta '-\zeta \right)\right) & 0 & \frac{\sqrt{\eta -\zeta } \tilde{\Theta }_F^\top}{n_1} & \frac{\sqrt{\zeta } \tilde{X}^\top}{\sqrt{n_0} n_1} \\
 0 & 0 & 0 & 0 & 0 & 0 & -\frac{\sqrt{\zeta } W_1}{\sqrt{n_0}} & -\sqrt{\eta -\zeta } \tilde{\Theta }_F & I_{n_1} & 0 & 0 \\
 0 & 0 & 0 & 0 & 0 & 0 & -\frac{\sqrt{\zeta } W_1}{\sqrt{n_0}} & -\sqrt{\eta -\zeta } \tilde{\Theta }_F & 0 & I_{n_1} & 0 \\
 0 & 0 & 0 & 0 & 0 & 0 & 0 & 0 & 0 & -W_1^\top & I_{n_0}
\end{smallmatrix}\right)\,.
}
The equations satisfied by the operator-valued Stieltjes transform $G$ of $\bar{Q}_{72}$ induce the following structure on $G$,
\eq{
G = \begin{pmatrix} 0 & G_{12} \\ G_{12}^\top & 0 \end{pmatrix}\,,
}
where,
\eq{
G_{12} = \left(
\begin{array}{ccccccccccc}
 g_{12} & 0 & 0 & 0 & 0 & 0 & 0 & 0 & 0 & 0 & 0 \\
 0 & g_1 & 0 & 0 & g_6 & 0 & 0 & 0 & g_{11} & g_3 & 0 \\
 0 & 0 & g_1 & g_4 & 0 & g_7 & g_8 & 0 & 0 & 0 & g_2 \\
 0 & 0 & 0 & g_{13} & 0 & g_7 & g_{16} & 0 & 0 & 0 & g_5 \\
 0 & 0 & 0 & 0 & g_{15} & 0 & 0 & 0 & g_{19} & g_{10} & 0 \\
 0 & 0 & 0 & g_{18} & 0 & g_{14} & g_{17} & 0 & 0 & 0 & g_9 \\
 0 & 0 & 0 & 0 & 0 & 0 & g_{13} & 0 & 0 & 0 & g_7 \\
 0 & 0 & 0 & 0 & 0 & 0 & 0 & g_{12} & 0 & 0 & 0 \\
 0 & 0 & 0 & 0 & 0 & 0 & 0 & 0 & g_1 & g_6 & 0 \\
 0 & 0 & 0 & 0 & 0 & 0 & 0 & 0 & 0 & g_{15} & 0 \\
 0 & 0 & 0 & 0 & 0 & 0 & g_{18} & 0 & 0 & 0 & g_{14} \\
\end{array}
\right)
\,,
}
and the independent entry-wise component functions $g_i$ give the error $E_{72}$ through the relation,
\eq{
E_{72} = -\frac{\phi g_3}{\psi} \,,
}
and themselves satisfy the following system of polynomial equations,
{\small
\begin{subequations}
\begin{align}
0 &= 1-g_1\\
0 &= -\zeta  g_{13} g_{15} \psi -g_{19} \phi\\
0 &= \sqrt{\zeta } g_{13} g_{15} \sqrt{n_0}-g_{18} \psi\\
0 &= \sqrt{\zeta } g_7 g_{15} \sqrt{n_0}-g_{14} \psi +\psi\\
0 &= g_{11} (-\phi )-\zeta  \big(g_1 g_4+g_6 g_{13}\big) \psi\\
0 &= -\sqrt{\zeta } g_8 g_{12} \psi -g_2 \sqrt{n_0} \phi  \big(\gamma +\sigma _{W_2}^2 \big(\eta '-\zeta \big)\big)\\
0 &= -\sqrt{\zeta } g_{12} g_{13} \psi -g_7 \sqrt{n_0} \phi  \big(\gamma +\sigma _{W_2}^2 \big(\eta '-\zeta \big)\big)\\
0 &= -\sqrt{\zeta } g_{12} g_{16} \psi -g_5 \sqrt{n_0} \phi  \big(\gamma +\sigma _{W_2}^2 \big(\eta '-\zeta \big)\big)\\
0 &= -\sqrt{\zeta } g_{12} g_{17} \psi -g_9 \sqrt{n_0} \phi  \big(\gamma +\sigma _{W_2}^2 \big(\eta '-\zeta \big)\big)\\
0 &= -\sqrt{\zeta } g_{12} g_{18} \psi -\big(g_{14}-1\big) \sqrt{n_0} \phi  \big(\gamma +\sigma _{W_2}^2 \big(\eta '-\zeta \big)\big)\\
0 &= -\sqrt{\zeta } g_1 g_{12} \psi -\sqrt{\zeta } g_4 g_{12} \psi -g_7 \sqrt{n_0} \phi  \big(\gamma +\sigma _{W_2}^2 \big(\eta '-\zeta \big)\big)\\
0 &= g_{12} g_{15} \psi  (\zeta -\eta )-\phi  \big(g_6-\sqrt{\zeta } g_7 g_{15} \sqrt{n_0}\big) \big(\gamma +\sigma _{W_2}^2 \big(\eta '-\zeta \big)\big)\\
0 &= g_7 \sqrt{n_0} \phi  \big(\gamma +\sigma _{W_2}^2 \big(\eta '-\zeta \big)\big)+g_{12} \big(\sqrt{\zeta } g_{14} \psi +\zeta  g_7 \sqrt{n_0} \sigma _{W_2}^2\big)\\
0 &= g_4 \sqrt{n_0} \phi  \big(\gamma +\sigma _{W_2}^2 \big(\eta '-\zeta \big)\big)+g_{12} \big(\sqrt{\zeta } g_{18} \psi +\zeta  g_{13} \sqrt{n_0} \sigma _{W_2}^2\big)\\
0 &= \sqrt{\zeta } \sqrt{n_0} \big(g_7 \big(g_{15} \big(\zeta  \psi -\eta  \psi -\zeta  \sigma _{W_2}^2\big)+\big(g_{10}+g_{19}\big) \phi \big)+g_5 g_{15} \phi \big)-g_9 \psi  \phi\\
0 &= \big(g_{13}-1\big) \sqrt{n_0} \phi  \big(\gamma +\sigma _{W_2}^2 \big(\eta '-\zeta \big)\big)+g_{12} \big(\sqrt{\zeta } g_{18} \psi +\zeta  g_{13} \sqrt{n_0} \sigma _{W_2}^2\big)\\
0 &= \sqrt{\zeta } \sqrt{n_0} \big(g_{13} \big(g_{15} \big(\zeta  \psi -\eta  \psi -\zeta  \sigma _{W_2}^2\big)+g_{19} \phi \big)+g_{10} g_{13} \phi +g_{15} g_{16} \phi \big)-g_{17} \psi  \phi\\
0 &= \sqrt{\zeta } g_9 g_{12} \psi  \phi +\sqrt{n_0} \big(g_2 \phi ^2 \big(\gamma +\sigma _{W_2}^2 \big(\eta '-\zeta \big)\big)+\zeta  g_{12} \sigma _{W_2}^2 \big(g_5 \phi -\zeta  g_7 \sigma _{W_2}^2\big)\big)\\
0 &= \sqrt{\zeta } g_{14} g_{15} \sqrt{n_0} \phi  \big(\sigma _{W_2}^2 \big(\zeta -\eta '\big)-\gamma \big)+g_{18} \psi  \big(\gamma  \phi +\sigma _{W_2}^2 \big(-\zeta  \phi +\phi  \eta '+\zeta  g_{12}\big)\big)\\
0 &= g_{15} \big(g_{12} \psi  (\zeta -\eta )+\phi  \big(\sqrt{\zeta } g_7 \sqrt{n_0}-1\big) \big(\gamma +\sigma _{W_2}^2 \big(\eta '-\zeta \big)\big)\big)+\phi  \big(\gamma +\sigma _{W_2}^2 \big(\eta '-\zeta \big)\big)\\
0 &= g_8 \sqrt{n_0} \phi ^2 \big(\gamma +\sigma _{W_2}^2 \big(\eta '-\zeta \big)\big)+g_{12} \big(\sqrt{\zeta } g_{17} \psi  \phi +\zeta  \sqrt{n_0} \sigma _{W_2}^2 \big(g_{16} \phi -\zeta  g_{13} \sigma _{W_2}^2\big)\big)\\
0 &= g_{13} \psi  \big(\gamma  \phi +\sigma _{W_2}^2 \big(-\zeta  \phi +\phi  \eta '+\zeta  g_{12}\big)\big)-\phi  \big(\sqrt{\zeta } g_7 g_{15} \sqrt{n_0}+\psi \big) \big(\gamma +\sigma _{W_2}^2 \big(\eta '-\zeta \big)\big)\\
0 &= g_{12} \big(\sqrt{\zeta } g_{18} \psi +\sqrt{n_0} \big(\gamma +g_{15} (\eta -\zeta )+\sigma _{W_2}^2 \big(\eta '+\zeta  \big(g_{13}-1\big)\big)\big)\big)-\sqrt{n_0} \big(\gamma +\sigma _{W_2}^2 \big(\eta '-\zeta \big)\big)\\
0 &= g_{12} g_{19} \psi  (\zeta -\eta )-\phi  \big(g_{11}-\sqrt{\zeta } g_7 g_{19} \sqrt{n_0}\big) \big(\gamma +\sigma _{W_2}^2 \big(\eta '-\zeta \big)\big)+\zeta  g_1 g_4 \psi  \big(\sigma _{W_2}^2 \big(\zeta -\eta '\big)-\gamma \big)\\
0 &= g_{19} \big(g_{12} \psi  (\zeta -\eta )+\phi  \big(\sqrt{\zeta } g_7 \sqrt{n_0}-1\big) \big(\gamma +\sigma _{W_2}^2 \big(\eta '-\zeta \big)\big)\big)-\zeta  g_1 g_{13} \psi  \big(\gamma +\sigma _{W_2}^2 \big(\eta '-\zeta \big)\big)\\
0 &= -\sqrt{\zeta } g_7 g_{15} \sqrt{n_0} \phi  \big(\gamma +\sigma _{W_2}^2 \big(\eta '-\zeta \big)\big)+g_4 \psi  \big(\gamma  \phi +\sigma _{W_2}^2 \big(-\zeta  \phi +\phi  \eta '+\zeta  g_{12}\big)\big)+\zeta  g_1 g_{12} \psi  \sigma _{W_2}^2\\
0 &= \sqrt{n_0} \phi  \big(\zeta  \big(g_1+g_6\big) g_{13}+g_{16} \phi \big) \big(\gamma +\sigma _{W_2}^2 \big(\eta '-\zeta \big)\big)+g_{12} \big(\sqrt{\zeta } g_{17} \psi  \phi +\zeta  \sqrt{n_0} \sigma _{W_2}^2 \big(g_{16} \phi -\zeta  g_{13} \sigma _{W_2}^2\big)\big)\\
0 &= g_6 \big(g_{12} \psi  (\zeta -\eta )+\phi  \big(\sqrt{\zeta } g_7 \sqrt{n_0}-1\big) \big(\gamma +\sigma _{W_2}^2 \big(\eta '-\zeta \big)\big)\big)+g_1 \big(g_{12} \psi  (\zeta -\eta )\nonumber\\
&\quad\quad+\sqrt{\zeta } g_7 \sqrt{n_0} \phi  \big(\gamma +\sigma _{W_2}^2 \big(\eta '-\zeta \big)\big)\big)\\
0 &= \sqrt{\zeta } g_9 g_{12} \psi  \phi +\sqrt{n_0} \big(g_5 \phi  \big(\gamma  \phi +\sigma _{W_2}^2 \big(\phi  \big(\eta '-\zeta \big)+\zeta  g_{12}\big)\big)+\zeta  g_7 \big(g_1 \phi  \big(\gamma +\sigma _{W_2}^2 \big(\eta '-\zeta \big)\big)\nonumber\\
&\quad\quad+g_6 \phi  \big(\gamma +\sigma _{W_2}^2 \big(\eta '-\zeta \big)\big)-\zeta  g_{12} \sigma _{W_2}^4\big)\big)\\
0 &= g_{12} \psi  (-(\zeta -\eta )) \big(g_{15} \psi  (\zeta -\eta )+g_{19} \phi \big)-\sqrt{\zeta } \sqrt{n_0} \phi  \big(\gamma +\sigma _{W_2}^2 \big(\eta '-\zeta \big)\big)\nonumber\\
&\quad\quad \big(g_7 \big(g_{15} \big(\zeta  \psi -\eta  \psi -\zeta  \sigma _{W_2}^2\big)+g_{19} \phi \big)+g_5 g_{15} \phi \big)+g_{10} \phi  \big(g_{12} \psi  (\eta -\zeta )\nonumber\\
&\quad\quad-\phi  \big(\sqrt{\zeta } g_7 \sqrt{n_0}-1\big) \big(\gamma +\sigma _{W_2}^2 \big(\eta '-\zeta \big)\big)\big)\\
0 &= g_{15} \big(g_{12} \psi ^2 \big(-(\zeta -\eta )^2\big)-\sqrt{\zeta } \sqrt{n_0} \phi  \big(\gamma +\sigma _{W_2}^2 \big(\eta '-\zeta \big)\big) \big(g_7 \big(\zeta  \psi -\eta  \psi -\zeta  \sigma _{W_2}^2\big)+g_5 \phi \big)\big)\nonumber\\
&\quad\quad+g_{10} \phi  \big(g_{12} \psi  (\eta -\zeta )-\phi  \big(\sqrt{\zeta } g_7 \sqrt{n_0}-1\big) \big(\gamma +\sigma _{W_2}^2 \big(\eta '-\zeta \big)\big)\big)+\zeta  g_6 g_{13} \psi  \phi  \big(\gamma +\sigma _{W_2}^2 \big(\eta '-\zeta \big)\big)\\
0 &= \zeta  g_{10} g_{12} \psi -\eta  g_{10} g_{12} \psi +\gamma  \sqrt{\zeta } g_7 g_{10} \sqrt{n_0} \phi +\gamma  \sqrt{\zeta } g_2 g_{15} \sqrt{n_0} \phi -\zeta ^{3/2} g_7 g_{10} \sqrt{n_0} \phi  \sigma _{W_2}^2\nonumber\\
&\quad\quad-\zeta ^{3/2} g_2 g_{15} \sqrt{n_0} \phi  \sigma _{W_2}^2+\sqrt{\zeta } \big(g_7 g_{10}+g_2 g_{15}\big) \sqrt{n_0} \phi  \eta ' \sigma _{W_2}^2+\zeta  g_4 g_6 \psi  \big(\sigma _{W_2}^2 \big(\zeta -\eta '\big)-\gamma \big)\nonumber\\
&\quad\quad-g_3 \phi  \big(\gamma +\sigma _{W_2}^2 \big(\eta '-\zeta \big)\big)\\
0 &= \phi  \big(g_8 \psi  \big(\gamma  \phi +\sigma _{W_2}^2 \big(-\zeta  \phi +\phi  \eta '+\zeta  g_{12}\big)\big)-\sqrt{\zeta } \sqrt{n_0} \big(\gamma +\sigma _{W_2}^2 \big(\eta '-\zeta \big)\big) \big(g_7 \big(g_{15} \big(\zeta  \psi -\eta  \psi -\zeta  \sigma _{W_2}^2\big)\nonumber\\
&\quad\quad+\big(g_{10}+g_{19}\big) \phi \big)+g_2 g_{15} \phi \big)\big)+\zeta  g_1 \psi  \big(g_4 \phi  \big(\gamma +\sigma _{W_2}^2 \big(\eta '-\zeta \big)\big)-\zeta  g_{12} \sigma _{W_2}^4\big)\nonumber\\
&\quad\quad+\zeta  g_4 \psi  \big(g_6 \phi  \big(\gamma +\sigma _{W_2}^2 \big(\eta '-\zeta \big)\big)-\zeta  g_{12} \sigma _{W_2}^4\big)\\
0 &= g_6 \big(g_{12} \psi ^2 \big(-(\zeta -\eta )^2\big)-\sqrt{\zeta } \sqrt{n_0} \phi  \big(\gamma +\sigma _{W_2}^2 \big(\eta '-\zeta \big)\big) \big(g_7 \big(\zeta  \psi -\eta  \psi -\zeta  \sigma _{W_2}^2\big)+g_5 \phi \big)\big)\nonumber\\
&\quad\quad+\phi  \big(g_{11} \big(g_{12} \psi  (\eta -\zeta )+\sqrt{\zeta } g_7 \sqrt{n_0} \phi  \big(\sigma _{W_2}^2 \big(\zeta -\eta '\big)-\gamma \big)\big)+\sqrt{\zeta } g_1 g_2 \sqrt{n_0} \phi  \big(\sigma _{W_2}^2 \big(\zeta -\eta '\big)-\gamma \big)\big)\nonumber\\
&\quad\quad+g_3 \phi  \big(g_{12} \psi  (\eta -\zeta )-\phi  \big(\sqrt{\zeta } g_7 \sqrt{n_0}-1\big) \big(\gamma +\sigma _{W_2}^2 \big(\eta '-\zeta \big)\big)\big)\\
0 &= \sqrt{\zeta } \big(-\gamma  \zeta  g_{14} g_{15} \sqrt{n_0} \psi  \phi +\gamma  \eta  g_{14} g_{15} \sqrt{n_0} \psi  \phi -\gamma  g_{10} g_{14} \sqrt{n_0} \phi ^2-\gamma  g_9 g_{15} \sqrt{n_0} \phi ^2-\gamma  g_{14} g_{19} \sqrt{n_0} \phi ^2\nonumber\\
&\quad\quad+\gamma  \zeta  g_{14} g_{15} \sqrt{n_0} \phi  \sigma _{W_2}^2+\zeta ^2 g_{14} g_{15} \sqrt{n_0} \psi  \phi  \sigma _{W_2}^2-\zeta ^2 g_{14} g_{15} \sqrt{n_0} \phi  \sigma _{W_2}^4\nonumber\\
&\quad\quad-\sqrt{n_0} \phi  \eta ' \sigma _{W_2}^2 \big(g_{14} \big(g_{15} \big(\zeta  \psi -\eta  \psi -\zeta  \sigma _{W_2}^2\big)+g_{19} \phi \big)+g_{10} g_{14} \phi +g_9 g_{15} \phi \big)-\zeta  \eta  g_{14} g_{15} \sqrt{n_0} \psi  \phi  \sigma _{W_2}^2\nonumber\\
&\quad\quad+\zeta  g_{10} g_{14} \sqrt{n_0} \phi ^2 \sigma _{W_2}^2+\zeta  g_9 g_{15} \sqrt{n_0} \phi ^2 \sigma _{W_2}^2+\zeta  g_{14} g_{19} \sqrt{n_0} \phi ^2 \sigma _{W_2}^2+\sqrt{\zeta } g_1 g_{18} \psi  \phi  \big(\gamma +\sigma _{W_2}^2 \big(\eta '-\zeta \big)\big)\nonumber\\
&\quad\quad+\sqrt{\zeta } g_6 g_{18} \psi  \phi  \big(\gamma +\sigma _{W_2}^2 \big(\eta '-\zeta \big)\big)-\zeta ^{3/2} g_{12} g_{18} \psi  \sigma _{W_2}^4\big)+g_{17} \psi  \phi  \big(\gamma  \phi +\sigma _{W_2}^2 \big(\phi  \big(\eta '-\zeta \big)+\zeta  g_{12}\big)\big)\\
0 &= \gamma  g_{16} \psi  \phi ^2-\gamma  \zeta ^{3/2} g_7 g_{15} \sqrt{n_0} \psi  \phi +\gamma  \sqrt{\zeta } \eta  g_7 g_{15} \sqrt{n_0} \psi  \phi -\gamma  \sqrt{\zeta } g_7 g_{10} \sqrt{n_0} \phi ^2-\gamma  \sqrt{\zeta } g_5 g_{15} \sqrt{n_0} \phi ^2\nonumber\\
&\quad\quad-\gamma  \sqrt{\zeta } g_7 g_{19} \sqrt{n_0} \phi ^2+\gamma  \zeta ^{3/2} g_7 g_{15} \sqrt{n_0} \phi  \sigma _{W_2}^2-\zeta ^{3/2} \eta  g_7 g_{15} \sqrt{n_0} \psi  \phi  \sigma _{W_2}^2+\zeta ^{3/2} g_7 g_{10} \sqrt{n_0} \phi ^2 \sigma _{W_2}^2\nonumber\\
&\quad\quad+\zeta ^{3/2} g_5 g_{15} \sqrt{n_0} \phi ^2 \sigma _{W_2}^2+\zeta ^{3/2} g_7 g_{19} \sqrt{n_0} \phi ^2 \sigma _{W_2}^2+\zeta ^{5/2} g_7 g_{15} \sqrt{n_0} \psi  \phi  \sigma _{W_2}^2-\zeta ^{5/2} g_7 g_{15} \sqrt{n_0} \phi  \sigma _{W_2}^4\nonumber\\
&\quad\quad+\phi  \eta ' \sigma _{W_2}^2 \big(g_{16} \psi  \phi -\sqrt{\zeta } \sqrt{n_0} \big(g_7 \big(g_{15} \big(\zeta  \psi -\eta  \psi -\zeta  \sigma _{W_2}^2\big)+\big(g_{10}+g_{19}\big) \phi \big)+g_5 g_{15} \phi \big)\big)\nonumber\\
&\quad\quad+\zeta  g_1 g_{13} \psi  \phi  \big(\gamma +\sigma _{W_2}^2 \big(\eta '-\zeta \big)\big)+\zeta  g_6 g_{13} \psi  \phi  \big(\gamma +\sigma _{W_2}^2 \big(\eta '-\zeta \big)\big)-\zeta ^2 g_{12} g_{13} \psi  \sigma _{W_2}^4-\zeta  g_{16} \psi  \phi ^2 \sigma _{W_2}^2\nonumber\\
&\quad\quad+\zeta  g_{12} g_{16} \psi  \phi  \sigma _{W_2}^2
\end{align}
\end{subequations}
}

After some straightforward algebra, one can eliminate all $g_i$ except for $g_7$ and $g_{12}$, which satisfy coupled polynomial equations. Those equations can be shown to be identical to eqn.~\eqref{eqn:tau} by invoking the change of variables,
\begin{equation}
    g_7 = - \frac{\sqrt{\zeta} \psi}{\sqrt{n_0} \phi} \tau_2\,,\quad\text{and}\quad
    g_{12} = \big(\gamma +\sigma _{W_2}^2 \big(\eta '-\zeta \big)\big)\tau_1\,.
\end{equation}

In terms of the related variables defined in eqn.~\eqref{eq:ttau}, the error $E_{72}$ is given by,
\begin{align}
        E_{72} &= -\frac{\psi  \tilde{\tau }_1^2 \big(2 \zeta  \psi  \tilde{\tau }_1 (\zeta -2 \eta )+\phi  \big(\zeta ^2 \psi -\zeta  \eta  (\psi +1)-\eta ^2 \psi \big)\big)}{2 \zeta  \phi  \big(\phi ^2-\psi  \tilde{\tau }_1^2\big)}+\frac{\psi ^2 \tilde{\tau }_1^2 (\zeta -\eta )^3}{2 \zeta ^2
        \big(\tilde{\tau }_2+1\big){}^2 \big(\psi  \tilde{\tau }_1^2-\phi ^2\big)}\nonumber\\
        &\quad+\frac{\psi  \tilde{\tau }_1 \tilde{\tau }_2 \big(\psi  \tilde{\tau }_1+\phi \big) \big(\zeta  \tilde{\tau }_1+\eta  \phi \big)}{\phi ^3-\psi  \phi  \tilde{\tau }_1^2}+\frac{\psi ^2 \tilde{\tau }_1^2 (\zeta -\eta )^2 \big(\tilde{\tau }_1+\phi \big)}{\zeta  \phi  \big(\tilde{\tau }_2+1\big) \big(\phi ^2-\psi  \tilde{\tau }_1^2\big)}+\frac{\zeta  \tilde{\tau }_2^2 \big(\psi  \tilde{\tau }_1+\phi \big){}^2}{2 \big(\phi ^2-\psi  \tilde{\tau }_1^2\big)}\\
        &= -T_2/\tau_1' - E_{22} - \eta \sigma_{W_2}^2\,,
\end{align}
where $T_2$ is given in eqn.~(\ref{eq:T2}).

\subsection{$H_{101}$}
\begin{align}
H_{101} &= \mathbb{E}\hat{y}(\bfx; P,X,\varepsilon)\hat{y}(\bfx; P,\tilde{X},\varepsilon)\\
&= \mathbb{E}\Big[ N_0(\bfx; P)N_0(\bfx; P)^\top + K(\bfx,\tilde{X};\P)K(\tilde{X},\tilde{X};P)^{-1} Y(\tilde{X},\varepsilon)^\top Y(X,\varepsilon)K(X,X;P)^{-1}K(X,\bfx;P)\nonumber\\
&\quad + K(\bfx,\tilde{X};\P)K(\tilde{X},\tilde{X};P)^{-1} N_0(\tilde{X})^\top N_0(X)K(X,X;P)^{-1}K(X,\bfx;P)\nonumber\\
&\quad - N_0(\bfx; P) N_0(X; P) K(X,X;P)^{-1}K(X,\bfx;P) - N_0(\bfx; P) N_0(\tilde{X}; P) K(\tilde{X},\tilde{X};P)^{-1}K(\tilde{X},\bfx;P)\Big]\\
&= \nu \sigma_{W_2}^2 \eta + \nu E_{22} + \mathbb{E} \tr \big(K(\tilde{X},\tilde{X};P)^{-1} (\tilde{X}^\top X + \nu \frac{\sigma_{W_2}^2}{n_1}f(W_1 \tilde{X})^T F) K(X,X;P)^{-1}K(X,\bfx;P)K(\bfx,\tilde{X};P)\big)\\
&=H_{100}\,,
\end{align}

\subsection{$H_{110}$}
\begin{align}
H_{110} &= \mathbb{E}\hat{y}(\bfx; P,X,\varepsilon)\hat{y}(\bfx; P,X,\tilde{\varepsilon})\\
&= \mathbb{E}\Big[ N_0(\bfx; P)N_0(\bfx; P)^\top + K(\bfx,;\P)K(X,X;P)^{-1} Y(X,\tilde{\varepsilon})^\top Y(X,\varepsilon)K(X,X;P)^{-1}K(X,\bfx;P)\nonumber\\
&\quad + K(\bfx,X;\P)K(X,X;P)^{-1} N_0(X)^\top N_0(X)K(X,X;P)^{-1}K(X,\bfx;P)\nonumber\\
&\quad -2 N_0(\bfx; P) N_0(X; P) K(X,X;P)^{-1}K(X,\bfx;P)\Big]\\
&= \mathbb{E}\Big[ N_0(\bfx; P)N_0(\bfx; P)^\top + K(\bfx,;\P)K(X,X;P)^{-1} X^\top X K(X,X;P)^{-1}K(X,\bfx;P)\nonumber\\
&\quad + K(\bfx,X;\P)K(X,X;P)^{-1} N_0(X)^\top N_0(X)K(X,X;P)^{-1}K(X,\bfx;P)\nonumber\\
&\quad -2 N_0(\bfx; P) N_0(X; P) K(X,X;P)^{-1}K(X,\bfx;P)\Big]\\
&= \nu \sigma_{W_2}^2 \eta + \nu E_{22} + E_{32} + \nu E_{33}
\end{align}

\subsection{$H_{111}$}
\begin{align}
H_{111} &= \mathbb{E}\hat{y}(\bfx; P,X,\varepsilon)\hat{y}(\bfx; P,X,\varepsilon)\\
&= \mathbb{E}\Big[ N_0(\bfx; P)N_0(\bfx; P)^\top + K(\bfx,;\P)K(X,X;P)^{-1} Y(X,\varepsilon)^\top Y(X,\varepsilon)K(X,X;P)^{-1}K(X,\bfx;P)\nonumber\\
&\quad + K(\bfx,X;\P)K(X,X;P)^{-1} N_0(X)^\top N_0(X)K(X,X;P)^{-1}K(X,\bfx;P)\nonumber\\
&\quad -2 N_0(\bfx; P) N_0(X; P) K(X,X;P)^{-1}K(X,\bfx;P)\Big]\\
&= \mathbb{E}\Big[ N_0(\bfx; P)N_0(\bfx; P)^\top + K(\bfx,;\P)K(X,X;P)^{-1} (X^\top X + \sigma_{\varepsilon}^2 n_1 I_m) K(X,X;P)^{-1}K(X,\bfx;P)\nonumber\\
&\quad + K(\bfx,X;\P)K(X,X;P)^{-1} N_0(X)^\top N_0(X)K(X,X;P)^{-1}K(X,\bfx;P)\nonumber\\
&\quad -2 N_0(\bfx; P) N_0(X; P) K(X,X;P)^{-1}K(X,\bfx;P)\Big]\\
&= \nu \sigma_{W_2}^2 \eta + \nu E_{22} +E_{31} + E_{32} + \nu E_{33}
\end{align}
\subsection{Combining results: asymptotic variance terms}
Summarizing the above result, we have,
\begin{align*}
    H_{000} &= E_4 \\
    H_{001} &= E_4 \\
    H_{010} &= E_5 \\
    H_{011} &= E_5 + E_6 \\
    H_{100} &= \nu \sigma_{W_2}^2 \eta + \nu E_{22} + E_{71} + \nu E_{72} \\
    H_{101} &= \nu \sigma_{W_2}^2 \eta + \nu E_{22} + E_{71} + \nu E_{72} \\
    H_{110} &= \nu \sigma_{W_2}^2 \eta + \nu E_{22} + E_{32} + \nu E_{33} \\
    H_{111} &= \nu \sigma_{W_2}^2 \eta + \nu E_{22} + E_{31} + E_{32} + \nu E_{33}\,,
\end{align*}
which using eqn.~\eqref{eq:VH} gives,
\begin{align*}
    B &= 1 + E_{21} + E_4\\
      &= \tau_2^2/\tau_1^2\\
    V_{P} &= H_{100} - H_{000}\\
          &= \nu \sigma_{W_2}^2 \eta + \nu E_{22} + E_{71} + \nu E_{72}  - E_4\\
          &= E_{71} - E_4 - \nu T_2/\tau_1'\\
          &= \tau_2'/\tau_1'+2\tau_2/\tau_1-1 - (\tau_2/\tau_1-1)^2 - \nu T_2/\tau_1'\\
           &=\tau_2'/\tau_1' - B - \nu T_2/\tau_1'\\
    V_{X} &= H_{010} - H_{000}\\
          &= E_5 - E_4\\
          &= \phi  \tilde{\tau }_2^2 \big(\tilde{\tau }_2+1\big){}^2/(1-\phi  \tilde{\tau }_2^2) \\
          &= \phi B (\tau_1-\tau_2)^2 /(\tau_1^2-\phi(\tau_1-\tau_2)^2)\\
    V_{\bfe} &= H_{001} - H_{000}\\
             &= 0 \\
    V_{PX} &= H_{110} - H_{010} - H_{100} + H_{000}\\
           &= \nu \sigma_{W_2}^2 \eta + \nu E_{22} + E_{32} + \nu E_{33} - E_5 - (\nu \sigma_{W_2}^2 \eta + \nu E_{22} + E_{71} + \nu E_{72}) + E_4\\
           &= E_{32} - E_{71}-E_5+E_4 + \nu(E_{33}-E_{72})\\
           &= 1-2\tau_2/\tau_1-\tau_2'/\tau_1^2 - (\tau_2'/\tau_1'-2\tau_2/\tau_1+1) - V_X\\
           &\quad + \nu\left[\sigma_{W_2}^2 \left[\left(\tau_1 + (\sigma_{W_2}^2(\eta'-\zeta) + \gamma)\tau_1'+\sigma_{W_2}^2\zeta \tau_2'\right)/\tau_1^2 -\eta\right] - E_{22} - (T_2 - E_{22} - \eta \sigma_{W_2}^2))\right]\\
           &=-\tau_2'/\tau_1^2 -\tau_2'/\tau_1' - V_X  + \nu\left[\sigma_{W_2}^2 \left[\left(\tau_1 + (\sigma_{W_2}^2(\eta'-\zeta) + \gamma)\tau_1'+\sigma_{W_2}^2\zeta \tau_2'\right)/\tau_1^2\right] - T_2)\right]\\
           &= -\tau_2'/\tau_1^2-B-V_P-V_X+\nu T_2/(\gamma \tau_1)^{2}\\
   V_{P\bfe} &= H_{101} - H_{001} - H_{100} + H_{000}\\
              &= \nu \sigma_{W_2}^2 \eta + \nu E_{22} + E_{71} + \nu E_{72} - E_4 - (\nu \sigma_{W_2}^2 \eta + \nu E_{22} + E_{71} + \nu E_{72}) + E_4\\
              &= 0\\
    V_{X\bfe} &= H_{011} - H_{001} - H_{010} + H_{000}\\
              &= E_5+E_6 -E_4 -E_5+E_4\\
              &=E_6\\
              &=\sigma_{\varepsilon}^2 \phi \ttau_2^2/(1-\ttau_2^2 \phi)\\
              &=\sigma_{\varepsilon}^2 V_X/B\\
    V_{PX\bfe} &= \nu \sigma_{W_2}^2 \eta + \nu E_{22} + E_{31} + E_{32} + \nu E_{33} - (E_5 + E_6) - (\nu \sigma_{W_2}^2 \eta + \nu E_{22} + E_{71} + \nu E_{72})\\
    &\quad  - (\nu \sigma_{W_2}^2 \eta + \nu E_{22} + E_{32} + \nu E_{33}) + E_4 + E_5 + \nu \sigma_{W_2}^2 \eta + \nu E_{22} + E_{71} + \nu E_{72} - E_4\\
    &= E_{31} - E_6\\
    &=\sigma_{\varepsilon}^2\big(-\tau_1'/\tau_1^2-1\big) - V_{X\bfe}\,.
\end{align*}
Therefore, we have established the main result, Theorem~\ref{thm:main_ntk}.
\section{Proof of Corollary~\ref{cor:interpolation_threshold}}
\subsection{Bias is non-increasing}
In terms of the auxiliary variables $\ttau_1$ and $\ttau_2$ defined in eqn.~(\ref{eq:ttau}), the coupled equations defining $\tau_1$ and $\tau_2$, eqn.~(\ref{eq:tau_ntk}), simplify to 
\begin{align}
    0&=\gamma  \phi  \tilde{\tau }_2-\gamma  \tilde{\tau }_1+\sigma _{W_2}^2 \left(\tilde{\tau
   }_2 \left(\zeta  \phi  \tilde{\tau }_2+\zeta +\phi  \eta '\right)+\tilde{\tau }_1
   \left(\eta -\eta '\right)+\zeta \right)\\
   0 &=\left(\tilde{\tau }_1-\phi  \tilde{\tau }_2\right) \left(\psi  \tilde{\tau }_1
   \left(\zeta  \tilde{\tau }_2+\eta \right)+\zeta  \phi  \left(\tilde{\tau
   }_2+1\right)\right)+\zeta  \phi  \tilde{\tau }_1 \left(\tilde{\tau }_2+1\right) \sigma
   _{W_2}^2\,.
\end{align}
Eliminating $\ttau_1$ from these equations gives,
\begin{equation}
\begin{split}
&\zeta  \phi  \left(\tilde{\tau }_2+1\right) \left(\left(\eta -\eta '\right) \sigma _{W_2}^2-\gamma \right) \left(\tilde{\tau }_2 \left(\zeta  \phi  \tilde{\tau }_2+\phi  (\gamma +\eta )+\zeta \right)+\sigma _{W_2}^2 \left(\tilde{\tau }_2 \left(\zeta  \phi  \tilde{\tau }_2+\zeta +\phi  \eta '\right)+\zeta \right)+\zeta \right) \\
&\quad=\psi  \left(\zeta  \tilde{\tau }_2+\eta \right) \left(\tilde{\tau }_2 \left(\zeta  \phi  \tilde{\tau }_2+\zeta +\eta  \phi \right)+\zeta \right) \left(\gamma  \phi  \tilde{\tau }_2+\sigma _{W_2}^2 \left(\tilde{\tau }_2 \left(\zeta  \phi  \tilde{\tau }_2+\zeta +\phi  \eta '\right)+\zeta \right)\right)
\end{split}
\end{equation}
Specializing to the random feature kernel ($\sigma_{W_2} = 0$), the equation becomes,
\begin{equation}
\label{eq:tau_2}
\left(\tilde{\tau }_2 \left(\zeta  \psi  \tilde{\tau }_2+\zeta +\eta  \psi \right)+\zeta
   \right) \left(\tilde{\tau }_2 \left(\zeta  \phi  \tilde{\tau }_2+\zeta +\eta  \phi
   \right)+\zeta \right)=-\gamma  \zeta  \phi  \tilde{\tau }_2 \left(\tilde{\tau
   }_2+1\right)\,.
\end{equation}
In the ridgeless limit, $\gamma = 0$, and the quartic equation factorizes into the product of two quadratic polynomials. The root of these equations that respects the conditions of Lemma~\ref{lemma:t1t2} is given by
\begin{equation}
    \label{eq:tau2root}
    \ttau_2 = \frac{-\zeta -\eta  \omega + \sqrt{(\zeta +\eta \omega )^2-4 \zeta ^2 \omega }}{2 \zeta  \omega}\,,
\end{equation}
where $\omega = \text{max}\{\phi,\psi\}$.  Next, recall from Theorem~\ref{thm:main} that $B = \tau_2^2/\tau_1^2 = (1+\ttau_2)^2$ so that
\begin{equation}
    \frac{\partial B}{\partial n_1} = -\frac{\psi^2}{n_0} \frac{\partial B}{\partial \psi} (1+\ttau_2)^2 =  -2\frac{\psi^2}{n_0} (1+\ttau_2)\frac{\partial \ttau_2}{\partial \psi}\,.
\end{equation}
To show that $\partial B/\partial n_1 \le 0$, we show that $(1+\ttau_2) \ge 0$ and that $\partial \ttau_2/\partial \psi \ge 0$. First of all,
\begin{align}
    1 + \ttau_2 &= 1 + \frac{-\zeta -\eta  \omega + \sqrt{(\zeta +\eta \omega )^2-4 \zeta ^2 \omega }}{2 \zeta  \omega} \\
    &= \frac{-(-2 \zeta  \omega+\zeta +\eta  \omega) + \sqrt{(\zeta +\eta \omega )^2-4 \zeta ^2 \omega }}{2 \zeta  \omega}\\
    &\ge \frac{-\sqrt{(-2 \zeta  \omega+\zeta +\eta  \omega)^2} + \sqrt{(\zeta +\eta \omega )^2-4 \zeta ^2 \omega }}{2 \zeta  \omega}\\
    &= \frac{-\sqrt{(\zeta +\eta \omega )^2-4 \zeta ^2 \omega-4 \zeta (\eta-\zeta)\omega^2 } + \sqrt{(\zeta +\eta \omega )^2-4 \zeta ^2 \omega }}{2 \zeta  \omega}
    \ge 0\,,
\end{align}
where we used the relation $\eta \ge \zeta$ which was proved in~\cite{pennington2017nonlinear}. As for the derivative, note that $\partial \ttau_2/\partial \psi = 0$ if $\psi < \phi$ and otherwise,
\begin{align}
    \frac{\partial \ttau_2}{\partial \psi} & = \frac{-(-2 \zeta  \omega +\zeta +\eta 
   \omega) + \sqrt{(\zeta +\eta  \omega )^2-4 \zeta ^2 \omega }}{2 \omega ^2 \sqrt{(\zeta +\eta  \omega )^2-4 \zeta ^2 \omega }}\\
   &\ge \frac{-\sqrt{(-2 \zeta  \omega+\zeta +\eta  \omega)^2}  + \sqrt{(\zeta +\eta  \omega )^2-4 \zeta ^2 \omega }}{2 \omega ^2 \sqrt{(\zeta +\eta  \omega )^2-4 \zeta ^2 \omega }}\\
   &= \frac{-\sqrt{(\zeta +\eta \omega )^2-4 \zeta ^2 \omega-4 \zeta (\eta-\zeta)\omega^2 }  + \sqrt{(\zeta +\eta  \omega )^2-4 \zeta ^2 \omega }}{2 \omega ^2 \sqrt{(\zeta +\eta  \omega )^2-4 \zeta ^2 \omega }}\\
   & \ge 0\,.
\end{align}
Therefore we have shown that
\begin{equation}
    \frac{\partial B}{\partial n_1} \le 0\,,
\end{equation}
i.e. the bias $B$ is monotonically decreasing.
\subsection{Behavior near the interpolation boundary}
From inspection of the expressions in Theorem~\ref{thm:main}, the bias and variance terms depend on $\tau_1$ and $\tau_2$ through four ratios, $\tau_2/\tau_1$, $\tau_2'/\tau_1'$, $\tau_2'/\tau_1^2$ and $\tau_1'/\tau_1^2$. In the ridgeless ($\gamma=0$) limit, these ratios can all be expressed in terms of $\ttau_2$ by using eqns.~(\ref{eq:dtau1}), (\ref{eq:dtau2}), (\ref{eq:ttau}) and~(\ref{eq:tau_2}).

We examine the behavior near the interpolation thereshold $\phi=\psi$ by taking the limit from both directions. It is straightforward algebraic substitution to show that for $\phi < \psi$, 
\begin{equation}
    \label{eq:ratios1}
    \frac{\tau_2}{\tau_1} = \frac{\tau_2'}{\tau_1'} = 1+\ttau_2\,,\quad \frac{\tau_1'}{\tau_1^2} = \frac{\zeta  \left(\tilde{\tau }_2+1\right)}{(\phi - \psi)  \tilde{\tau }_2 \left(\zeta 
   \tilde{\tau }_2+\eta \right)}\,,\quad
    \frac{\tau_2'}{\tau_1^2} = \frac{\zeta  \left(\tilde{\tau }_2+1\right){}^2}{(\phi - \psi)  \tilde{\tau }_2 \left(\zeta 
   \tilde{\tau }_2+\eta \right)}\,.
\end{equation}
From eqn.~(\ref{eq:tau2root}), we see that $\ttau_2$ is finite when $\phi = \psi$ so that the terms in eqn.~(\ref{eq:ratios1}) obey,
\begin{equation}
\label{eq:ratios1_limit}
    \frac{\tau_2}{\tau_1} = \frac{\tau_2'}{\tau_1'} = \mathcal{O}(1)\,,\quad \frac{\tau_1'}{\tau_1^2} = \mathcal{O}\left(\frac{1}{\phi - \psi}\right)\,,\quad
    \frac{\tau_2'}{\tau_1^2} = \mathcal{O}\left(\frac{1}{\phi - \psi}\right)\,.
\end{equation}
Turning now to the case of $\phi>\psi$, similar algebraic substitutions yield,
\begin{align}
   \frac{\tau_2}{\tau_1} &= \tilde{\tau }_2+1\\
   \frac{\tau_2'}{\tau_1'} &= \frac{\zeta ^2 \tilde{\tau }_2 \left(\tilde{\tau }_2^2 (-\psi +\phi +1)+3 \tilde{\tau
   }_2+2\right)+\zeta  \eta  \left(\tilde{\tau }_2^3 (\psi -\phi )+\tilde{\tau }_2^2
   (\phi -\psi )+\tilde{\tau }_2+1\right)+\eta ^2 \tilde{\tau }_2^2 (\psi -\phi )}{\zeta 
   \left(\zeta  \tilde{\tau }_2 \left(\tilde{\tau }_2^2 (\phi -\psi )+\tilde{\tau
   }_2+2\right)+\eta  \tilde{\tau }_2^2 (\phi -\psi )+\eta \right)}\\
  \frac{\tau_1'}{\tau_1^2} &= \frac{\zeta  \left(\tilde{\tau }_2+1\right)}{\tilde{\tau }_2 (\psi -\phi ) \left(\zeta 
   \tilde{\tau }_2+\eta \right)}-\frac{\zeta  \tilde{\tau }_2 \left(\tilde{\tau
   }_2+1\right)}{\zeta  \tilde{\tau }_2 \left(\tilde{\tau }_2+2\right)+\eta }\\
   \frac{\tau_2'}{\tau_1^2} &= \frac{\zeta  \left(\tilde{\tau }_2+1\right){}^2}{\tilde{\tau }_2 (\psi -\phi )
   \left(\zeta  \tilde{\tau }_2+\eta \right)}+\frac{\tilde{\tau }_2 \left(\tilde{\tau
   }_2+1\right) (\eta -\zeta )}{\zeta  \tilde{\tau }_2 \left(\tilde{\tau
   }_2+2\right)+\eta }\,.
\end{align}
We can isolate the pole at $\phi=\psi$ by examining the relevant functions of $\ttau_2$. In particular,
substituting the solution~(\ref{eq:tau2root}) gives,
\begin{equation}
\frac{\tau_2'}{\tau_1'}|_{\psi = \phi} = \frac{\sqrt{(\zeta +\eta  \phi )^2-4 \zeta ^2 \phi }+2 \zeta  \phi -\zeta -\eta  \phi }{2\zeta  \phi }
\end{equation}
which is evidently finite when $\phi = \psi$ and,
\begin{equation}
\frac{\tilde{\tau }_2 \left(\tilde{\tau
   }_2+1\right)}{\zeta  \tilde{\tau }_2 \left(\tilde{\tau }_2+2\right)+\eta } = \frac{2 \zeta \phi }{(\zeta +\eta  \phi ) \sqrt{(\zeta +\eta  \phi )^2-4 \zeta ^2 \phi}+(\zeta +\eta  \phi )^2-4 \zeta ^2 \phi}\,
\end{equation}
whose denominator is a sum of non-negative terms that only vanishes if $\phi = 1$ and $\eta = \zeta$, i.e. the activation function is linear. Therefore, we find for $\phi > \psi$ the same behavior as for $\phi < \psi$, namely,
\begin{equation}
\frac{\tau_2}{\tau_1} = \frac{\tau_2'}{\tau_1'} = \mathcal{O}(1)\,,\quad \frac{\tau_1'}{\tau_1^2} = \mathcal{O}\left(\frac{1}{\phi - \psi}\right)\,,\quad\frac{\tau_2'}{\tau_1^2} = \mathcal{O}\left(\frac{1}{\phi - \psi}\right)\,.
\end{equation}
Altogether, we conclude that as $\phi\to\psi$,
\begin{equation}
B  = \mathcal{O}(1)\,,\quad V_P = \mathcal{O}(1)\,,\quad
V_X = \mathcal{O}(1)\,,\quad V_{X\bfe} = \mathcal{O}(1)\,,\quad
V_{PX} = \mathcal{O}\left(\frac{1}{\phi - \psi}\right)\,,\quad V_{PX\bfe} = \mathcal{O}\left(\frac{1}{\phi - \psi}\right)\,,
\end{equation}
i.e. the only divergent terms are $V_{PX}$ and $V_{PX\bfe}$.

\section{The Bias-Variance Decomposition of d’Ascoli \emph{et al.} \cite{d2020double}}
\label{sec_dasc_bv}

While finalizing this manuscript, we became aware of a related work~\cite{d2020double} that similarly proposes and calculates a multivariate variance decomposition in order to examine the origins of double descent. Their approach is sequential in nature, first defining $\mathcal{E}_{\text{Noise}}$ to be the (expected) variance conditional on $P$ and $X$, then $\mathcal{E}_{\text{Init}}$ to be the remaining variance conditional on $X$, and finally $\mathcal{E}_{\text{Samp}}$ to be the remaining variance. In terms of our fine-grained decomposition, their expressions read,
\begin{equation}
   \mathcal{E}_{\text{Bias}} = B\,,\quad \mathcal{E}_{\text{Init}}  = V_{P} + V_{PX}\,,\quad \mathcal{E}_{\text{Samp}} = V_X\,,\quad \mathcal{E}_{\text{Noise}} = V_{PX\varepsilon} + V_{X_\varepsilon} + V_{P\varepsilon} + V_{\varepsilon}\,.
\end{equation}
Fig.~\ref{fig:dAscoli}(a) illustrates their decomposition in terms of a Venn diagram and Fig.~\ref{fig:dAscoli}(b) shows how the components of their decomposition behave as the number of random features varies, similarly to Figs.~\ref{fig:venn_variance} and~\ref{fig:venn_variance_sm}. Note that their total bias and total variance agree with ours, and that their decomposition also resolves the two separate divergent terms at the interpolation threshold (since $\mathcal{E}_\text{Noise}$ contains $V_{PX\e}$ and $\mathcal{E}_\text{Init}$ contains $V_{PX}$). However, because their decomposition is not fully multivariate, the resulting areas do not necessarily possess the interpretations one might expect from the names ``noise variance," ``initialization variance," and ``sampling variance."  For example, the divergence in $\mathcal{E}_{\text{Noise}}$ ultimately comes from the contribution of $V_{PX\varepsilon}$, which vanishes when you ensemble over initial parameters, for example. This strong dependence on the parameters does not seem like a desirable property of a quantity designed to measure the variance due to noise. Similarly, the divergence of $\mathcal{E}_\text{Init}$ can be eliminated by ensembling (bagging) over different training samples, which also seems like a undesirable property of ``initialization variance." The underlying reason for these inconsistent interpretations is that the divergences ultimately arise from the \emph{interaction} terms $V_{PX}$ and $V_{PX\varepsilon}$, but these interactions are not captured in their decomposition.
\begin{figure*}[t!]
    \centering
\includegraphics[width=0.7\linewidth]{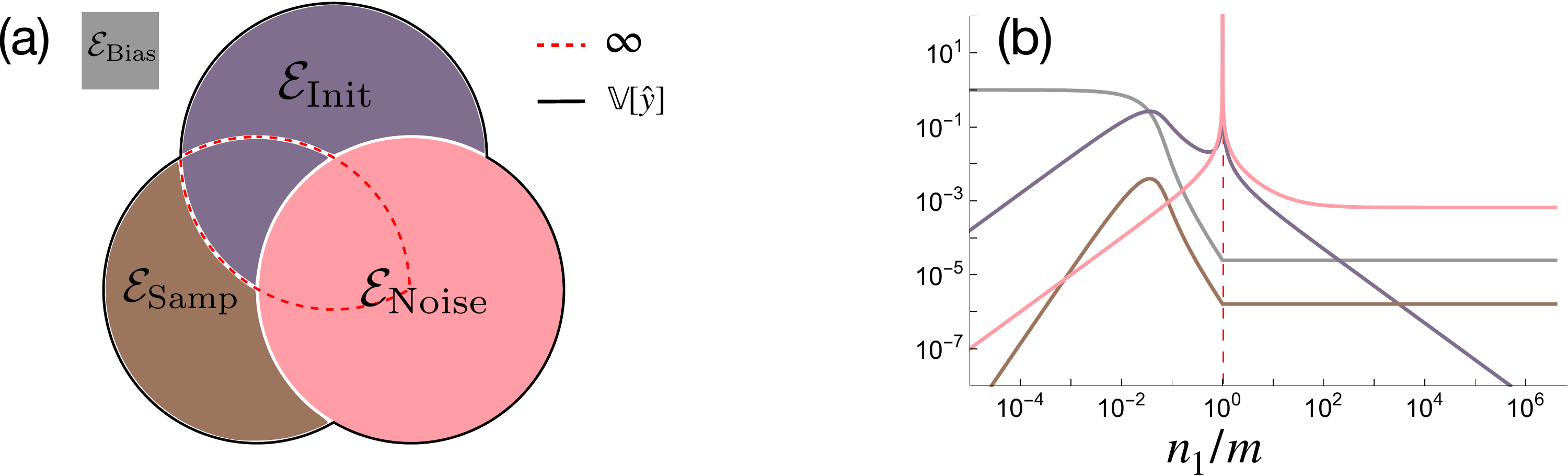}
\caption{The multivariate variance decomposition of~\cite{d2020double}. Following the setup of Fig.~\ref{fig:venn_variance}, panel (a) depicts the decomposition with a Venn diagram and panel (b) shows plots of the individual terms as functions of the overparameterization ratio $n_1/m$. The total variance is partitioned into three terms in a sequential manner, breaking the symmetry of the random variables and failing to account for their interactions. Since it is those interactions that cause the divergences (see Corollary~\ref{cor:interpolation_threshold}), it is not possible to unambiguously attribute the divergences to a univariate source of variance, despite the the observed spikes in $\mathcal{E}_\text{Noise}$ and $\mathcal{E}_\text{Init}$.}
\label{fig:dAscoli}
\end{figure*}